\documentclass[twoside]{article}

\PassOptionsToPackage{sort,round,compress}{natbib}

\usepackage[preprint]{aistats2027}

\usepackage{natbib}

\usepackage[utf8]{inputenc}
\usepackage[T1]{fontenc}
\usepackage{url}
\usepackage{booktabs}
\usepackage{tabularx}
\usepackage{nicefrac}
\usepackage{microtype}
\usepackage{doi}
\usepackage{blindtext}
\usepackage{multicol, multirow, makecell}
\usepackage[normalem]{ulem}
\usepackage{algorithm}

\usepackage{enumerate}
\usepackage{subcaption}
\usepackage{rotating}
\usepackage{tablefootnote, threeparttable}
\usepackage[dvipsnames,table]{xcolor}

\definecolor{darkblue}{rgb}{0.0,0.0,0.65}
\definecolor{darkred}{rgb}{0.65,0.0,0.0}
\definecolor{darkgreen}{rgb}{0.0,0.5,0.0}
\definecolor{tab:blue}{RGB}{31,119,180}
\definecolor{tab:red}{RGB}{214,39,40}
\definecolor{tab:green}{RGB}{44,160,44}
\definecolor{tab:orange}{RGB}{255,127,14}
\definecolor{thmframe}{RGB}{0,114,178}
\definecolor{thmback}{RGB}{248,251,255}
\definecolor{propframe}{RGB}{0,158,115}
\definecolor{propback}{RGB}{246,252,250}
\definecolor{corframe}{RGB}{204,121,167}
\definecolor{corback}{RGB}{252,248,251}
\definecolor{assframe}{RGB}{213,94,0}
\definecolor{assback}{RGB}{255,249,245}
\definecolor{defframe}{RGB}{145,145,145}
\definecolor{defback}{RGB}{250,250,250}
\definecolor{rmkframe}{RGB}{120,120,120}
\definecolor{rmkback}{RGB}{249,249,249}
\definecolor{algframe}{RGB}{86,180,233}
\definecolor{algback}{RGB}{247,252,255}

\usepackage{amsfonts, amsmath, amssymb}
\usepackage{dsfont}
\usepackage{thmtools, thm-restate}
\usepackage{amsthm}
\usepackage{tcolorbox}
\tcbuselibrary{theorems,skins,breakable}

\usepackage{hyperref}
\hypersetup{
    hypertexnames = false,
    colorlinks = true,
    citecolor  = darkblue,
    linkcolor  = darkred,
    filecolor  = darkblue,
    urlcolor   = darkblue,
}

\definecolor{weightpurple}{RGB}{168,88,141}
\definecolor{kappavermillion}{RGB}{156,63,0}
\definecolor{curvatureblue}{RGB}{0,61,112}
\definecolor{dispersionteal}{RGB}{0,107,79}
\DeclareRobustCommand{\weightcolor}[1]{{\color{weightpurple}#1}}
\DeclareRobustCommand{\kappacolor}[1]{{\color{kappavermillion}#1}}
\DeclareRobustCommand{\curvaturecolor}[1]{{\color{curvatureblue}#1}}
\DeclareRobustCommand{\dispersioncolor}[1]{{\color{dispersionteal}#1}}
\usepackage[capitalize,noabbrev]{cleveref}
\usepackage{amsmath,amsfonts,bm,mathtools}

\def\ceil#1{\lceil #1 \rceil}

\def\1{\bm{1}}

\def\vzero{{\bm{0}}}

\def\vtheta{{\bm{\theta}}}

\def\ve{{\bm{e}}}

\def\vv{{\bm{v}}}

\def\vx{{\bm{x}}}

\def\mA{{\bm{A}}}
\def\mB{{\bm{B}}}

\def\mH{{\bm{H}}}
\def\mI{{\bm{I}}}

\def\mV{{\bm{V}}}

\DeclareMathAlphabet{\mathsfit}{\encodingdefault}{\sfdefault}{m}{sl}
\SetMathAlphabet{\mathsfit}{bold}{\encodingdefault}{\sfdefault}{bx}{n}

\def\gA{{\mathcal{A}}}
\def\gB{{\mathcal{B}}}
\def\gC{{\mathcal{C}}}

\def\gE{{\mathcal{E}}}
\def\gF{{\mathcal{F}}}

\def\gN{{\mathcal{N}}}
\def\gO{{\mathcal{O}}}

\def\gS{{\mathcal{S}}}
\def\gT{{\mathcal{T}}}

\def\gX{{\mathcal{X}}}

\def\sN{{\mathbb{N}}}

\def\sP{{\mathbb{P}}}

\def\sR{{\mathbb{R}}}

\newcommand{\E}{\mathbb{E}}

\newcommand{\R}{\mathbb{R}}

\newcommand{\KL}{D_{\mathrm{KL}}}

\DeclareMathOperator*{\argmax}{arg\,max}
\DeclareMathOperator*{\argmin}{arg\,min}

\tcbset{
    sharpthm/.style n args={2}{
        enhanced,
        breakable,
        sharp corners,
        boxrule=0.6pt,
        left=6pt,
        right=6pt,
        top=5pt,
        bottom=5pt,
        before skip=8pt,
        after skip=8pt,
        colback=#1,
        colframe=#2,
        fonttitle=\bfseries,
        coltitle=black,
        colbacktitle=#1
    },
    algobox/.style={
        enhanced,
        breakable,
        sharp corners,
        boxrule=0.6pt,
        left=6pt,
        right=6pt,
        top=5pt,
        bottom=5pt,
        before skip=8pt,
        after skip=8pt,
        colback=algback,
        colframe=algframe,
        coltitle=black,
        colbacktitle=algback,
        fonttitle=\bfseries
    }
}
\theoremstyle{plain}
\newtheorem{theorem}{Theorem}%
\newtheorem{proposition}{Proposition}
\newtheorem{lemma}{Lemma}

\theoremstyle{definition}
\newtheorem{definition}[theorem]{Definition}
\theoremstyle{plain}
\newtheorem{assumption}{Assumption}
\newtheorem{remark}{Remark}

\tcolorboxenvironment{theorem}{sharpthm={thmback}{thmframe}}
\tcolorboxenvironment{proposition}{sharpthm={propback}{propframe}}
\tcolorboxenvironment{corollary}{sharpthm={corback}{corframe}}
\tcolorboxenvironment{definition}{sharpthm={defback}{defframe}}
\tcolorboxenvironment{lemma}{sharpthm={defback}{defframe}}

\usepackage{autonum}
\crefname{assumption}{assumption}{assumptions}
\Crefname{assumption}{Assumption}{Assumptions}
\crefname{algocf}{algorithm}{algorithms}
\Crefname{algocf}{Algorithm}{Algorithms}

\DeclareMathOperator{\Poi}{Poi}

\newcommand{\bignorm}[1]{\left\lVert #1 \right\rVert}

\newcommand{\Reg}{\mathrm{Reg}}

\usepackage{pifont}
\newcommand{\cmark}{\text{\ding{51}}}
\newcommand{\xmark}{\text{\ding{55}}}
\newcommand{\gyes}{{\color[rgb]{0,.8,0}\cmark}}
\newcommand{\rno}{\color[rgb]{.8,0,0}\xmark}
\usepackage{adjustbox}

\allowdisplaybreaks

\usepackage{enumitem}
\setlist[itemize]{leftmargin=*, labelindent = 10pt}
\setlist[enumerate]{leftmargin=*, labelindent = 10pt}

\usepackage[algo2e, linesnumbered, ruled, vlined]{algorithm2e}
\SetKwInput{KwInput}{Input}
\SetKwInput{KwOutput}{Output}

\newcommand{\junghyun}[1]{\textcolor{tab:red}{\bfseries JL: #1}}

\newcommand{\debug}[1]{\textcolor{red}{#1}}

\def\bignorm#1{\left\lVert #1 \right\rVert}

\newcommand{\dmu}{\dot{\mu}}
\newcommand{\ddmu}{\ddot{\mu}}

\newif\ifFINAL
\FINALfalse %

\ifFINAL
    \renewcommand{\junghyun}[1]{}
    \renewcommand{\hanseul}[1]{}
    \renewcommand{\debug}[1]{}
    \newcommand{\kj}[1]{}
    
\else
\fi
\definecolor{revisionpurple}{RGB}{0,0,255}

\AddToNoCaseChangeList{\Cref \cref \ref \eqref}
\newcommand{\papertitle}{A Jointly Efficient and Near-Optimal Algorithm for Heteroskedastic Generalized Linear Bandits with Adversarial Corruptions}
\runningtitle{Efficient and Corruption-Robust Heteroskedastic Generalized Linear Bandits}

\begin{document}

\twocolumn[

\aistatstitle{\papertitle}

\aistatsauthor{ Sanghwa Kim$^*$ \And Junghyun Lee$^*$ \And Se-Young Yun }

\aistatsaddress{ KAIST AI \\ \texttt{tkdghk9667@kaist.ac.kr} \And KAIST AI \\ \texttt{jh\_lee00@kaist.ac.kr} \And KAIST AI \\ \texttt{yunseyoung@kaist.ac.kr}} ]

{\renewcommand{\thefootnote}{\fnsymbol{footnote}}\footnotetext[1]{Equal contributions.}}

\begin{abstract}
    We consider heteroskedastic generalized linear bandits (GLBs) with adversarial corruptions and propose \texttt{HCW-GLB-OMD}, a simple algorithm that combines an online mirror descent (OMD)-based estimator with Hessian-based confidence weights.
    Although both components are well studied, establishing guarantees for the combined algorithm requires a nontrivial analysis.
    We establish a corruption-robust confidence sequence using technical tools not required in prior analyses, including a weighted likelihood supermartingale and a self-bounding argument.
    For the regret analysis, we derive a new Hessian potential lemma to handle the interaction between adversarial corruption and adaptive confidence weighting.
    For self-concordant GLMs, we prove a high-probability regret bound of $\widetilde{\mathcal{O}}\bigl( d \sqrt{\sum_t \dispersioncolor{g(\tau_t)} \curvaturecolor{\dot\mu_{t,\star}}} + d^2 \dispersioncolor{g_{\max}} \kappacolor{\kappa} + d(1+\kappacolor{\kappa})(C\vee1) \bigr)$, where $\curvaturecolor{\dot\mu_{t,\star}}$ is the slope of $\mu$ at the optimal arm at time $t$, $\dispersioncolor{g(\tau_t)}$ is the observed dispersion, $\dispersioncolor{g_{\max}}=\max_t \dispersioncolor{g(\tau_t)}$, $\kappacolor{\kappa}$ is the worst-case inverse slope of $\mu$, and $C$ is the corruption budget.
    The algorithm attains this bound with one OMD update per round and $\mathcal{O}(d^2)$ storage.
    We complement this guarantee with matching lower bounds for the leading $\sqrt{T}$ term and, up to $\kappacolor{\kappa}$, the $dC$ corruption term.
\end{abstract}

\section{Introduction}
\label{sec:intro}

\paragraph{Generalized Linear Bandits.}
Stochastic \textit{contextual bandits} provide a mathematical framework for sequential decision making under uncertainty, with applications ranging from personalized recommendation systems~\citep{li2010news, qin2014contextual} to adaptive clinical trials~\citep{villar2015multi,durand2018carcinogenesis}.
To capture the nonlinear rewards in the real world, \textbf{\textit{generalized linear bandits (GLBs)}}~\citep{filippi2010glm,li2017glm,lee2024glm,liu2024free,zhang2025onepass} have emerged as a principled parametric framework, where rewards follow a \emph{generalized linear model (GLM)}~\citep{glm}: for an arm $\vx_t \in \sR^d$, known dispersion parameter $\tau_t \geq 0$, and unknown parameter $\vtheta_\star \in \sR^d$, a stochastic reward is sampled as $r_t \sim GLM(\cdot \mid \vx_t, \tau_t; \vtheta_\star)$ with density
\begin{equation}
	\label{eqn:GLM}
	d p (r \mid \vx, \tau_t; \vtheta_\star) \propto \exp \left( \frac{r \langle \vx, \vtheta_\star \rangle - m ( \langle \vx, \vtheta_\star \rangle) }{\dispersioncolor{g(\tau_t)}} \right) d\nu.
\end{equation}
Here, $\dispersioncolor{g} : \sR \rightarrow \sR_{\geq 0}$ is the dispersion function controlling the variability of the GLM; $d\nu$ is a known base measure, independent of $\vtheta_\star$ but possibly dependent on $\tau_t$; and $m: \sR \mapsto \sR$ is the log-partition function.
We define the \textit{(inverse) link function} as $\mu := m'$.
The conditional mean and variance satisfy~\citep{graphicalmodels}: $\mathbb{E} [ r_t | \vx_t, \tau_t; \vtheta_\star ] = \mu(\langle \vx_t, \vtheta_\star \rangle)$, and $\mathrm{Var}[r_t | \vx_t, \tau_t; \vtheta_\star] = \dispersioncolor{g(\tau_t)} \curvaturecolor{\dmu}(\langle \vx_t, \vtheta_\star \rangle )$.

This framework encompasses various canonical parametric stochastic bandit settings, including Gaussian linear bandits~\citep{abbasiyadkori2011linear}, logistic bandits~\citep{zhang2016logistic,faury2020logistic,abeille2021logistic}, and Poisson bandits~\citep{mutny2021poisson,lee2024glm}.
In real-world deployments, however, feedback can differ in its statistical precision (\emph{heteroskedasticity}), and can also be strategically manipulated (\emph{adversarial corruption}).
Exploiting the former while remaining robust to the latter, \emph{simultaneously}, is the central objective of this work.

\begin{table*}[t]
    \centering
    \normalsize
    \setlength{\tabcolsep}{5pt}
    \newcommand{\resultcell}[1]{\ensuremath{\vcenter{\hbox{\shortstack[c]{#1}}}}}
    \caption{\label{tab:results}
    Comparison of regret bounds for \textbf{heteroskedastic GLBs} with and without corruption. Logarithmic factors are omitted to highlight dependence on $d, T, {\curvaturecolor{\dmu_{t,\star}}}, { \dispersioncolor{g(\tau_t)}}, {\kappacolor{\kappa}}$, and $C$. We define ${\dispersioncolor{g_{\max}} \coloneq \max_{t \in [T]} \dispersioncolor{g(\tau_t)}}$. For heteroskedastic linear bandits, $\sigma_t^2$ denotes the conditional noise variance. ``Comp. Complexity'' reflects both space and time dependence on the iteration $t$ only. ``Self-con.'' indicates self-concordance (\Cref{assumption:link-function}).}
    \begin{adjustbox}{max width=\textwidth}
    \begin{threeparttable}
    \begin{tabular}{|c|c|c|c|c|c|} \hline 
        \textbf{Algorithm} & \textbf{Regret Bound} & 
        \textbf{\!\!Corrup. \!\!} & \textbf{\!\!Reward\!\!}  & \textbf{\!\!Noise\!\!}& \makecell{ \!\! \textbf{Comp.} \\ \textbf{Complexity} \!\!} \\
        \hline
        \makecell{\texttt{WeightedOFUL+} \\ {\small\citep[Theorem 4.1]{zhou2022variance}}} 
       & $d\sqrt{\sum_t \sigma_t^2} + d$ 
        & \rno & Linear & \makecell{Bounded by $R$ } & $\gO(1)$ \\
        \hline
        \makecell{\texttt{MOR-UCB}\tnote{$*$} \\ {\small\citep[Theorem 6.3]{zhao2023optimal}}} 
       & \makecell{$d{\kappacolor{\kappa}}\sqrt{ \sum_t {\dispersioncolor{g(\tau_t)}}} + {\kappacolor{\kappa}} \sqrt{dT}$ \\ and\\ $d{\kappacolor{\kappa}}\sqrt{ \sum_t {\dispersioncolor{g(\tau_t)}} {\curvaturecolor{\dmu_{t, \star}}}} + {\kappacolor{\kappa}} \sqrt{dT} + d^2 { \dispersioncolor{g_{\max}}} {\kappacolor{\kappa^2}}$}
        & \rno & GLM & Bounded by $R$ & $\gO(t)$ \\
        \hline 
        \hline 
        \makecell{\texttt{GLB-OMD}\tnote{$\dagger$} \\ {\small\citep[Theorem 2]{zhang2025onepass}}} 
        & $d\sqrt{ {\dispersioncolor{g(\tau)}} \sum_t {\curvaturecolor{\dmu_{t, \star}}}}  + d {\kappacolor{\kappa}} + d^2 { \dispersioncolor{g_{\max}}} {\kappacolor{\kappa}}$ 
        & \rno & \makecell{Self-con. \\ GLM} & GLM & $\gO(1)$ \\
        \hline
        \makecell{Lower Bound\tnote{$\ddagger$} \\ \small\citep[Theorem 2]{abeille2021logistic}} 
       & $d \sqrt{T {\curvaturecolor{\dmu_{\star}}}}$
        & \rno & Logistic & &  \\
        \hline
        \makecell{Lower Bound \\ {\small\citep[Theorem 4.1]{he2025lowerbound}}} 
       & $d \sqrt{\sum_t \sigma_t^2} / \log T$
        & \rno & Linear & &  \\
        \hline\hline
        \makecell{\texttt{CW-OFUL} \\ {\small \citep[Theorem 4.2]{he2022corruption}}} 
        & $d \sqrt{T} + d (C\vee1)$ 
        & \gyes & Linear & $\sigma$-SG & $\gO(1)$ \\
        \hline
        \makecell{\texttt{CR-Eluder-UCB} \\ {\small \citep[Theorem 4.1]{ye2023corruption}}} 
        & \makecell{$d {\kappacolor{\kappa}} \sqrt{T} + d {\kappacolor{\kappa^2}} (C\vee1)$}
        & \gyes & GLM & $\sigma$-SG & $\gO(t)$ \\
        \hline
        \makecell{\texttt{GAdaOFUL}\tnote{$*$} \\ {\small \citep[Theorem 4]{yu2025corruption}}} 
        & \makecell{$d {\kappacolor{\kappa}} \sqrt{\sum_t {\dispersioncolor{g(\tau_t)}}} + d {\kappacolor{\kappa}} (C\vee1)$ \\ or\\ $d{\kappacolor{\kappa}}\sqrt{ \sum_t {\dispersioncolor{g(\tau_t)}} {\curvaturecolor{\dmu_{t, \star}}}} + d {\kappacolor{\kappa}} (C\vee1) + d^2 { \dispersioncolor{g_{\max}}} {\kappacolor{\kappa^2}}$}
        & \gyes & GLM & Finite variance & $\gO(t)$ \\
        \hline
        \makecell{Lower Bound \\ {\small\citep[Theorem 3]{bogunovic2021corrupted}}} 
       & $d C$
        & \gyes & Linear & &  \\
        \hline\hline
        \rowcolor{gray!30}
        \resultcell{\strut\quad\quad\:\, \textbf{\texttt{HCW-GLB-OMD} (Ours)} \quad\quad\:\: \\ {\small\strut \textbf{(Theorem~\ref{theorem:regret})}}} 
        & $d \sqrt{\sum_t { \dispersioncolor{g(\tau_t)}} {\curvaturecolor{\dmu_{t, \star}}}} + d {\kappacolor{\kappa}}(C\vee1) + d^2 { \dispersioncolor{g_{\max}}} {\kappacolor{\kappa}}$
        & \gyes & \resultcell{\strut Self-con. \\ \strut GLM} & GLM & $\gO(1)$ \\
        \hline
        \rowcolor{gray!30}
        \resultcell{\strut\quad\quad\quad \textbf{Lower Bound (Ours)} \quad\quad\quad \\ {\small\strut \textbf{(Theorems~\ref{thm:lower-bound}--\ref{thm:lower-bound-corruption})}}}
        & \resultcell{\strut$d\sqrt{\sum_t \dispersioncolor{g(\tau_t)}\curvaturecolor{\dmu_{t,\star}}/\log T}$ and $dC$}
        & \gyes & \resultcell{\strut Self-con. \\ \strut GLM} & &\\
        \hline
    \end{tabular}
    \begin{tablenotes}
        \item[$*$] The second regret bound uses the conditional variance $\dispersioncolor{g(\tau_t)}\curvaturecolor{\dmu}(\langle\vx_t,\vtheta_\star\rangle)$ as feedback at each round $t$ (\Cref{app:comparison}).
        \item[$\dagger$] The algorithms work only when ${\dispersioncolor{g(\tau_t)} = \dispersioncolor{g(\tau)}}$, i.e., when the dispersion is fixed.
        \item[$\ddagger$] {Here, $\curvaturecolor{\dmu_\star} \triangleq \curvaturecolor{\dmu}(\langle \vx_\star, \vtheta_\star \rangle)$, where $\vx_\star \coloneq \argmax_{\vx \in \gX} \mu(\langle \vx, \vtheta_\star \rangle )$. The lower bound holds for fixed arm-set.}
    \end{tablenotes}
    \end{threeparttable}
    \end{adjustbox}
\end{table*}

\paragraph{Heteroskedasticity.}
Real-world environments often exhibit \textbf{\textit{heteroskedasticity}}, where the reward variance varies across arms or rounds.
The local slope $\curvaturecolor{\dmu}(\langle\vx_t, \vtheta_\star \rangle)$ captures endogenous variation (e.g., Bernoulli, Poisson), while the dispersion $\dispersioncolor{g(\tau_t)}$ can also vary \emph{exogenously}.
One prominent example of the latter is \emph{heteroskedastic Gaussian linear bandits}, where the variance sequence $\dispersioncolor{g(\tau_t)} = \sigma_t^2$ is chosen exogenously~\citep{zhou2021variance,zhang2021variance,kim2022variance,zhou2022variance,zhao2023variance,jia2024lowerbound,he2025lowerbound,zhao2026breaking,xu2023adaptive}.
Another example is the \emph{average} of $n_t\geq1$ conditionally independent Bernoulli responses under the same arm, such as a click-through rate (CTR) over $n_t$ impressions.
For the logistic link, e.g., in conversion modeling for display advertising~\citep{schwartz-bradlow-fader}, $r_t=K_t/n_t$ with $K_t\sim\mathrm{Bin}(n_t,p_t)$ and $p_t=\mu(\langle\vx_t,\vtheta_\star\rangle)$ has conditional variance $p_t(1-p_t)/n_t$, giving $\dispersioncolor{g(\tau_t)}=1/n_t$ in Eqn.~\eqref{eqn:GLM}.
Likewise, for the Poisson link, e.g., demand-rate estimation in dynamic pricing~\citep{besbes-zeevi}, an event count $N_t\sim\Poi(e_t\mu(\langle\vx_t,\vtheta_\star\rangle))$ over an exposure $e_t>0$ yields a normalized rate $r_t=N_t/e_t$ with dispersion $\dispersioncolor{g(\tau_t)}=1/e_t$.

\paragraph{Adversarial Corruptions.}
Such feedback is also vulnerable to \textbf{\textit{adversarial corruption}}: an adversary may inflate reported CTRs or manipulate event counts in spatiotemporal monitoring, e.g., ride-sharing~\citep{mutny2021poisson}.
We allow an adaptive adversary to allocate a corruption budget $C$ across rounds.

The literature on corruption-robust algorithms has grown steadily, ranging from unstructured multi-armed bandits~\citep{lykouris2018corruption,guan2020robust,gupta2019corrupted} to (generalized) linear bandits~\citep{zhao2021corrupted,bogunovic2020corruption,bogunovic2021corrupted,ding2022corrupted,he2022corruption,liu2024corruption,yu2025corruption}, and contextual bandits with general function approximation~\citep{ye2023corruption}.
For linear bandits, a minimax-optimal regret bound of $\widetilde{\Theta}(d\sqrt{T} + d(C\vee1))$ has been proven~\citep{dani2007linear,bogunovic2021corrupted,he2022corruption}.
\citet{yu2025corruption} address variance and corruption jointly under a finite-variance noise model, using adaptive Huber regression and valid variance upper bounds; see \Cref{app:comparison} for a more detailed comparison.
For corrupted \textbf{GLB}s, the key difficulty is to retain the additive corruption term \emph{and} the benefit of local curvature.
Without corruption and {with} fixed dispersion $\tau_t = \tau$, self-concordant \textbf{GLB}s admit a leading regret term of $\widetilde{\gO}(d\sqrt{\dispersioncolor{g(\tau)}\sum_t{\curvaturecolor{\dmu_{t,\star}}}})$~\citep{abeille2021logistic,lee2024glm}, where $\curvaturecolor{\dmu_{t,\star}}=\curvaturecolor{\dmu}(\langle\vx_{t,\star},\vtheta_\star\rangle)$ is the local slope at an optimal arm; local minimax optimality is known for logistic bandits~\citep[Theorem 2]{abeille2021logistic}.
Existing corruption-robust \textbf{GLB} algorithms~\citep{yu2025corruption,ye2023corruption} instead retain a worst-case curvature factor
\begin{equation}
    {\kappacolor{\kappa}}\coloneq\left( \min_{\vx\in\mathcal{X},\vtheta\in\Theta}\curvaturecolor{\dmu}(\langle\vx,\vtheta\rangle) \right)^{-1}
\end{equation}
in their leading term (\Cref{tab:results}), losing local optimality.

\paragraph{Interaction Protocol.}
At each $t\in[T]=\{1,\ldots,T\}$, the learner receives a potentially adversarial, time-varying contextual arm-set $\gX_t\subset\sR^d$ and selects an action $\vx_t\in\gX_t$.
After observing $\vx_t$, the adversary selects a dispersion parameter $\tau_t$, and a stochastic reward is then sampled as $r_t\sim GLM(\cdot\mid\vx_t,\tau_t;\vtheta_\star)$.
After observing $r_t$, the adversary chooses a corruption $c_t\in\sR$.
At the end of the round, the learner observes the pair $(\tau_t,\tilde r_t)$, where $\tilde r_t=r_t+c_t$; the dispersion is thus revealed only after the action, as in variance-aware linear bandits~\citep{zhou2021variance,zhou2022variance}.
The learner observes neither the true reward $r_t$ nor whether the observed $\tilde{r}_t$ is corrupted (or by how much).
The \emph{corruption budget} $C$ satisfies $\sum_{t=1}^T|c_t|\leq C$ almost surely, as in \citet{he2022corruption}.
Both adversarial choices may depend on the past history and selected arm; only $\tau_t$ is reported truthfully at the end of each round $t$.

Known dispersion and unknown corruption describe different features of the feedback.
A platform may reliably record its impression counts or observation periods, which vary across rounds and determine $\dispersioncolor{g(\tau_t)}$, while receiving outcome reports from a source it cannot fully trust.
Even a large batch or long exposure does not rule out systematic manipulation.
Moreover, knowing $\dispersioncolor{g(\tau_t)}$ does not reveal the conditional variance in logistic or Poisson models, since the local slope depends on the unknown parameter: $\mathrm{Var}[r_t | \vx_t, \tau_t; \vtheta_\star] = \dispersioncolor{g(\tau_t)} \curvaturecolor{\dmu}(\langle \vx_t, \vtheta_\star \rangle )$.
Moreover, the adversary can allocate its budget $C$ to \emph{any} rounds, including those with small dispersion.

The learner's goal is to minimize the pseudo-regret: denoting $\vx_{t,\star} \triangleq \argmax_{\vx \in \gX_t} \mu ( \langle \vx, \vtheta_\star \rangle),$
\begin{equation}
	\label{eqn:regret}
	\Reg(T) := \sum_{t=1}^{T} \mu ( \langle \vx_{t,\star}, \vtheta_\star \rangle) - \sum_{t=1}^{T}\mu ( \langle \vx_t, \vtheta_\star \rangle).
\end{equation}

Exploiting heteroskedasticity requires care: observations with smaller stochastic variance may still be corrupted, and giving them greater weight can increase the influence of their corruption.

\paragraph{Contributions.}
Our central question is,
\begin{center}
\emph{
    Can we design a computationally \textbf{efficient} algorithm for \textbf{heteroskedastic GLBs with adversarial corruptions} that adapts to both dispersion and local curvature while keeping the corruption term additive?
}
\end{center}
Here, computational efficiency refers to online estimator updates and storage that do not grow with the number of past observations, as in recent advances for \textbf{GLB}s~\citep{zhang2025onepass,zhang2023online}; the cost of optimistic arm selection is treated separately.
Specifically, we seek a leading term of $d \sqrt{\sum_t {\dispersioncolor{g(\tau_t)}} {\curvaturecolor{\dmu_{t,\star}}}}$ and an additive corruption term of order $dC$, up to other factors such as $\kappacolor{\kappa}$.
We emphasize that \emph{none} of the prior works on corruption-robust \textbf{GLB}s achieves this ``best-of-both-worlds'' guarantee, as seen in \Cref{tab:results}.
We answer this question affirmatively:

\begin{enumerate}
    \item \textbf{An Efficient, Curvature-Aware Algorithm.} We propose \texttt{HCW-GLB-OMD}, which extends the confidence-weighting of \texttt{CW-OFUL}~\citep{he2022corruption} from the linear design matrix to the Hessian matrix to accurately account for local curvature. We construct corruption-robust confidence sequences using the online mirror descent (OMD) estimator of \texttt{GLB-OMD}~\citep{zhang2025onepass} \textbf{(\Cref{theorem:confidence-sequence})}.
    For a known corruption budget $C$, we prove the following high-probability regret bound:
    \begin{equation}
        \begin{aligned}
        \Reg(T) \lesssim_{\log}{}& d \sqrt{ \sum\nolimits_{t=1}^T {\dispersioncolor{g(\tau_t)}} {\curvaturecolor{\dmu_{t,\star}}}}\\
        &{}+ d^2 {\dispersioncolor{g_{\max}}} {\kappacolor{\kappa}} + d(1 + {\kappacolor{\kappa}}) (C\vee1) ,
        \end{aligned}
    \end{equation}
    for self-concordant \textbf{\textit{heteroskedastic GLBs with adversarial corruptions}} \textbf{(\Cref{theorem:regret})}.
    Here, $\dispersioncolor{g_{\max}}=\max_t \dispersioncolor{g(\tau_t)}$, and we treat $S$, $R_s$, and $L_\mu$ as constants.
    The leading term recovers the known rates for fixed-dispersion logistic/Poisson bandits and heteroskedastic linear bandits (\Cref{tab:results}).
    The estimator uses one OMD update per round and $\gO(d^2)$ storage without retaining past observations.

\item \textbf{Nontrivial Analysis.}
    Although our algorithm combines two established ideas, their analyses do not directly carry over.
    First, the OMD analysis pairs each weighted corruption contribution with the same round's estimation error, enabling control under the Hessian that determines its weight.
    Combining the resulting self-bounding inequality~\eqref{eqn:self-bounding} with a weighted likelihood supermartingale \textbf{(\Cref{lemma:mixture})} yields a corruption-robust, curvature-aware confidence sequence \textbf{(\Cref{theorem:confidence-sequence})}.
    Second, in the regret analysis \textbf{(\Cref{theorem:regret})}, the confidence weights regulate Hessian growth according to directional uncertainty.
    We control this growth with \emph{Hessian Potential Lemma} \textbf{(\Cref{lemma:revised-epl})}, retaining tight dependencies on both dispersion and local curvature in the regret.

    \item \textbf{Regret Lower Bounds.}
    We establish two lower bounds that complement our upper bound.
    One is a local minimax lower bound of $\widetilde\Omega(d\sqrt{\sum_t \dispersioncolor{g(\tau_t)}\curvaturecolor{\dmu_{t,\star}}})$ \textbf{(\Cref{thm:lower-bound})}, and another is a corruption-dependent lower bound of $\Omega(dC)$ for sufficiently large $T$ \textbf{(\Cref{thm:lower-bound-corruption})}.
    Compared with the upper bound, we achieve near-optimality: at most logarithmic gap in the leading $\sqrt{T}$ term and a $\kappacolor{\kappa}$ gap in the corruption term.
\end{enumerate}

\paragraph{Notation.}
For $a, b \in \mathbb{R}$, we let $a \wedge b \coloneqq \min\{a, b\}$ and $a \vee b \coloneqq \max\{a, b\}$. 
$\mathcal{B}^d(r) \coloneqq \{ \vx \in \mathbb{R}^d \mid \lVert \vx \rVert_2 \leq r \}$ denotes the $d$-dimensional Euclidean ball of radius $r$.
For positive semi-definite (PSD) matrices $\mA, \mB \in \mathbb{R}^{d \times d}$, the Mahalanobis norm is denoted as $\lVert \vx \rVert_{\mA} \coloneqq \sqrt{ \vx^\top \mA \vx}$, and $\mA \succeq \mB$ means $\mA - \mB$ is PSD.
We use $\gO$ (or $\lesssim$), $\Omega$ and $\Theta$ to denote asymptotic upper, lower and tight bounds, respectively, while $\widetilde{\gO}$, $\widetilde{\Omega}$, $\widetilde{\Theta}$, and $\lesssim_{\log}$ ignore polylogarithmic factors.

\section{Assumptions}
\label{sec:assumptions}
We first describe the assumptions that we impose throughout our paper.

First, the log-partition function $m: \sR \mapsto \sR$ is assumed to satisfy the following:
\begin{assumption}
\label{assumption:m}
    $m \in C^3(\sR)$ and is strictly convex.
\end{assumption}
This implies that $\mu$ is well-defined and non-decreasing, i.e., $\curvaturecolor{\dmu} \geq 0$.
The next two are standard in the \textbf{GLB} literature~\citep{russac2021glm,lee2024glm}.
Denoting $\gX_{[T]} := \bigcup_{t \in [T]} \gX_t$,
\begin{assumption}
\label{assumption:bound}
    $\gX_{[T]} \subseteq \mathcal{X} \triangleq \gB^d(1)$ and $\vtheta_\star \in \Theta \triangleq \gB^d(S)$ for a known $S \in (0, \infty)$.
\end{assumption}
\begin{assumption}
\label{assumption:link-function}
    There exists $L_\mu \in (0, \infty)$ and $R_s \in [0, \infty)$ such that for all $(\vx, \vtheta) \in \gX \times \Theta$, $\curvaturecolor{\dmu}(\langle \vx, \vtheta \rangle) \leq L_\mu$, and $|\ddmu(z)| \leq R_s \curvaturecolor{\dmu}(z)$ for all $z \in \R$.
\end{assumption}
Here, $L_\mu$ is the Lipschitz constant of $\mu(\cdot)$.
The second inequality is referred to as \textbf{\emph{(generalized) self-concordance}}~\citep{abeille2021logistic,russac2021glm}, originally introduced by \citet{bach2010self} in the analysis of logistic regression to avoid exponential constants.

\section{Our Algorithm: HCW-GLB-OMD}
\label{sec:algorithm}
We present \texttt{Hessian-Confidence Weighted GLB-OMD (HCW-GLB-OMD)}, whose pseudocode is provided in \Cref{algorithm:CW-GLB-OMD}.
This combines Hessian-based confidence weighting (HCW) and an online mirror descent (OMD) estimator.
HCW restricts each observation's influence according to directional uncertainty along the selected arm; OMD incorporates each observation while retaining the local curvature needed in the analysis.
Throughout, we assume the corruption budget $C$ is known, and we defer discussion of unknown $C$ to \Cref{app:unknown-C}.

\begin{algorithm2e}[!t]
	\caption{\label{algorithm:CW-GLB-OMD}\texttt{HCW-GLB-OMD}}
	\KwIn{$\delta \in (0, 1)$, $C, \alpha, \eta, \lambda > 0$}
	\BlankLine
	Initialize variables: $\vtheta_1 \in \Theta$, $\mH_1 = \lambda \mI_d$\;
	\For{$t=1$ \KwTo $T$}{
		Receive $\gX_t \subseteq \gB^d(1)$\;
		
		Construct $\mathcal{C}_t(\delta)$ as described in \textbf{\Cref{theorem:confidence-sequence}}\;
		
		Select an arm $\vx_t \gets \argmax_{\vx \in \gX_t, \vtheta \in \mathcal{C}_t(\delta)} \langle \vx, \vtheta \rangle$\;
        Observe the \emph{corrupted} reward $\tilde{r}_t$ and ${\tau_t}$\;

        Update $\vtheta_{t+1}$ by Eqn.~\eqref{eq:OMD-estimator}
		
		Update $\mH_{t+1} \gets \mH_t + \nabla^2 \tilde{\ell}_t(\vtheta_{t+1})$\;
	}
\end{algorithm2e}

\subsection{Overview of the Algorithm}
At each round $t\in[T]$, \texttt{HCW-GLB-OMD} selects an arm $\vx_t$ optimistically using the estimator $\vtheta_t$ and the confidence set $\gC_t(\delta)$ specified in \Cref{theorem:confidence-sequence}. It observes a corrupted reward $\tilde{r}_t$ together with the dispersion $\dispersioncolor{g(\tau_t)}$, then assigns a confidence weight 
\begin{equation}
\label{eqn:weight}
    {\weightcolor{w_t}} = 1 \wedge \frac{\alpha \dispersioncolor{g(\tau_t)}}{\| \vx_t \|_{\mH_t^{-1}}}
\end{equation}
to the negative log-likelihood with respect to the corrupted reward $\tilde{\ell}_t(\vtheta) \triangleq \frac{{\weightcolor{w_t}}}{\dispersioncolor{g(\tau_t)}} ( m( \langle \vx_t , \vtheta \rangle) - \tilde{r}_t \langle \vx_t, \vtheta\rangle )$.
The algorithm updates the OMD estimator as
\begin{align}\label{eq:OMD-estimator}
    \vtheta_{t+1} \in \argmin_{\vtheta \in \Theta} \bigg[ Q_t(\vtheta) +\frac{1}{2\eta} \lVert \vtheta - \vtheta_t \rVert_{\mH_t}^2 \bigg],
\end{align}
where $Q_t(\vtheta)$ is quadratic approximation of the local likelihood of the observation $\tilde{r}_t$ around $\vtheta_t$ defined as 
\begin{equation}
    Q_t(\vtheta) = \big\langle \nabla\tilde{\ell}_t(\vtheta_t), \vtheta - \vtheta_t \big\rangle
    + \frac{1}{2}\lVert \vtheta - \vtheta_t \rVert_{\nabla^2 \tilde{\ell}_t(\vtheta_t)}^2.
\end{equation}

Finally, the algorithm accumulates the temporally local information along $\vx_t$, using the curvature $\curvaturecolor{\dmu}(\langle \vx_t, \vtheta_{t+1} \rangle)$ evaluated at the updated estimate $\vtheta_{t+1}$ as $\mH_{t+1} = \mH_t + \nabla^2 \tilde{\ell}_t (\vtheta_{t+1})$.

\paragraph{Hessian-based Confidence Weight (HCW).}
This mechanism is key to robustness against adversarial corruption.
The weight ${\weightcolor{w_t}}$ reflects two sources of uncertainty in the observation: the observed dispersion $\dispersioncolor{g(\tau_t)}$ accounts for the known variability of the current observation, while $\lVert \vx_t \rVert_{\mH_t^{-1}}$ quantifies the directional uncertainty along $\vx_t$ based on the curvature-aware information accumulated so far. 
When the unseen directional uncertainty dominates the known dispersion, a corrupted reward can significantly affect the estimate.
To guard against unseen corruptions, HCW limits this influence by downweighting such observations, without identifying which rewards are corrupted.
The parameter $\alpha>0$ balances the relative importance of two sources of uncertainty and controls the degree of downweighting, which is tuned in \Cref{theorem:regret}.

When $\mu(z) = z$ and $\dispersioncolor{g(\tau_t)} = 1$, ${\weightcolor{w_t}}$ reduces to the confidence weight of \citet{he2022corruption} for \emph{linear} bandits. For nonlinear $\mu$, ${\weightcolor{w_t}}$ depends on the local curvature accumulated in $\mH_t$, following the standard recipe for obtaining curvature-dependent regret bounds in \textbf{GLB}s~\citep{abeille2021logistic,zhang2025onepass}.

\paragraph{OMD Estimator.}
We employ OMD for both analytical and computational reasons.
The weight $\weightcolor{w_s}$ at round $s$ is determined by the current Hessian $\mH_s$, which accumulates the local curvature information from previous rounds.
The OMD analysis pairs the corruption $c_s$ at round $s$ with the local estimation error $\vtheta_{s+1} - \vtheta_\star$ measured through the local Hessian $\mH_s$ that determines the weight $\weightcolor{w_s}$.
The Hessian update also records the local curvature at the corresponding estimate $\vtheta_{s+1}$.
Our analysis exploits this temporal alignment to control the corruption contribution through the self-bounding inequality (Eqn. \eqref{eqn:self-bounding}).
In addition, OMD avoids repeatedly solving for batch MLE with cumulating observations, giving computational benefits in both space and time complexities~\citep{zhang2025onepass}.

\subsection{Corruption-Robust, Curvature-Aware Confidence Sequence}

Our first main technical contribution is a corruption-robust confidence sequence that is also curvature-aware:
\begin{theorem}\label{theorem:confidence-sequence}
    Let $\delta \in (0, 1)$, $\alpha >0$, and set $\eta = 1 + R_s S$ and 
    $\lambda = \max \left\{ 14 d \eta R_s^2, 36 \eta^2 \alpha^2 R_s^2 S^2 L_\mu^2, \frac{d}{4S^2} \right\}$.
    For each $t \in [T]$, define the confidence set as 
    \begin{equation}\label{eq:CS}
        \mathcal{C}_t (\delta) \triangleq \left\{ \vtheta \in \Theta : \lVert \vtheta - \vtheta_t \rVert_{\mH_t} \leq \beta_t(\delta) + 2\eta \alpha C \right \},
    \end{equation}
    where $\vtheta_t$ is the OMD estimator (Eqn.~\eqref{eq:OMD-estimator}) and the radius $\beta_t(\delta)$ is given by
    \begin{equation}
        \beta_t(\delta) = \sqrt{2 \eta \log \frac{1}{\delta} \!+\! ( 6 \eta^2 \!+\! \eta) \log \frac{\det(\mH_t)}{\det(\mH_1)} \!+\! 4 \lambda S^2}.
    \end{equation}
    Then, under \Cref{assumption:m,assumption:bound,assumption:link-function}, we have $\mathbb{P}\big( \forall t \geq 1, \vtheta_\star \in \mathcal{C}_t(\delta) \big) \geq 1 - \delta$. 
\end{theorem}
Note that the confidence radius has two components: $\beta_t$ accounts for stochastic uncertainty and regularization, while the \emph{additive} correction $2 \eta \alpha C$ accounts for adversarial corruption.
Also, the confidence set $\gC_t(\delta)$ is a convex ellipsoid with respect to the Hessian $\mH_t$, intersected with the ball $\gB^d(S)$.
Importantly, our radius contains no explicit worst-case curvature $\kappacolor{\kappa}$; this is an improvement from prior work~\citep{yu2025corruption, ye2023corruption} where the radius explicitly scales with $\kappacolor{\kappa}$.

We provide the proof sketch in \Cref{sec:proof-sketch} and highlight the key technical novelties in \Cref{sec:novelty}.
We provide the full detailed proof in \Cref{app:proof-cs}.

\subsection{\texorpdfstring{Proof Sketch of \Cref{theorem:confidence-sequence}}{Proof Sketch of Theorem 1}}
\label{sec:proof-sketch}
The proof proceeds in three steps.

\textbf{Step (i): Estimation error decomposition.}
Due to the form of the GLM (Eqn.~\eqref{eqn:GLM}), a difference of weighted log-likelihoods can be decomposed into a weighted likelihoods of the true rewards and a corruption term:
\begin{equation}
\label{eq:corruption-decompose}
    \tilde{\ell}_s (\vtheta) - \tilde{\ell}_s (\vtheta')
    = {\weightcolor{w_s}} \ell_s(\vtheta) - {\weightcolor{w_s}} \ell_s(\vtheta') + \frac{{\weightcolor{w_s}}}{\dispersioncolor{g(\tau_s)}} c_s \big\langle \vx_s, \vtheta' - \vtheta \big\rangle,
\end{equation}
where $\ell_s(\vtheta) := \frac{m(\langle \vx_s, \vtheta \rangle) - r_s \langle \vx_s, \vtheta \rangle}{\dispersioncolor{g(\tau_s)}}$ is the unweighted negative log-likelihood of the \emph{true} reward $r_s$.

Combining this with the one-step lemma for OMD and self-concordance \textbf{(\Cref{lemma:inverse-regret})}, we obtain:
\begin{equation}\label{eqn:OMD}
\begin{aligned}
    \lVert \vtheta_{t+1} - \vtheta_\star \rVert_{\mH_{t+1}}^2 &\leq  2\eta \underbrace{\left( \sum_{s=1}^{t} {\weightcolor{w_s}} {\ell}_s ( \vtheta_\star) - \sum_{s=1}^{t} {\weightcolor{w_s}} {\ell}_s (\vtheta_{s+1}) \right)}_{(*)} \\
    &\qquad + 4\lambda S^2 - \frac{2}{3}\sum_{s=1}^{t} \lVert \vtheta_s - \vtheta_{s+1} \rVert_{\mH_s}^2 \\
    &\qquad + 2\eta \underbrace{\sum_{s=1}^{t} \frac{{\weightcolor{w_s}}}{\dispersioncolor{g(\tau_s)}} c_s \langle \vx_s, \vtheta_{s+1} - \vtheta_\star \rangle}_{(**)},
    \end{aligned}
\end{equation}
where $(*)$ is the weighted analogue of the ``inverse regret'' in \citet{zhang2025onepass}, and $(**)$ is the corruption contribution additively separated from the inverse regret $(*)$.
Importantly, the corruption $c_s$ at round $s$ is paired with the corresponding local update $\vtheta_{s+1}$.

\textbf{Step (ii): Bounding $(*)$.}
We introduce the mix-loss $m_s(P_s) = -\log ( \mathbb{E}_{\vtheta \sim P_s} [ \exp ( - \ell_s (\vtheta) )] )$, and decompose the inverse regret $(*)$ as
\begin{equation}
     \begin{aligned}
     (*) & = \underbrace{\sum_{s=1}^{t} {\weightcolor{w_s}} {\ell}_s ( \vtheta_\star) - \sum_{s=1}^{t} {\weightcolor{w_s}} m_s(P_s)}_{\triangleq Y_t}\\
     & \qquad + \underbrace{\sum_{s=1}^{t} {\weightcolor{w_s}} m_s(P_s)- \sum_{s=1}^{t} {\weightcolor{w_s}} {\ell}_s (\vtheta_{s+1})}_{\triangleq Z_t},
     \end{aligned}
\end{equation}
where $Y_t$ is the weighted mixed likelihood-ratio and $Z_t$ is the weighted mix-loss gap.
With $P_s = \mathcal{N} ( \vtheta_{s}, \gamma \mH_{s}^{-1})$ for some $\gamma > 0$ to be determined later, \textbf{\Cref{lemma:mixture}} gives a time-uniform bound on $Y_t$ via a \emph{supermartingale argument}, which we elaborate in \Cref{sec:novelty}.
The remaining $Z_t$ is bounded by the weighted mix-loss concentration inequality \textbf{(\Cref{lemma:mixture-ineq})}.
Consequently, we obtain
\begin{equation}
    2\eta (*) + 4\lambda S^2 - \frac{2}{3}\sum_{s=1}^{t} \lVert \vtheta_s - \vtheta_{s+1} \rVert_{\mH_s}^2 \leq \beta_{t+1}(\delta)^2
\end{equation}
for all $t \geq 1$, with probability at least $1 - \delta$.

\textbf{Step (iii): Bounding $(**)$.}
By Cauchy--Schwarz and definition of the confidence weight $\weightcolor{w_s}$ (Eqn.~\eqref{eqn:weight}):
\begin{align}
    (**) & \leq \sum_{s=1}^t \frac{{\weightcolor{w_s}}}{\dispersioncolor{g(\tau_s)}} |c_s| \cdot \bignorm{\vx_s}_{\mH_s^{-1}} \bignorm{\vtheta_{s+1} - \vtheta_\star}_{\mH_s}\nonumber \\
    & \leq \alpha \sum_{s=1}^t |c_s|\cdot \bignorm{\vtheta_{s+1} - \vtheta_\star}_{\mH_s}\\
    & \leq \alpha \sum_{s=1}^t |c_s|\cdot \underbrace{\bignorm{\vtheta_{s+1} - \vtheta_\star}_{\mH_{s+1}}}_{\triangleq X_{s+1}}\tag{$\mH_{s+1} \succeq \mH_s$} \\
    & \leq \alpha C \max_{s \in [t+1]} X_s.
\end{align}
Substituting this and the result of \textbf{Step (ii)} into \textbf{Step (i)} gives a maximum-involved \textbf{\emph{self-bounding inequality}}
\begin{equation}\label{eqn:self-bounding}
    X_t^2 \leq \beta_t(\delta)^2 + 2\eta\alpha C \max_{s \in [t]} X_s.
\end{equation}
Letting $M_t:=\max_{s\in [t]}X_s$ and using the monotonicity of
$\beta_t(\delta)$ gives $M_t^2\leq\beta_t(\delta)^2+2\eta\alpha C M_t $.
Solving this quadratic inequality yields $X_t\leq M_t
\leq \beta_t(\delta)+2\eta\alpha C$,
which proves \Cref{theorem:confidence-sequence}.
\qed

\subsection{Technical Novelties Worth Noting}
\label{sec:novelty}

\paragraph{Role of OMD: aligning weights with local updates.}
Extending the confidence-weighting of \citet{he2022corruption} for linear bandits to GLBs requires more than replacing the design matrix with the Hessian.
The main difficulty is the \emph{time-index mismatch}: $\weightcolor{w_s}$ is determined by $\mH_s$ at round $s$, whereas its corruption contribution is evaluated at the later global estimate $\hat{\vtheta}_t$.
Although \citet{he2022corruption} also use a batched estimate, their linear design matrix does not depend on the estimate.
For GLBs, however, the Hessian depends on the estimate.
In particular, a standard batched MLE analysis~\citep{lee2024glm,lee2024logistic} produces the cumulative corruption contribution $\sum_{s=1}^t\frac{\weightcolor{w_s}}{\dispersioncolor{g(\tau_s)}}c_s\langle\vx_s,\hat{\vtheta}_t-\vtheta_\star\rangle$. Thus, exploiting the local curvature information at $s$ requires an additional comparison between $\hat \vtheta_s$ and $\hat \vtheta_t$.

The OMD analysis instead produces a corruption contribution of $\sum_{s=1}^{t} \frac{\weightcolor{w_s}}{\dispersioncolor{g(\tau_s)}} c_s \langle \vx_s, \vtheta_{s+1} - \vtheta_\star \rangle$, keeping the confidence weight $\weightcolor{w_s}$, arm $\vx_s$, local estimate $\vtheta_{s+1}$ and Hessian $\mH_s$ \emph{aligned} at each round.
This temporal alignment is what allows our confidence analysis to control corruption by exploiting the local curvature that determines each weight.

\paragraph{Weighted likelihood concentration.}
The second difficulty arises from the confidence weights themselves. The confidence-sequence analysis of \citet{zhang2025onepass} relies on an unweighted likelihood-ratio martingale ~\citep{emmenegger2023likelihood,lee2024glm,kirschner2025sequential}, which does not directly cover adaptive confidence-weights.
Under HCW, conditional on the past, arm, and dispersion, $\weightcolor{w_s} \leq 1$ is fixed; hence, the mapping $x\mapsto x^{\weightcolor{w_s}}$ is concave on $[0,\infty)$. Thus, by the (conditional) Jensen's inequality\footnote{We consider its ``expectation version'': $\E[f(X)] \leq f(\E[X])$ for any concave $f$ and random variable $X$.} and Fubini's theorem to deal with the mixture distribution $P_s$, we show that \textbf{\emph{$M_t \triangleq \exp(Y_t)$} is a nonnegative supermartingale} \textbf{(\Cref{lemma:mixture})}.
Thus, by applying Ville's inequality~\citep{Ville1939}, we have that $\sP(Y_t \leq \log\frac{1}{\delta}, \ \forall t \geq 1) \geq 1 - \delta$.
To control $Z_t$, we extend \citet[Lemmas 6--7]{zhang2025onepass} to predictable weights $\weightcolor{w_s} \in [0,1]$ by deriving the weighted variational inequality and a Gaussian remainder bound \textbf{(\Cref{lemma:mixture-ineq})}.

\paragraph{Self-bounding control of corruption.}
The key point is that the online Hessian update adds a positive semidefinite matrix and thus induces a monotone L\"{o}wner ordering across $s \in [t]$: $\mH_{s+1}-\mH_s=\nabla^2\tilde\ell_s(\vtheta_{s+1})\succeq0$.
Thus, $\mH_{s+1}\succeq\mH_s$, and the local estimation error $\|\vtheta_{s+1}-\vtheta_\star\|_{\mH_s}$ is bounded by $\|\vtheta_{s+1}-\vtheta_\star\|_{\mH_{s+1}}=X_{s+1}$.
Combined with confidence weighting, this bounds the corruption contribution by $\alpha C\max_{s\in[t+1]}X_s$ and yields the self-bounding inequality in Eqn.~\eqref{eqn:self-bounding}.
Taking the maximum over time gives a quadratic inequality, whose solution yields the \emph{additive} $2\eta\alpha C$ correction to the confidence radius.

\section{Regret Upper Bound}
\label{sec:upper-bound}

Building upon this confidence sequence, we present the regret upper bound attained by \texttt{HCW-GLB-OMD}:
\begin{theorem}\label{theorem:regret}
    Let $\delta \in (0, 1)$, set $\eta, \lambda$ as in \Cref{theorem:confidence-sequence} and $\alpha = \Theta\big( \sqrt{d} \:(C \vee 1)^{-1} \big)$.
    Under \Cref{assumption:m,assumption:link-function,assumption:bound}, \texttt{HCW-GLB-OMD} attains the following regret upper bound with probability at least $1 - \delta$:
    \begin{equation}
        \widetilde{\gO}\left( d \sqrt{ \sum_{t=1}^T {\dispersioncolor{g(\tau_t)}} {\curvaturecolor{\dmu_{t,\star}}}} + d ( 1+ {\kappacolor{\kappa}}) (C\vee1) + d^2 {\dispersioncolor{g_{\max}}} {\kappacolor{\kappa}} \right).
    \end{equation}
\end{theorem}
The local curvature ${ \curvaturecolor{\dmu_{t,\star}} \triangleq \curvaturecolor{\dmu}(\langle \vx_{t,\star}, \vtheta_\star \rangle)}$ reflects the instance-dependent statistical hardness of the given \textbf{GLB} instance. 
Together with the observed dispersion $\dispersioncolor{g(\tau_t)}$, recall that $\dispersioncolor{g(\tau_t)}\curvaturecolor{\dmu_{t,\star}}$ is precisely the conditional variance at the optimal arm.
Hence, the leading term $d\sqrt{\sum_{t=1}^T \dispersioncolor{g(\tau_t)}\curvaturecolor{\dmu_{t,\star}}}$ is precisely the first-order regret~\citep{bakhtiari2025eluder}.
In particular, the worst-case curvature parameter $\kappacolor{\kappa}$ does not multiply the leading term, and appears only in the additive transient and corruption-dependent terms.

As summarized in \Cref{tab:results}, our bound recovers the known optimal regret bounds across several bandit instances.
When $\mu(z)=z$, $C=0$, and $\dispersioncolor{g(\tau_t)}=\sigma_t^2$, its leading term matches the optimal bound $\widetilde{\gO}(d\sqrt{\sum_t\sigma_t^2})$ for heteroskedastic linear bandits~\citep{zhou2022variance}.
For linear bandits with $\dispersioncolor{g(\tau_t)}\equiv1$ and adversarial corruption, \Cref{theorem:regret} reduces to $\widetilde{\gO}(d\sqrt{T}+dC)$~\citep{bogunovic2021corrupted}, which is minimax optimal.

For nonlinear \textbf{GLB}s, the leading term retains the state-of-the-art instance-dependent, curvature-aware rate of corruption-free GLBs~\citep{abeille2021logistic,zhang2025onepass}, while allowing time-varying dispersion and adversarial corruption.
Compared with existing bounds of heteroskedastic~~\citep{zhao2023optimal} and corruption-robust~\citep{yu2025corruption, ye2023corruption} GLBs, our bound removes the multiplicative $\kappacolor{\kappa}$ factor from the leading term, and adapts to local curvature, yielding up to a $\kappacolor{\kappa^{3/2}}$ improvement when $\curvaturecolor{\dmu_{t,\star}}\asymp1/\kappacolor{\kappa}$.
Moreover, our corruption term also improves the $\kappacolor{\kappa^2} C $ dependence of \citet{ye2023corruption} to $\kappacolor{\kappa} C$, matching \cite{yu2025corruption}.
The remaining gap to our lower bounds is therefore confined to the transient term and the additional $\kappacolor{\kappa}$ factor in the corruption-dependent term.

In \Cref{app:experiments}, we demonstrate the numerical efficacy of \texttt{HCW-GLB-OMD} across various bandit instances, including its adaptation to dispersion and local curvature and robustness to adversarial corruption.

\paragraph{Proof Sketch of \Cref{theorem:regret}.}

We outline the main arguments here, deferring the full proof to \Cref{app:proof-regret}.
Throughout, we work on the event
$\mathcal E_\delta=\{\forall t  \geq 1, \vtheta_\star\in\mathcal C_t(\delta)\}$, which satisfies $\sP(\gE_\delta) \geq 1-\delta$ by \Cref{theorem:confidence-sequence}.

Following the regret decomposition of \citet{abeille2021logistic}, optimism and second-order Taylor expansion of the regret around $\vtheta_\star$ decompose the regret into first-order estimation errors and quadratic remainders.
We then partition the time horizon $T$ into unweighted ($\weightcolor{w_t}=1$) and downweighted ($\weightcolor{w_t}<1$) rounds, yielding the four terms ($S_0$, $S_1$, $S_2$, and $S_3$) in \textbf{\Cref{lemma:decomposition}}.
Here, $S_0$ is the first-order contribution on unweighted rounds and yields the leading variance-dependent term, while $S_1$ is its counterpart on downweighted rounds.
The terms $S_2$ and $S_3$ collect the quadratic errors on these two sets of rounds, arising from the Taylor remainder and curvature approximation.

Our key technical ingredient for bounding these terms is the \emph{Hessian Potential Lemma} \textbf{(\Cref{lemma:revised-epl})}.
For each round, define the weighted Hessian potential:
\[
    q_t
    \coloneqq
    \frac{\weightcolor{w_t}}{\dispersioncolor{g(\tau_t)}}
    \curvaturecolor{\dmu}(\langle \vx_t,\vtheta_{t+1}\rangle)
    \|\vx_t\|_{\mH_t^{-1}}^2 .
\]
Exploiting the coupling between $\weightcolor{w_t}$ and the directional uncertainty $\| \vx_t \|_{\mH_t^{-1}}$, we bound the one-step trace growth by $\mathrm{tr}(\mH_{t+1}) - \mathrm{tr}(\mH_t) \leq \alpha L_\mu \sqrt{\mathrm{tr}(\mH_t)}$.
Unrolling this recursion with the matrix determinant lemma, we bound both the log-determinant ratio and $\sum_t q_t$ by $\widetilde{\gO}(d)$.
Applying this lemma gives the bounds on $S_1$, $S_2$, and $S_3$ as specified in \textbf{\Cref{lemma:epl-results}}.

For $S_0$, Cauchy--Schwarz separates the variance along the selected arms from it.
Following the self-concordance argument in \citet{abeille2021logistic}, we relate the curvature at the selected arms to that at the optimal arms:
\[
    \sum_{t=1}^T \dispersioncolor{g(\tau_t)}\curvaturecolor{\dmu}(\langle\vx_t,\vtheta_\star\rangle)
    \leq\sum_{t=1}^T \dispersioncolor{g(\tau_t)}\curvaturecolor{\dmu_{t,\star}}
       +\dispersioncolor{g_{\max}}R_s\Reg(T).
\]
Substituting this into the first-order bound gives a self-bounding quadratic inequality in $\Reg(T)$.
Solving this with $\alpha=\Theta(\sqrt d/(C\vee1))$ yields \Cref{theorem:regret}.

\section{Regret Lower Bounds}
\label{sec:lower-bound}
We complement our upper bound with two lower bounds.
Throughout, $\inf_\pi$ denotes an infimum over all possible algorithms $\pi$.

The first provides the lower bound that matches the corruption-free $\sqrt{T}$ term, up to logarithmic factors, in a local-minimax sense of \citet{abeille2021logistic,zhang2025ellipsoid,lee2025gl-lowpopart}.
Its implicit constants may depend on $S$ and $R_s$.
\begin{theorem}[Local Minimax Lower Bound]\label{thm:lower-bound}
    Fix $T\geq2$, $d\geq2$, and $\vtheta_\star\in\Theta$ with $\lVert\vtheta_\star\rVert_2\geq1$, and write $\kappacolor{\kappa_\star}=1/\curvaturecolor{\dmu}(\lVert\vtheta_\star\rVert_2)$.
    For any valid dispersion sequence $\{\tau_t\}$ satisfying
    \begin{equation}
    \label{eqn:tau-condition}
        \sum_{t=1}^T \dispersioncolor{g(\tau_t)}
        \gtrsim \dispersioncolor{g_{\max}}
          +d^2\dispersioncolor{g_{\max}^2}\kappacolor{\kappa_\star},
    \end{equation}
    there exist arm sets $\{\gX_t\}_{t=1}^T$ and a local radius $\varepsilon_T>0$ such that $\curvaturecolor{\dmu_{t,\star}}=\kappacolor{\kappa_\star}^{-1}$ for all $t$, and
    \begin{equation}
        \inf_\pi\sup_{\substack{\tilde\vtheta_\star\in\Theta\\
            \lVert\tilde\vtheta_\star-\vtheta_\star\rVert_2\leq\varepsilon_T}}
            \E_{\tilde\vtheta_\star}\Reg^\pi(T)
            \gtrsim d\sqrt{\frac{\sum_{t=1}^T \dispersioncolor{g(\tau_t)}\curvaturecolor{\dmu_{t,\star}}}{\log T}}.
    \end{equation}
    If $\dispersioncolor{g(\tau_t)}=\dispersioncolor{g}$, taking $\gX_t=\gS^{d-1}(1)$ yields the improved lower bound $\Omega(d\sqrt{\dispersioncolor{g}T/\kappacolor{\kappa_\star}})$ with $\varepsilon_T\asymp(d^2\dispersioncolor{g}\kappacolor{\kappa_\star}/T)^{1/4}$.
\end{theorem}

This extends the local minimax lower bound for logistic bandits of \citet[Theorem 2]{abeille2021logistic} to self-concordant \textbf{GLB}s in the regime $\lVert\vtheta_\star\rVert_2\geq1$.
Condition~\eqref{eqn:tau-condition} ensures that a group of rounds with comparable dispersions contains enough observations for the local lower-bound construction and carries an $\Omega(1/\log T)$ fraction of the total dispersion.
The resulting parameter perturbations remain feasible in $\Theta$ and small enough that self-concordance keeps their optimal-arm curvatures within a constant factor of $\curvaturecolor{\dmu}(\|\vtheta_\star\|_2)$.
For constant dispersion $\dispersioncolor{g(\tau_t)}=\dispersioncolor{g}$, the condition reduces to $T\gtrsim1+d^2\dispersioncolor{g}\kappacolor{\kappa_\star}$, the same requirement as \citet[Theorem 2]{abeille2021logistic}.

The second lower bound \emph{nearly} matches the corruption term $dC$, up to a $\kappacolor{\kappa}$ factor:
\begin{theorem}[Corruption-Dependent Lower Bound]\label{thm:lower-bound-corruption}
    Let $d\geq2$ and $C>0$, and fix any sequence of valid dispersions.
    Then there exists a fixed arm set $\gX\subseteq\gB^d(1)$ such that for $T \gtrsim dC$,
    \begin{equation}
        \inf_\pi \sup_{\tilde\vtheta_\star\in\Theta} \E_{\tilde{\vtheta}_\star}\Reg^\pi(T)
        \gtrsim dC.
    \end{equation}
\end{theorem}

\paragraph{Tightness.}
\Cref{thm:lower-bound,thm:lower-bound-corruption} match the leading rates of \Cref{theorem:regret} in their stated regimes, up to logarithmic factors and $\kappacolor{\kappa}$ in the corruption term.
The former recovers the logistic local minimax rate~\citep[Theorem 2]{abeille2021logistic} and improves the logarithmic loss of the prefixed-variance rate of \citet[Theorem 4.1]{he2025lowerbound} under Gaussian feedback.
The latter matches the $\Omega(dC)$ lower bound for linear bandits~\citep[Theorem 3]{bogunovic2021corrupted}.
Both also apply to Poisson bandits; the transient terms remain unmatched.

\paragraph{$\kappacolor{\kappa}$ in the Transient Terms.}
We now examine the worst-case curvature $\kappacolor{\kappa}$ in the additive terms of \Cref{theorem:regret}, $\widetilde{\gO}\left(d^2\kappacolor{\kappa} \dispersioncolor{g_{\max}}+d(1+\kappacolor{\kappa})(C\vee1)\right).$
For corruption-free logistic bandits, a $\kappacolor{\kappa}$-dependent transient is avoidable on the unit ball~\citep[Theorem 3]{abeille2021logistic}, while suitable finite arm sets admit long initial periods of linear Bayesian regret~\citep[Theorem 8 and Proposition 11]{dong2019logistic}.
The optimal scaling of our $d^2\kappacolor{\kappa}\dispersioncolor{g_{\max}}$ term remains open.

For corruption, we show that $\kappacolor{\kappa}$-dependence can arise even on a two-arm logistic instance (\textbf{\Cref{prop:kappa-corruption-obstruction}} in \Cref{app:curvature-gap}).
The construction uses $\gX=\{\vzero,\ve_1\}$ and $\vtheta_\star=-S\ve_1$.
Arm $\vzero$ has mean $1/2$, while $\ve_1$ has success probability $\mu(-S)=\Theta(1/\kappacolor{\kappa})$ and a reward gap of at least $1/4$.
Thus a budget $C$ can erase its positive rewards over order $\kappacolor{\kappa}C$ rounds with constant probability.
For \texttt{HCW-GLB-OMD} with $\vtheta_1=\vzero$ and $\alpha=\sqrt d/C$, these zero observations decrease the OMD estimate along $\ve_1$ and the local curvature used in the Hessian updates.
The weighted Hessian grows too slowly for the confidence set to exclude parameters with a positive first coordinate, so optimism continues to favor $\ve_1$ over $\vzero$.
This yields $\Omega(\kappacolor{\kappa}C)$ expected regret for $T\gtrsim\kappacolor{\kappa}C$.
Whether a different algorithm can attain $\widetilde{\gO}(dC)$ corruption dependence remains open.

\paragraph{Proof Sketch of \Cref{thm:lower-bound}.}
We apply the peeling technique of \citet[Theorem 4.1]{he2025lowerbound} to group rounds with comparable dispersions.
Condition~\eqref{eqn:tau-condition} guarantees a group admitting the local construction and carrying an $\Omega(1/\log T)$ fraction of the total dispersion.
We retain all $d$ coordinates in this group, setting $\gX_t=\gS^{d-1}(1)$ there and $\gX_t=\{\vtheta_\star/\|\vtheta_\star\|_2\}$ elsewhere.
The latter rounds reveal no information about the unknown signs and incur no regret.
On the selected rounds, we extend the local minimax lower bound of \citet[Theorem 2]{abeille2021logistic} to self-concordant \textbf{GLB}s using the integral representation of KL between GLMs and self-concordance~\citep{lee2024logistic,lee2024glm}.
Combining this with the group's dispersion mass yields a lower bound of $\frac{d}{\sqrt{\log T}}\sqrt{\sum_{t=1}^T\dispersioncolor{g(\tau_t)}\curvaturecolor{\dmu_{t,\star}}}$.
The full proof is in \Cref{app:proof-lower-bound}.

\paragraph{Proof Sketch of \Cref{thm:lower-bound-corruption}.}
We construct $d$ instances with the same arm set but different optimal arms.
Our key observation is that stochastic ordering of the exponential family allows the adversary to couple optimal-arm rewards with smaller rewards having the suboptimal-arm law.
This makes the instances indistinguishable to the learner.
If their common reward gap is $\Delta$, the resulting regret is of order $T\Delta$, while averaging over the optimal-arm index gives expected corruption $T\Delta/d$.
For sufficiently large $T$, we choose $\Delta\asymp dC/T$ with a sufficiently small constant and enforce the budget by stopping corruption before exceeding $C$.
Markov's inequality then shows that this only loses a constant fraction of the expected regret, yielding the claimed $\Omega(dC)$ lower bound.
The full proof is in \Cref{app:lower-bound-corruption}.

\section{Conclusion}
\label{sec:conclusion}

In this work, we investigate self-concordant, \textbf{\textit{heteroskedastic GLBs with adversarial corruptions}}.
We proposed \texttt{HCW-GLB-OMD}, a computationally efficient algorithm requiring only $\gO(1)$ space and time complexity per round while achieving a regret upper bound of $\widetilde{\gO}\left( d \sqrt{\sum_t {\dispersioncolor{g(\tau_t)}} {\curvaturecolor{\dmu_{t,\star}}}} + d^2 {\dispersioncolor{g_{\max}}}{\kappacolor{\kappa}} + d(1 + {\kappacolor{\kappa}} ){(C\vee1)}\right)$.
We complemented this with a local stochastic lower bound of $\widetilde{\Omega}(d \sqrt{\sum_t {\dispersioncolor{g(\tau_t)}} {\curvaturecolor{\dmu_{t,\star}}}})$ and a corruption-dependent lower bound of $\Omega(d C)$ for large $T$, which establish near-optimality of our algorithm across logistic, Poisson, and heteroskedastic linear bandits.
Further future directions are detailed in \Cref{app:future}.

\label{page:main-end}

\bibliographystyle{plainnat}
\bibliography{references}

\clearpage
\onecolumn
\appendix

{\small\tableofcontents}
\newpage

\crefalias{Section}{Appendix}
\crefalias{Subsection}{Appendix}
\crefalias{section}{appendix} %
\crefalias{subsection}{appendix} %

\newpage
\section{\texorpdfstring{Proof of \Cref{theorem:confidence-sequence}: Corruption-Robust Confidence Sequence}{Proof of Theorem~\ref{theorem:confidence-sequence}: Corruption-Robust Confidence Sequence}}
\label{app:proof-cs}

Since we choose $\lambda \geq 36 \eta^2 \alpha^2 R_s^2 S^2 L_\mu^2$, \Cref{lemma:inverse-regret} implies that
\begin{align}
    \lVert \vtheta_{t+1} - \vtheta_\star \rVert_{\mH_{t+1}}^2 & \leq 2\eta \bigg( \underbrace{\sum_{s=1}^{t} \tilde{\ell}_s ( \vtheta_\star) - \sum_{s=1}^{t} \tilde{\ell}_s (\vtheta_{s+1})}_{\text{inverse regret}} \bigg) +  4\lambda S^2 - \frac{2}{3} \sum_{s=1}^{t} \lVert \vtheta_s - \vtheta_{s+1} \rVert_{\mH_s}^2.
\end{align}
Now, we focus on the inverse regret term. First, we separate the corruption-related term from the inverse regret term as:
\begin{align}
    \tilde{\ell}_s(\vtheta_1) - \tilde{\ell}_s(\vtheta_2) & = \frac{{\weightcolor{w_s}} (m(\langle \vx_s, \vtheta_1 \rangle) - \tilde{r}_s \langle \vx_s, \vtheta_1 \rangle)}{{\dispersioncolor{g(\tau_s)}}} - \frac{{\weightcolor{w_s}} (m(\langle \vx_s, \vtheta_2 \rangle) - \tilde{r}_s \langle \vx_s, \vtheta_2 \rangle)}{{\dispersioncolor{g(\tau_s)}}}\\
    & = {\weightcolor{w_s}} \ell_s(\vtheta_1) - {\weightcolor{w_s}} \ell_s(\vtheta_2) + \frac{{\weightcolor{w_s}}}{{\dispersioncolor{g(\tau_s)}}} c_s \langle \vx_s , \vtheta_2 - \vtheta_1 \rangle.
\end{align}
Next, by adding and subtracting weighted mix loss, we decompose the weighted log-likelihood ratio as:
\begin{align}
    \sum_{s=1}^{t} \tilde{\ell}_s ( \vtheta_\star) - \sum_{s=1}^{t} \tilde{\ell}_s (\vtheta_{s+1}) & = \sum_{s=1}^{t} {\weightcolor{w_s}} \ell_s ( \vtheta_\star) - \sum_{s=1}^{t} {\weightcolor{w_s}} \ell_s (\vtheta_{s+1}) + \sum_{s=1}^{t} \frac{{\weightcolor{w_s}}}{{\dispersioncolor{g(\tau_s)}}} c_s \big\langle \vx_s , {\vtheta_{s+1} - \vtheta_\star} \big\rangle\\
    & = \underbrace{\sum_{s=1}^{t} {\weightcolor{w_s}} \ell_s ( \vtheta_\star) - \sum_{s=1}^{t} {\weightcolor{w_s}} m_s(P_s)}_{\triangleq Y_t} + \underbrace{\sum_{s=1}^{t} {\weightcolor{w_s}} m_s(P_s) - \sum_{s=1}^{t} {\weightcolor{w_s}} \ell_s (\vtheta_{s+1})}_{\triangleq Z_t}\nonumber\\
    & \qquad + \sum_{s=1}^{t} \frac{{\weightcolor{w_s}}}{{\dispersioncolor{g(\tau_s)}}} c_s \big\langle \vx_s, {\vtheta_{s+1} - \vtheta_\star} \big\rangle, \label{eq:inverse-regret-expansion}
\end{align}
where $m_s(P) = - \log \left( \mathbb{E}_{\vtheta \sim P} \left[ \exp ( - {\ell}_s (\vtheta) )\right] \right)$ for a distribution $P$.
We set $P_s = \mathcal{N} ( \vtheta_{s}, \gamma \mH_{s}^{-1})$ as a multivariate normal distribution with $\gamma = \frac{3}{2}\eta$, which is $\mathcal{F}_s$-measurable, then we characterize the behavior of stochastic process $Y_t$ applying Lemma \ref{lemma:mixture}: with probability at least $1-\delta$, for all $t \geq 1$,
\begin{equation}\label{eq:term-a}
    Y_t  \leq \log\frac{1}{\delta}.
\end{equation}
For the term $Z_t$, since we set $\lambda \geq 14 d \eta R_s^2 \geq 64 d \gamma R_s^2 / 7$, Lemma \ref{lemma:mixture-ineq} yields
\begin{equation}\label{eq:term-b}
    Z_t \leq \sum_{s=1}^{t} \frac{1}{2\gamma} \lVert \vtheta_s - \vtheta_{s+1} \rVert_{\mH_s}^2 + \left( 2\gamma + \frac{1}{2} \right) \log \frac{\mathrm{det} (\mH_{t+1})}{\mathrm{det}(\mH_1)}.
\end{equation}
Substituting these results into Eqn.~\eqref{eq:inverse-regret-expansion}, we get
\begin{align}
    \lVert \vtheta_{t+1} - \vtheta_\star \rVert_{\mH_{t+1}}^2 & \leq 2\eta \log \frac{1}{\delta} + (4\gamma + 1) \eta \log \frac{\mathrm{det}(\mH_{t+1})}{\mathrm{det} (\mH_1)}  + 4\lambda S^2 \nonumber\\
    & \quad + \frac{\eta}{\gamma} \sum_{s=1}^{t} \lVert \vtheta_s - \vtheta_{s+1} \rVert_{\mH_s}^2 - \frac{2}{3} \sum_{s=1}^{t} \lVert \vtheta_s - \vtheta_{s+1} \rVert_{\mH_s}^2\nonumber\\
    & \quad + 2\eta \sum_{s=1}^{t} \frac{{\weightcolor{w_s}}}{{\dispersioncolor{g(\tau_s)}}} c_s \big\langle \vx_s, \vtheta_{s+1} - \vtheta_\star \big\rangle.%
\end{align}
By setting $\gamma = \frac{3}{2} \eta$, we get
\begin{align}
    \lVert \vtheta_{t+1} - \vtheta_\star \rVert_{\mH_{t+1}}^2 & \leq \underbrace{ 2\eta \log \frac{1}{\delta} + (6 \eta^2 + \eta) \log \frac{\mathrm{det}(\mH_{t+1})}{\mathrm{det} (\mH_1)}  + 4\lambda S^2}_{\eqqcolon \beta_{t+1}^2(\delta)} \\
    & \quad + 2 \eta \sum_{s=1}^{t} \frac{{\weightcolor{w_s}}}{{\dispersioncolor{g(\tau_s)}}} c_s \big\langle \vx_s, \vtheta_{s+1} - \vtheta_\star \big\rangle.
\end{align}
Now, we focus on the corruption term. The term is upper bounded as
\begin{align}
    \sum_{s=1}^{t} \frac{{\weightcolor{w_s}}}{{\dispersioncolor{g(\tau_s)}}} c_s \big\langle \vx_s, \vtheta_{s+1} - \vtheta_\star \big\rangle & \leq \sum_{s=1}^{t} \frac{{\weightcolor{w_s}}}{{\dispersioncolor{g(\tau_s)}}} |c_s| \cdot \lVert \vx_s \rVert_{\mH_s^{-1}} \lVert \vtheta_{s+1} - \vtheta_\star \rVert_{\mH_s} \tag{Cauchy-Schwarz inequality}\\
    & \leq \sum_{s=1}^{t} \frac{{\weightcolor{w_s}}}{{\dispersioncolor{g(\tau_s)}}} |c_s| \cdot \lVert \vx_s \rVert_{\mH_s^{-1}} \lVert \vtheta_{s+1} - \vtheta_\star \rVert_{\mH_{s+1}}\tag{$\mH_{s} \preceq \mH_{s+1}$} \\
    & \leq \sum_{s=1}^{t} \alpha |c_s| \cdot  \lVert \vtheta_{s+1} - \vtheta_\star \rVert_{\mH_{s+1}} \tag{Definition of $\weightcolor{w_s}$}\\
    & \leq \alpha \left( \sum_{s=1}^{t} |c_s|\right) \cdot \max_{s \in [t]} \lVert \vtheta_{s+1} - \vtheta_\star \rVert_{\mH_{s+1}} \\
    & \leq \alpha C \cdot \max_{s \in [t]} \lVert \vtheta_{s+1} - \vtheta_\star \rVert_{\mH_{s+1}}.
\end{align}
Define a random variable $X_s = \lVert \vtheta_{s} - \vtheta_\star \rVert_{\mH_{s}}$ and an event
\begin{equation}
    \mathcal{E}_\delta \coloneqq \left\{ \forall t \geq 1, \quad X_t^2 \leq \beta_t^2(\delta) + 2 \eta \sum_{s=1}^{t-1} \frac{{\weightcolor{w_s}}}{{\dispersioncolor{g(\tau_s)}}} c_s \langle \vx_s, \vtheta_{s+1} - \vtheta_\star \rangle \right\},
\end{equation}
then by the above inequality under the event $\mathcal{E}_\delta$,
\begin{equation}
    X_t^2 \leq \beta_t^2(\delta) + 2\eta \alpha C \cdot \max_{s \in [t]} X_s
\end{equation} for any $t \geq 1$.
Let $M_t = \max_{s\in [t]} X_s$. Since $M_s \leq M_t$ and $\beta_s(\delta)\leq\beta_t(\delta)$ for $s\leq t$, 
\begin{equation}
    X_s^2 \leq \beta_s^2(\delta) + 2\eta \alpha C \cdot M_s \leq \beta_t^2(\delta) + 2\eta \alpha C \cdot M_t
\end{equation}
holds for all $s \in [t]$.
Taking the maximum over $s\in[t]$ gives
\begin{equation}M_t^2\leq\beta_t^2(\delta)+2\eta\alpha C \cdot M_t.
\end{equation}
By \Cref{lemma:solve-ineq}, $X_t\leq M_t\leq\beta_t(\delta)+2\eta\alpha C$. This holds for all $t\geq1$ simultaneously on $\mathcal E_\delta$, whose probability is at least $1-\delta$.
\qed

\subsection{Supporting Lemmas}
\begin{lemma}[Lemma 1 in \citet{zhang2025onepass}]\label{lemma:update-rule}
    Let $f : \Theta \to \mathbb{R}$ be a differentiable convex function on a convex set $\Theta$ and $A \in \mathbb{R}^{d \times d}$ be a symmetric positive definite matrix.
    Then, for any $u \in \Theta$ and $\eta >0$,  the update rule $\vtheta_{t+1} = \argmin_{\theta \in \Theta} f(\theta) + \frac{1}{2\eta} \lVert \theta - \theta_t \rVert_A^2$ satisfies
    \begin{equation}
        \lVert \theta_{t+1} - u \rVert_A^2 \leq 2 \eta \big\langle \nabla f(\theta_{t+1}), u - \theta_{t+1} \big\rangle + \lVert \theta_t - u \rVert_A^2 - \lVert \theta_t - \theta_{t+1} \rVert_A^2.
    \end{equation}
\end{lemma}

\begin{lemma}\label{lemma:inverse-regret}
    Under \Cref{assumption:bound,assumption:link-function} and setting $\eta = 1 + R_s S$, for any $\lambda > 0$, the online estimator returned by Eqn.~\eqref{eq:OMD-estimator} satisfies
    \begin{align}
        \lVert \vtheta_{t+1} - \vtheta_\star \rVert_{\mH_{t+1}}^{{2}} & \leq 2 \eta \left( {\sum_{s=1}^{t} \tilde{\ell}_s (\vtheta_\star) - \sum_{s=1}^{t} \tilde{\ell}_s(\vtheta_{s+1})}\right) + 4\lambda S^2\\
        & \quad +  \frac{2\eta \alpha R_s S L_\mu }{\sqrt{\lambda}} \sum_{s=1}^{t} \lVert \vtheta_{s+1} - \vtheta_s \rVert_{\mH_s}^2 - \sum_{s=1}^{t} \lVert \vtheta_{s+1} - \vtheta_{s} \rVert_{\mH_s}^2.
    \end{align}
    In addition, if $\lambda \geq 36 \eta^2 \alpha^2 R_s^2 S^2 L_\mu^2 $, it follows that
    \begin{align}
        \lVert \vtheta_{t+1} - \vtheta_\star \rVert_{\mH_{t+1}}^{{2}} & \leq 2 \eta \left( {\sum_{s=1}^{t} \tilde{\ell}_s (\vtheta_\star) - \sum_{s=1}^{t} \tilde{\ell}_s(\vtheta_{s+1})}\right) + 4\lambda S^2 - \frac{2}{3} \sum_{s=1}^{t} \lVert \vtheta_{s+1} - \vtheta_{s} \rVert_{\mH_s}^2.
    \end{align}
\end{lemma}

\begin{proof}
By the integral formulation of Taylor's expansion and the fact that $\nabla^2 \tilde{\ell}_s(\vtheta) = \frac{{\weightcolor{w_s}}}{{\dispersioncolor{g(\tau_s)}}}\curvaturecolor{\dmu} (\langle \vx_s, \vtheta \rangle) \vx_s \vx_s^\top$, we get
\begin{align}
    \tilde{\ell}_s (\vtheta_{s+1}) - \tilde{\ell}_s (\vtheta_\star) = \big\langle \nabla\tilde{\ell}_s (\vtheta_{s+1}), \vtheta_{s+1} - \vtheta_\star \big\rangle - \lVert \vtheta_{s+1} - \vtheta_\star \rVert_{{\weightcolor{w_s}} \tilde{\alpha}(\vx_s, \vtheta_{s+1}, \vtheta_\star) \vx_s \vx_s^\top/ {\dispersioncolor{g(\tau_s)}}}^2,
\end{align}
where $\tilde{\alpha} (\vx, \vtheta_1 , \vtheta_2) = \int_0^1 (1-v) \curvaturecolor{\dmu}( \langle\vx, \vtheta_1 \rangle + v \langle \vx, \vtheta_2 - \vtheta_1\rangle) dv$.
Since $\tilde{\alpha}(\vx_s, \vtheta_{s+1}, \vtheta_\star) \geq \curvaturecolor{\dmu} ( \langle \vx_s, \vtheta_{s+1} \rangle) / (2+ 2R_s S)$ (by Lemma \ref{lemma:self-concordance}), we have
\begin{equation}
    \frac{{\weightcolor{w_s}} \tilde{\alpha}(\vx_s, \vtheta_{s+1}, \vtheta_\star)}{{\dispersioncolor{g(\tau_s)}}} \vx_s \vx_s^\top \succeq \frac{1}{2 + 2R_s S} \nabla^2 \tilde{\ell}_s (\vtheta_{s+1}),
\end{equation}
which implies that
\begin{equation}
    \tilde{\ell}_s (\vtheta_{s+1}) - \tilde{\ell}_s (\vtheta_\star) \leq \big\langle \nabla\tilde{\ell}_s (\vtheta_{s+1}), \vtheta_{s+1} - \vtheta_\star \big\rangle - \frac{1}{2+2R_s S}\lVert \vtheta_{s+1} - \vtheta_\star \rVert_{\nabla^2 \tilde{\ell}_s (\vtheta_{s+1})}^2.
\end{equation}
Denoting by $f_s(\vtheta) = \big\langle \nabla\tilde{\ell}_s(\vtheta_s), \vtheta - \vtheta_s \big\rangle + \frac{1}{2} \lVert \vtheta - \vtheta_s \rVert_{\nabla^2 \tilde{\ell}_s (\vtheta_s)}^2$, we decompose the first term of the r.h.s of the above equation as:
\begin{align}
    \big\langle \nabla\tilde{\ell}_s (\vtheta_{s+1}), \vtheta_{s+1} - \vtheta_\star \big\rangle & = \big\langle \nabla\tilde{\ell}_s (\vtheta_{s+1}) - \nabla f_s(\vtheta_{s+1}), \vtheta_{s+1} - \vtheta_\star \big\rangle\\
    &\quad +\big\langle \nabla f_s(\vtheta_{s+1}), \vtheta_{s+1} - \vtheta_\star \big\rangle.
\end{align}
The first term is bounded as 
\begin{align}
    \Big\langle & \nabla\tilde{\ell}_s (\vtheta_{s+1}) - \nabla f_s(\vtheta_{s+1}), \vtheta_{s+1} - \vtheta_\star \Big\rangle\\
    & = \Big\langle \nabla\tilde{\ell}_s (\vtheta_{s+1}) - \nabla \tilde{\ell}_s(\vtheta_s) - \nabla^2\tilde{\ell}_s(\vtheta_s) (\vtheta_{s+1} - \vtheta_{s}), \vtheta_{s+1} - \vtheta_\star \Big\rangle\\
    & = \frac{{\weightcolor{w_s}}}{{\dispersioncolor{g(\tau_s)}}} \Big\langle \mu( \langle \vx_s, \vtheta_{s+1} \rangle) \vx_s - \mu( \langle \vx_s, \vtheta_s \rangle) \vx_s - \curvaturecolor{\dmu}(\langle \vx_s, \vtheta_s \rangle) \vx_s \vx_s^\top ( \vtheta_{s+1} - \vtheta_s ), \vtheta_{s+1} - \vtheta_\star \Big\rangle\\
    & = \frac{{\weightcolor{w_s}} \langle \vx_s,  \vtheta_{s+1} - \vtheta_\star \rangle}{{\dispersioncolor{g(\tau_s)}}} \Big( \mu( \langle \vx_s, \vtheta_{s+1} \rangle) - \mu( \langle \vx_s, \vtheta_s \rangle) - \curvaturecolor{\dmu}( \langle \vx_s, \vtheta_s \rangle) \langle \vx_s, \vtheta_{s+1} - \vtheta_s \rangle \Big)\\
    & = \frac{{\weightcolor{w_s}} \langle \vx_s, \vtheta_{s+1} - \vtheta_\star \rangle}{{\dispersioncolor{g(\tau_s)}}} \times \frac{\ddmu ( \xi_s )}{2} \langle \vx_s,  \vtheta_{s+1} - \vtheta_s \rangle^2 \tag{mean value theorem}\\
    & \leq \frac{{\weightcolor{w_s}} (2 S) }{{\dispersioncolor{g(\tau_s)}}} \times \frac{ R_s L_\mu}{2}\langle \vx_s,  \vtheta_{s+1} - \vtheta_s \rangle^2 \tag{\Cref{assumption:bound,assumption:link-function}}\\
    & \leq \frac{{\weightcolor{w_s}} R_s S L_\mu}{{\dispersioncolor{g(\tau_s)}}} \lVert \vx_s \rVert_{\mH_s^{-1}}^2 \lVert \vtheta_{s+1} - \vtheta_s \rVert_{\mH_s}^2, \tag{Cauchy-Schwarz}
\end{align}
where $\xi_s$ lies between $ \langle \vx_s, \vtheta_{s+1} \rangle$ and $\langle \vx_s, \vtheta_s \rangle$.
Note that by the definition of ${\weightcolor{w_s}}$, we have ${\weightcolor{w_s}} \leq \alpha {\dispersioncolor{g(\tau_s)}} / \lVert \vx_s \rVert_{\mH_s^{-1}}$, which yields ${\weightcolor{w_s}} \lVert \vx_s \rVert_{\mH_s^{-1}} / {\dispersioncolor{g(\tau_s)}} \leq \alpha$; and by the definition of $\mH_s \succeq \lambda \mI_d$, $\lVert \vx_s \rVert_{\mH_s^{-1}} \leq 1 / \sqrt{\lambda}$.
Then, we finally get
\begin{equation}
    \Big\langle \nabla\tilde{\ell}_s (\vtheta_{s+1}) - \nabla f_s(\vtheta_{s+1}), \vtheta_{s+1} - \vtheta_\star \Big\rangle \leq \frac{\alpha R_s S L_\mu }{\sqrt{\lambda}} \lVert \vtheta_{s+1} - \vtheta_s \rVert_{\mH_s}^2.
\end{equation}
The second term is bounded by applying Lemma \ref{lemma:update-rule} with $u = \vtheta_\star$:
\begin{equation}
    \lVert \vtheta_{s+1} - \vtheta_\star \rVert_{\mH_s}^2 \leq 2 \eta \big\langle \nabla f_s ( \vtheta_{s+1}), \vtheta_\star - \vtheta_{s+1} \big\rangle + \lVert \vtheta_s - \vtheta_\star \rVert_{\mH_s}^2 - \lVert \vtheta_{s+1} - \vtheta_{s} \rVert_{\mH_s}^2,
\end{equation}
which is equivalent to
\begin{equation}
    \big\langle \nabla f_s ( \vtheta_{s+1}), \vtheta_{s+1}  - \vtheta_\star \big\rangle \leq \frac{1}{2\eta} \left(  \lVert \vtheta_s - \vtheta_\star \rVert_{\mH_s}^2  - \lVert \vtheta_{s+1} - \vtheta_\star \rVert_{\mH_s}^2 - \lVert \vtheta_{s+1} - \vtheta_{s} \rVert_{\mH_s}^2 \right).
\end{equation}
Combining them with setting $\eta = 1+R_s S $, we get 
\begin{align}
    \tilde{\ell}_s (\vtheta_{s+1}) - \tilde{\ell}_s (\vtheta_\star) & \leq \frac{1}{2\eta} \left(  \lVert \vtheta_s - \vtheta_\star \rVert_{\mH_s}^2  - \lVert \vtheta_{s+1} - \vtheta_\star \rVert_{\mH_s}^2 - \lVert \vtheta_{s+1} - \vtheta_{s} \rVert_{\mH_s}^2 \right)\\
    & \quad + \frac{\alpha R_s S L_\mu }{\sqrt{\lambda}} \lVert \vtheta_{s+1} - \vtheta_s \rVert_{\mH_s}^2 - \frac{1}{2+2R_sS}\lVert \vtheta_{s+1} - \vtheta_\star \rVert_{\nabla^2 \tilde{\ell}_s (\vtheta_{s+1})}^2\\
    & = \frac{1}{2\eta} \left(  \lVert \vtheta_s - \vtheta_\star \rVert_{\mH_s}^2  - \lVert \vtheta_{s+1} - \vtheta_\star \rVert_{\mH_{s+1}}^2 - \lVert \vtheta_{s+1} - \vtheta_{s} \rVert_{\mH_s}^2 \right)\\
    & \quad + \frac{\alpha R_s S L_\mu }{\sqrt{\lambda}} \lVert \vtheta_{s+1} - \vtheta_s \rVert_{\mH_s}^2.
\end{align}
Taking the summation over $s \in [t]$ yields
\begin{align}
    \sum_{s=1}^{t} \tilde{\ell}_s (\vtheta_{s+1}) - \sum_{s=1}^{t} \tilde{\ell}_s ( \vtheta_\star) & \leq \frac{1}{2\eta} \left( \lVert \vtheta_1 - \vtheta_\star \rVert_{\mH_1}^2 - \lVert \vtheta_{t+1} - \vtheta_\star \rVert_{\mH_{t+1}}^2 - \sum_{s=1}^{t} \lVert \vtheta_{s+1} - \vtheta_{s} \rVert_{\mH_s}^2 \right)\\
    & \quad  + \frac{\alpha R_s S L_\mu }{\sqrt{\lambda}} \sum_{s=1}^{t} \lVert \vtheta_{s+1} - \vtheta_s \rVert_{\mH_s}^2.
\end{align}
Rearranging the terms with the fact $\lVert \vtheta_1 - \vtheta_\star \rVert_{\mH_1}^2 \leq 4\lambda S^2$, we get 
\begin{align}
    \lVert \vtheta_{t+1} - \vtheta_\star \rVert_{\mH_{t+1}}^{{2}} & \leq 2 \eta \left( {\sum_{s=1}^{t} \tilde{\ell}_s (\vtheta_\star) - \sum_{s=1}^{t} \tilde{\ell}_s(\vtheta_{s+1})}\right) + 4\lambda S^2\\
    & \quad +  \frac{2\eta \alpha R_s S L_\mu }{\sqrt{\lambda}} \sum_{s=1}^{t} \lVert \vtheta_{s+1} - \vtheta_s \rVert_{\mH_s}^2 - \sum_{s=1}^{t} \lVert \vtheta_{s+1} - \vtheta_{s} \rVert_{\mH_s}^2.
\end{align}
Furthermore, if we set $\sqrt{\lambda} \geq 6 \eta \alpha R_s S L_\mu$
\begin{align}
    \lVert \vtheta_{t+1} - \vtheta_\star \rVert_{\mH_{t+1}}^{{2}} & \leq 2 \eta \left( {\sum_{s=1}^{t} \tilde{\ell}_s (\vtheta_\star) - \sum_{s=1}^{t} \tilde{\ell}_s(\vtheta_{s+1})}\right) + 4\lambda S^2 - \frac{2}{3} \sum_{s=1}^{t} \lVert \vtheta_{s+1} - \vtheta_{s} \rVert_{\mH_s}^2.
\end{align}
\end{proof}

\begin{lemma}\label{lemma:mixture}
    Let $\{ \mathcal{F}_t\}_{t=1}^\infty$ be a filtration defined by $\mathcal{F}_t = \sigma\left ( \{(\vx_s, r_s, c_s, \tau_s)\}_{s=1}^{t-1}\cup \{\vx_t, \tau_t\} \right)$. Let $\{P_t\}_{t=1}^\infty$ be a stochastic process such that the random variable $P_t$ is a distribution over $\mathbb{R}^d$ and is $\mathcal{F}_t$-measurable. Moreover assume that the loss function $\ell_t$ is $\mathcal{F}_{t+1}$-measurable. For any $t \geq 1$, define
    \begin{equation}
        M_t \coloneqq \exp \left( \sum_{s=1}^{t} {\weightcolor{w_s}} \ell_s(\vtheta_\star) - \sum_{s=1}^{t} {\weightcolor{w_s}} m_s(P_s) \right).
    \end{equation}
    Then, for any $\delta \in (0, 1]$, we have
    \begin{equation}
        \mathbb{P} \left( \forall t \geq 1, \sum_{s=1}^{t} {\weightcolor{w_s}} \ell_s(\vtheta_\star) \leq \sum_{s=1}^{t} {\weightcolor{w_s}} m_s(P_s) + \log \frac{1}{\delta} \right) \geq 1- \delta.
    \end{equation}
\end{lemma}

\begin{proof}
We first verify that $M_t$ is a supermartingale, and apply Ville's inequality to this supermartingale to complete the proof.
By the definition of $M_t$, it can be written as:
\begin{align}
    M_t & = M_{t-1} \exp \left( \weightcolor{w_t} \ell_t(\vtheta_\star) - \weightcolor{w_t} m_t(P_t) \right) = M_{t-1} \left[ \exp \left( \ell_t(\vtheta_\star) - m_t(P_t) \right)\right]^{\weightcolor{w_t}}.
\end{align}
Let $p_t (r; \vtheta)$ denote the conditional probability density of reward $r$ given selected arm $\vx_t$, dispersion $g(\tau_t)$ and parameter $\vtheta$ (see Eqn.~\eqref{eqn:GLM}), then the factors independent of $\vtheta$ are canceled as
\begin{align}
    \exp \left( \ell_t(\vtheta_\star) - m_t(P_t) \right) & = \frac{\E_{\vtheta\sim P_t}[\exp(-\ell_t (\vtheta)]}{\exp(-\ell_t(\vtheta_\star))}\\
    & = \frac{\E_{\vtheta\sim P_t}[p_t (r_t ;\vtheta)]}{p_t(r_t; \vtheta_\star)}.
\end{align}

The key observation is that $\weightcolor{w_t}$ is fixed
conditional on $\mathcal F_t$.
Since $\weightcolor{w_t} \in (0,1]$, the mapping $x \mapsto x^{\weightcolor{w_t}}$ is concave in the domain $x \in [0, \infty)$ as mentioned in the main text.
Thus, we can utilize the expectation version of Jensen's inequality~\citep{jensen}:
\begin{align}
    \mathbb{E} [ M_t | \mathcal{F}_t ] & = M_{t-1} \mathbb{E} \left[ \left( \frac{ \mathbb{E}_{\vtheta \sim P_t} [ p_t (r_t ; \vtheta)]}{p_t (r_t; \vtheta_\star)} \right)^{\weightcolor{w_t}} \middle| \mathcal{F}_{t} \right]\\
    &\leq M_{t-1} \left( \mathbb{E} \left[ \frac{ \mathbb{E}_{\vtheta \sim P_t} [ p_t (r_t; \vtheta)]}{p_t (r_t; \vtheta_\star)}  \middle| \mathcal{F}_{t} \right]\right)^{\weightcolor{w_t}} \tag{Jensen's inequality} \\
    & = M_{t-1} \left( \int \frac{ \mathbb{E}_{\vtheta \sim P_t} [ p_t (r; \vtheta)]}{p_t (r; \vtheta_\star)}p_t (r; \vtheta_\star) dr\right)^{\weightcolor{w_t}}\\
    & = M_{t-1} \left( \int \mathbb{E}_{\vtheta \sim P_t} [ p_t (r; \vtheta)] dr\right)^{\weightcolor{w_t}}\\
    & = M_{t-1} \left( \mathbb{E}_{\vtheta \sim P_t} \left[ \int p_t (r; \vtheta) dr \right]\right)^{\weightcolor{w_t}} \tag{Fubini's theorem} \\
    & = M_{t-1} \tag{$\int p_t (r; \vtheta) dr = 1$}.
\end{align}
This establishes that $(M_t)_{t\geq0}$ is a nonnegative supermartingale with respect to $(\mathcal F_{t+1})_{t\geq0}$.
Applying Ville's inequality~\citep{Ville1939}, with $M_0=1$, completes the proof.
\end{proof}

\begin{lemma}[Weighted version of Lemma 6 of \citet{zhang2025onepass}]\label{lemma:mixture-ineq}
    Under {\Cref{assumption:bound,assumption:link-function}}, let $P_s = \mathcal{N}(\vtheta_s, \gamma \mH_s^{-1})$, where $\gamma$ is any positive constant. Then, if we set {$\lambda>0$ with} $\lambda \geq 64 d \gamma R_s^2 /7$, we have
    \begin{equation}
        \sum_{s=1}^{t} {\weightcolor{w_s}} m_s(P_s) \leq \sum_{s=1}^{t} {\weightcolor{w_s}} \ell_s ( \vtheta_{s+1} ) + \sum_{s=1}^{t} \frac{1}{2\gamma} \lVert \vtheta_s - \vtheta_{s+1} \rVert_{\mH_s}^2 + \left( 2\gamma + \frac{1}{2} \right) \log \frac{\mathrm{det} (\mH_{t+1})}{\mathrm{det}(\mH_1)},
    \end{equation}
    where mix loss is defined as $m_s(P_s) = -\log (\mathbb{E}_{\vtheta \sim P_s} [ \exp( - \ell_s(\vtheta))])$.
\end{lemma}

\begin{proof}
According to Lemma 12 in \citet{zhang2025onepass}, the mix loss can be decomposed as:
\begin{equation}
    m_s(P_s) = - \log \mathbb{E}_{\vtheta \sim P_s} \left[ \exp \left(- \ell_s(\vtheta)\right) \right] = \mathbb{E}_{\vtheta \sim Q_s} \left[ \ell_s(\vtheta) \right] +  \mathrm{KL}(Q_s \Vert P_s) - \mathrm{KL}(Q_s \Vert P_s^*)
\end{equation}
for any distribution $Q_s$ on $\mathbb{R}^d$, where $P_s^* (\vtheta) \propto P_s(\vtheta) \exp ( - \ell_s (\vtheta))$.
We choose $Q_s = \mathcal{N}(\vtheta_{s+1}, \gamma \mH_{s+1}^{-1})$.
By ignoring the last term which is nonpositive, we have
\begin{equation}
    {\weightcolor{w_s}} m_s(P_s) \leq {\weightcolor{w_s}} \mathbb{E}_{\vtheta \sim Q_s} \left[ \ell_s(\vtheta) \right] + {\weightcolor{w_s}} \mathrm{KL}(Q_s \Vert P_s).
\end{equation}
Since $\vtheta \sim Q_s$, which is symmetric around $\vtheta_{s+1}$, it follows that $(\vtheta - \vtheta_{s+1}) \sim \mathcal{N}(0, \gamma\mH_{s+1}^{-1})$.
Then, we can upper bound the first term as:
\begin{align}
    {\weightcolor{w_s}} \mathbb{E}_{ \vtheta \sim Q_s} [ \ell_s (\vtheta)] & = {\weightcolor{w_s}} \mathbb{E}_{ \vtheta \sim Q_s} \big[ \ell_s (\vtheta_{s+1}) + \langle \nabla\ell_s(\vtheta_{s+1}), \vtheta - \vtheta_{s+1} \rangle \big] + {\weightcolor{w_s}} \mathbb{E}_{\vtheta \sim Q_s} [ \mathcal{D}_{\ell_s} ( \vtheta, \vtheta_{s+1})]\\
    & = {\weightcolor{w_s}} \ell_s (\vtheta_{s+1}) + {\weightcolor{w_s}}\mathbb{E}_{\vtheta \sim Q_s} [ \mathcal{D}_{ \ell_s} ( \vtheta, \vtheta_{s+1})]\\
    & = {\weightcolor{w_s}} \ell_s(\vtheta_{s+1}) + \mathbb{E}_{\vtheta \sim Q_s} [ \mathcal{D}_{{\weightcolor{w_s}} \ell_s} ( \vtheta, \vtheta_{s+1})]\\
    & \leq {\weightcolor{w_s}} \ell_s(\vtheta_{s+1}) + 2\gamma \log \frac{\mathrm{det}(\mH_{s+1})}{\mathrm{det} ( \mH_{s})},
\end{align}
where the last equality holds due to the linearity of Bregman divergence and the inequality holds due to Lemma \ref{lemma:bregman}.
For the second term, recall that $P_s = \mathcal{N}(\vtheta_s, \gamma \mH_s^{-1})$ and $Q_s = \mathcal{N}(\vtheta_{s+1}, \gamma \mH_{s+1}^{-1})$, then by Lemma \ref{lemma:normal-kl-divergence}, we have
\begin{align}
    \mathrm{KL}(Q_s \Vert P_s ) & = \frac{1}{2} \left( \log \frac{\mathrm{det}(\mH_{s+1})}{\mathrm{det}(\mH_s)} + \mathrm{Tr} ( \mH_{s+1}^{-1} \mH_s) + \frac{1}{\gamma} \lVert \vtheta_s - \vtheta_{s+1} \rVert_{\mH_s}^2 - d \right)\\
    & \leq \frac{1}{2} \log \frac{\mathrm{det}(\mH_{s+1})}{\mathrm{det}(\mH_s)} + \frac{1}{2\gamma} \lVert \vtheta_s - \vtheta_{s+1} \rVert_{\mH_s}^2,
\end{align}
where the inequality holds because $\mH_{s+1} \succeq \mH_s$ implies that $\mH_{s+1}^{-1} \mH_s$ has eigenvalues at most 1, hence $\mathrm{Tr} (\mH_{s+1}^{-1} \mH_s) \leq d$.
Combining the above results, we get
\begin{align}
    {\weightcolor{w_s}} m_s(P_s) & \leq {\weightcolor{w_s}} \ell_s(\vtheta_{s+1}) + \left( 2\gamma + \frac{{\weightcolor{w_s}}}{2} \right) \log \frac{\mathrm{det}(\mH_{s+1})}{\mathrm{det} ( \mH_{s})} +  \frac{{\weightcolor{w_s}}}{2\gamma} \lVert \vtheta_s - \vtheta_{s+1} \rVert_{\mH_s}^2\\
    & \leq {\weightcolor{w_s}} \ell_s(\vtheta_{s+1}) + \left( 2\gamma + \frac{1}{2} \right) \log \frac{\mathrm{det}(\mH_{s+1})}{\mathrm{det} ( \mH_{s})} +  \frac{1}{2\gamma} \lVert \vtheta_s - \vtheta_{s+1} \rVert_{\mH_s}^2
\end{align}
and summing up over $s \in [t]$, we get the desired result:
\begin{equation}
    \sum_{s=1}^{t} {\weightcolor{w_s}} m_s(P_s) \leq \sum_{s=1}^{t} {\weightcolor{w_s}} \ell_s ( \vtheta_{s+1}) + \left( 2 \gamma + \frac{1}{2} \right) \log \frac{\mathrm{det}(\mH_{t+1})}{\mathrm{det}(\mH_1)} + \frac{1}{2\gamma} \sum_{s=1}^{t} \lVert \vtheta_s - \vtheta_{s+1} \rVert_{\mH_s}^2.
\end{equation}
\end{proof}

\begin{lemma}[Weighted version of Lemma 7 of \citet{zhang2025onepass}]\label{lemma:bregman}
    Under \Cref{assumption:bound,assumption:link-function}, let $\gamma > 0$ and $\lambda \geq 64d \gamma R_s^2 / 7$, then we have
    \begin{equation}
        \mathbb{E}_{\vtheta \sim Q_s} [ \mathcal{D}_{{\weightcolor{w_s}} \ell_s} ( \vtheta, \vtheta_{s+1})] \leq 2 \gamma \left( \log \mathrm{det}( \mH_{s+1}) - \log \mathrm{det} ( \mH_s) \right).
    \end{equation}
\end{lemma}

\begin{proof}

Recall that Bregman divergence is defined as $\mathcal{D}_f(\vtheta_1, \vtheta_2) = f(\vtheta_1) - f(\vtheta_2) - \langle \nabla f(\vtheta_2), \vtheta_1 - \vtheta_2 \rangle $, and the integral formulation of Taylor's expansion implies
\begin{align}
    f(\vtheta_1) & = f(\vtheta_2) + \langle \nabla f (\vtheta_2) , \vtheta_1 - \vtheta_2 \rangle + \int_0^1 (1-v) (\vtheta_1 - \vtheta_2)^\top \nabla^2 f(\vtheta_2 + v (\vtheta_1 - \vtheta_2)) (\vtheta_1 - \vtheta_2) dv \\
    & = f(\vtheta_2) + \langle \nabla f (\vtheta_2) , \vtheta_1 - \vtheta_2 \rangle + \lVert \vtheta_1 - \vtheta_2 \rVert _{\int_0^1 (1-v)\nabla^2 f(\vtheta_2 + v (\vtheta_1 - \vtheta_2)) dv}^2.
\end{align}
Combining them, the l.h.s. can be written as
\begin{equation}
    \mathbb{E}_{\vtheta \sim Q_s} [ \mathcal{D}_{{\weightcolor{w_s}} \ell_s} ( \vtheta, \vtheta_{s+1})] = \mathbb{E}_{\vtheta \sim Q_s} \left[ \lVert \vtheta - \vtheta_{s+1} \rVert_{{\weightcolor{w_s}}\tilde{\alpha}(\vx_s, \vtheta_{s+1}, \vtheta)\vx_s \vx_s^\top / {\dispersioncolor{g(\tau_s)}}}^2 \right],
\end{equation}
where $\tilde{\alpha} (\vx, \vtheta_1 , \vtheta_2) = \int_0^1 (1-v) \curvaturecolor{\dmu}( \langle\vx, \vtheta_1 \rangle + v \langle \vx, \vtheta_2 - \vtheta_1\rangle) dv$.
By Lemma \ref{lemma:self-concordance}, we have $\tilde{\alpha}( \vx_s , \vtheta_{s+1}, \vtheta) \leq \exp ( R_s^2 \lVert \vtheta - \vtheta_{s+1} \rVert_2^2) \curvaturecolor{\dmu}( \langle \vx_s, \vtheta_{s+1} \rangle)$, which implies that
\begin{equation}
    \lVert \vtheta - \vtheta_{s+1} \rVert_{{\weightcolor{w_s}}\tilde{\alpha}(\vx_s, \vtheta_{s+1}, \vtheta)\vx_s \vx_s^\top / {\dispersioncolor{g(\tau_s)}}}^2 \leq \exp ({ R_s^2 \lVert \vtheta - \vtheta_{s+1} \rVert_2^2}) \lVert \vtheta - \vtheta_{s+1} \rVert_{{\weightcolor{w_s}} \nabla^2 \ell_s (\vtheta_{s+1})}^2.
\end{equation}
Taking expectation on both terms we get
\begin{align}
    \mathbb{E}_{\vtheta \sim Q_s} [ \mathcal{D}_{{\weightcolor{w_s}} \ell_s} ( \vtheta, \vtheta_{s+1})] & \leq \mathbb{E}_{\vtheta \sim Q_s} \left[ \exp ( R_s^2 { \lVert \vtheta - \vtheta_{s+1} \rVert_2^2}) \lVert \vtheta - \vtheta_{s+1} \rVert_{{\weightcolor{w_s}} \nabla^2 \ell_s (\vtheta_{s+1})}^2 \right]\\
    & \leq \sqrt{\mathbb{E}_{\vtheta \sim Q_s} \left[ \exp ( 2 R_s^2 { \lVert \vtheta - \vtheta_{s+1} \rVert_2^2})\right] \mathbb{E}_{\vtheta \sim Q_s} \left[ \lVert \vtheta - \vtheta_{s+1} \rVert_{{\weightcolor{w_s}} \nabla^2 \ell_s (\vtheta_{s+1})}^4 \right]}\\
    & \leq \sqrt{ \frac{4}{3} \mathbb{E}_{\vtheta \sim Q_s} \left[ \left\lVert ( {\weightcolor{w_s}} \nabla^2 \ell_s(\vtheta_{s+1}))^{\frac{1}{2}}( \vtheta - \vtheta_{s+1}) \right\rVert_2^4 \right]},
\end{align}
where the last inequality holds due to Lemma \ref{lemma:gaussian-exponential} with $c = 2R_s^2$.
For the remaining term, note that $(\vtheta - \vtheta_{s+1})$ follows the normal distribution $\mathcal{N}(0, \gamma\mH_{s+1}^{-1})$, then $({\weightcolor{w_s}} \nabla^2 \ell_s(\vtheta_{s+1}))^{\frac{1}{2}} (\vtheta - \vtheta_{s+1})$ follows a normal distribution with covariance matrix $ \gamma {\weightcolor{w_s}} (\nabla^2 \ell_s(\vtheta_{s+1}))^{\frac{1}{2}} \mH_{s+1}^{-1} (\nabla^2 \ell_s(\vtheta_{s+1}))^{\frac{1}{2}}$, and has the same distribution as $\sum_{i=1}^{d} \sqrt{\gamma \lambda_i} X_i \ve_i$, where $X_i$'s are i.i.d. standard normal random variables, $\lambda_i$ is the $i$-th largest eigenvalue of the matrix $ {\weightcolor{w_s}} (\nabla^2 \ell_s(\vtheta_{s+1}))^{\frac{1}{2}} \mH_{s+1}^{-1} (\nabla^2 \ell_s(\vtheta_{s+1}))^{\frac{1}{2}}$ and $\ve_i$'s form an orthonormal basis. Then, we have
\begin{align}
    \sqrt{ \mathbb{E}_{\vtheta \sim Q_s} \left[ \left\lVert ( {\weightcolor{w_s}} \nabla^2 \ell_s(\vtheta_{s+1}))^{\frac{1}{2}}( \vtheta - \vtheta_{s+1}) \right\rVert_2^4 \right]} & = \sqrt{ \mathbb{E}_{X_1, \cdots, X_d} \left[ \left(\sum_{i=1}^{d}  \gamma \lambda_i X_i^2\right)^2 \right] }\\
    & = \sqrt{ \sum_{i=1}^{d} \sum_{j=1}^{d} \gamma^2 \lambda_i \lambda_j \mathbb{E}_{X_i, X_j} \left[ X_i^2 X_j^2 \right]}
\end{align}
Note that $\mathbb{E}_{X_i}[ X_i^4] = 3 $ and $\mathbb{E}_{X_i, X_j} [ X_i^2 X_j^2] = 1 $ for $i \neq j$, then the last term is upper bounded as:
\begin{align}
    \sqrt{ \sum_{i=1}^{d} \sum_{j=1}^{d} \gamma^2 \lambda_i \lambda_j \mathbb{E}_{X_i, X_j} \left[ X_i^2 X_j^2 \right]} & \leq \sqrt{ \sum_{i=1}^{d} \sum_{j=1}^{d} 3 \gamma^2 \lambda_i \lambda_j }\\
    & = \sqrt{3}\gamma \sqrt{\sum_{i=1}^{d} \sum_{j=1}^{d} \lambda_i \lambda_j}\\
    & = \sqrt{3}\gamma \sqrt{\left(\sum_{i=1}^{d} \lambda_i \right) \left( \sum_{j=1}^{d} \lambda_j \right)}\\
    & = \sqrt{3}\gamma \text{Tr} \left( {\weightcolor{w_s}} (\nabla^2 \ell_s(\vtheta_{s+1}))^{\frac{1}{2}} \mH_{s+1}^{-1} (\nabla^2 \ell_s(\vtheta_{s+1}))^{\frac{1}{2}} \right)\\
    & = \sqrt{3}\gamma \text{Tr} \left(  \mH_{s+1}^{-1} ({\weightcolor{w_s}} \nabla^2 \ell_s(\vtheta_{s+1})) \right)\\
    & = \sqrt{3}\gamma \text{Tr} \left(  \mH_{s+1}^{-1} ( \mH_{s+1} - \mH_s) \right)\\
    & = \sqrt{3}\gamma \text{Tr} \left(  \mI_d - \mH_s \mH_{s+1}^{-1}  \right)\\
    & \leq \sqrt{3} \gamma \left( \log \mathrm{det}(\mH_{s+1}) - \log \mathrm{det} ( \mH_s ) \right)
\end{align}
Here, we use the facts that the summation of eigenvalues of a matrix is equal to the trace of the matrix, $\mathrm{Tr}(AB) = \mathrm{Tr}(BA)$ and $\mathrm{Tr}(I - AB^{-1}) \leq \log \mathrm{det} (B) - \log \mathrm{det} (A)$.
\end{proof}

\begin{lemma}[Eqn. (25) of \citet{zhang2025onepass}]\label{lemma:gaussian-exponential}
    Let $P = \mathcal{N}(0 ,\gamma H^{-1})$ be a multivariate normal distribution, where $H \succeq \lambda \mI_d$ and $\gamma > 0$. Then, if $\lambda \geq 32 d \gamma c / 7$, we have
    \begin{equation}
        \mathbb{E}_{\vtheta \sim P} \left[ \exp \left( c \lVert \vtheta \rVert_2^2 \right) \right] \leq \frac{4}{3}.
    \end{equation}
\end{lemma}

\begin{lemma}\label{lemma:normal-kl-divergence}
    Let $P = \mathcal{N}(\mu_p, \Sigma_p)$ and $Q = \mathcal{N}( \mu_q, \Sigma_q)$ be multivariate normal distributions.
    The Kullback-Leibler divergence between them is given by
    \begin{equation}
        \mathrm{KL}(Q \Vert P) = \frac{1}{2} \left( \log \frac{\mathrm{det}(\Sigma_p)}{\mathrm{det}(\Sigma_q)} + \mathrm{Tr} ( \Sigma_q \Sigma_p^{-1}) + \lVert \mu_p - \mu_q \rVert_{\Sigma_p^{-1}}^2 - d \right).
    \end{equation}
\end{lemma}

\newpage
\section{\texorpdfstring{Proof of \Cref{theorem:regret}: Regret Upper Bound}{Proof of Theorem~\ref{theorem:regret}: Regret Upper Bound}}
\label{app:proof-regret}

The entire proof will proceed under the event $\mathcal{E}_\delta \coloneqq \{ \forall t \geq 1, \vtheta_\star \in \mathcal{C}_t(\delta) \}$, where $\mathcal{C}_t(\delta)$ is the confidence sequence constructed in Theorem \ref{theorem:confidence-sequence}; the above event holds with probability at least $1 - \delta$.

We begin with the following lemma, whose proof is given in \Cref{sec:regret-decompostion-proof}, decomposing the regret into four terms:

\begin{lemma}\label{lemma:decomposition}
    Under the event $\mathcal{E}_\delta$, the regret is upper bounded as:
    \begin{align}
        \Reg(T) & \leq 2 \big( \beta_T(\delta) + 2\eta \alpha C \big) \underbrace{\sum_{t \in [T]: \weightcolor{w_t} = 1} \sqrt{\curvaturecolor{\dmu}(\langle \vx_t, \vtheta_\star \rangle)} \left\lVert \sqrt{\curvaturecolor{\dmu}( \langle \vx_t, \vtheta_{t+1} \rangle)} \vx_t \right\rVert_{\mH_t^{-1}}}_{\triangleq S_0}\\
        & \quad +2 \big( \beta_T(\delta) + 2\eta \alpha C \big) \underbrace{\sum_{t \in [T]: \weightcolor{w_t} < 1} \curvaturecolor{\dmu}(\langle \vx_t, \vtheta_{t+1} \rangle) \left\lVert \vx_t \right\rVert_{\mH_t^{-1}}}_{\triangleq S_1}\\
        & \quad + 4 R_s L_\mu \big( \beta_{T+1}(\delta) + 2\eta \alpha C \big)^2 \underbrace{\sum_{t \in [T]: \weightcolor{w_t} = 1} \lVert \vx_t \rVert_{\mH_t^{-1}}^2}_{\triangleq S_2}\\
        & \quad + 4 R_s L_\mu \big( \beta_{T+1}(\delta) + 2\eta \alpha C \big)^2 \underbrace{\sum_{t \in [T]: \weightcolor{w_t} < 1} \lVert \vx_t \rVert_{\mH_t^{-1}}^2}_{\triangleq S_3}.
    \end{align}
\end{lemma}
{The terms $S_1$, $S_2$, and $S_3$ are bounded in \Cref{lemma:epl-results}, using the Hessian potential lemma (\Cref{lemma:revised-epl}).}
We first decompose the term $S_0$.
By the Cauchy-Schwarz inequality, we have
\begin{align}
    S_{0} & = \sum_{t \in [T]: \weightcolor{w_t}=1} \sqrt{ {\dispersioncolor{g(\tau_t)}} \curvaturecolor{\dmu}(\langle \vx_t, \vtheta_\star \rangle)} \left\lVert \sqrt{\frac{\weightcolor{w_t} }{{\dispersioncolor{g(\tau_t)}}} \curvaturecolor{\dmu}( \langle \vx_t, \vtheta_{t+1} \rangle)} \vx_t \right\rVert_{\mH_t^{-1}}\\
    & \leq \sqrt{\sum_{t \in [T]: \weightcolor{w_t}=1} {\dispersioncolor{g(\tau_t)}} \curvaturecolor{\dmu}( \langle \vx_t, \vtheta_\star \rangle)} \sqrt{\sum_{t \in [T]: \weightcolor{w_t}=1}\left\lVert \sqrt{\frac{\weightcolor{w_t} }{{\dispersioncolor{g(\tau_t)}}} \curvaturecolor{\dmu}( \langle \vx_t, \vtheta_{t+1} \rangle)} \vx_t \right\rVert_{\mH_t^{-1}}^2}\tag{Cauchy-Schwarz}\\
    & \leq \sqrt{\sum_{t =1}^{T} {\dispersioncolor{g(\tau_t)}} \curvaturecolor{\dmu}(\langle \vx_t, \vtheta_\star \rangle)} \sqrt{\sum_{t=1}^{T}\left\lVert \sqrt{\frac{\weightcolor{w_t} }{{\dispersioncolor{g(\tau_t)}}} \curvaturecolor{\dmu}( \langle \vx_t, \vtheta_{t+1} \rangle)} \vx_t \right\rVert_{\mH_t^{-1}}^2}, \tag{Monotonicity of summation}
\end{align}
where $\weightcolor{w_t}$ can be inserted in the first equation since $\weightcolor{w_t}=1$.
We adopt the self-concordance arguments in \citep[p.18]{abeille2021logistic} and \citep[p.21]{zhang2025onepass} as follows.
For each $t$, we have that
\begin{align}
    \curvaturecolor{\dmu}( \langle \vx_t, \vtheta_\star \rangle) & = {\curvaturecolor{\dmu}(\langle \vx_{t,\star}, \vtheta_\star \rangle)} + \int_0^1 \ddmu( \langle \vx_{t,\star}, \vtheta_\star \rangle + v \langle \vx_t - \vx_{t,\star}, \vtheta_\star \rangle ) dv \langle \vx_t - \vx_{t, \star} , \vtheta_\star \rangle\\
    & \leq {\curvaturecolor{\dmu}(\langle \vx_{t,\star}, \vtheta_\star \rangle)} + R_s \int_0^1 \curvaturecolor{\dmu}( \langle \vx_{t,\star}, \vtheta_\star \rangle + v \langle \vx_t - \vx_{t,\star}, \vtheta_\star \rangle ) dv {\big|\langle \vx_t - \vx_{t, \star} , \vtheta_\star \rangle\big|} \tag{Self-concordance of $\mu$}\\
    & = {\curvaturecolor{\dmu}(\langle \vx_{t,\star}, \vtheta_\star \rangle)} + R_s \Big[ \mu( \langle \vx_{t, \star}, \vtheta_\star \rangle) - \mu ( \langle \vx_t, \vtheta_\star \rangle) \Big].\tag{Fundamental theorem of calculus}
\end{align}
Therefore, we can upper bound the first factor of the above upper bound as:
\begin{align}
    \sum_{t=1}^{T} {\dispersioncolor{g(\tau_t)}} { \curvaturecolor{\dmu} ( \langle \vx_t, \vtheta_\star \rangle)} &\leq \sum_{t=1}^{T} {\dispersioncolor{g(\tau_t)}} {\curvaturecolor{\dmu_{t,\star}}} + R_s \sum_{t=1}^{T} {\dispersioncolor{g(\tau_t)}}\big[ \mu(\langle \vx_{t,\star}, \vtheta_\star \rangle) - \mu(\langle \vx_t, \vtheta_\star \rangle) \big]\\
    & \leq \sum_{t=1}^{T} {\dispersioncolor{g(\tau_t)}} {\curvaturecolor{\dmu_{t,\star}}} + {\dispersioncolor{g_{\max}}} R_s \cdot \Reg(T),
\end{align}
where we denote ${\curvaturecolor{\dmu_{t,\star}} \coloneqq \curvaturecolor{\dmu}( \langle \vx_{t, \star}, \vtheta_\star \rangle)}$ and ${\dispersioncolor{g_{\max}}} \coloneqq \max_{t \in [T]} {\dispersioncolor{g(\tau_t)}}$.
Now, the {Hessian potential lemma} (\Cref{lemma:revised-epl}) implies that
\begin{equation}
    \sum_{t=1}^{T} \left\lVert \sqrt{\frac{\weightcolor{w_t}}{{\dispersioncolor{g(\tau_t)}}} \curvaturecolor{\dmu}( \langle \vx_t, \vtheta_{t+1} \rangle )} \vx_t \right\rVert_{\mH_t^{-1}}^2 \leq 2d \left( 1 + \frac{\alpha L_\mu}{\sqrt{\lambda}} \right) \log \left( 1 + \frac{\alpha L_\mu T}{2 \sqrt{d \lambda}} \right).
\end{equation}
Using \Cref{lemma:epl-results}, which is also direct results of applying \Cref{lemma:revised-epl} to $S_1$, $S_2$, and $S_3$, and the above inequalities, we have
\begin{align}
    \Reg(T) & \leq \underbrace{2 \big( \beta_T(\delta) + 2\eta \alpha C \big) \sqrt{2d \left( 1 + \frac{\alpha L_\mu}{\sqrt{\lambda}} \right) \log \left( 1 + \frac{\alpha L_\mu T}{2 \sqrt{d \lambda}} \right)}}_{\triangleq \gamma_1} \sqrt{\sum_{t=1}^{T}{\dispersioncolor{g(\tau_t)}} {\curvaturecolor{\dmu_{t,\star}}} + {\dispersioncolor{g_{\max}}} R_s\cdot\Reg(T)}\\
    & \quad + \underbrace{\frac{4 d }{\alpha} \big( \beta_T(\delta) + 2\eta \alpha C \big) \left( 1 + \frac{\alpha L_\mu}{\sqrt{\lambda}} \right) \log \left( 1 + \frac{\alpha L_\mu T}{2 \sqrt{d \lambda}} \right)}_{\triangleq \gamma_2}\\
    & \quad + \underbrace{8  R_s L_\mu d{\kappacolor{\kappa}} {\dispersioncolor{g_{\max}}} \big( \beta_{T+1}(\delta) + 2\eta \alpha C \big)^2 \left( 1 + \frac{\alpha L_\mu}{\sqrt{\lambda}} \right) \log \left( 1 + \frac{\alpha L_\mu T}{2 \sqrt{d \lambda}} \right)}_{\triangleq \gamma_3}\\
    & \quad + \underbrace{8 R_s L_\mu \frac{d \kappacolor{\kappa}}{\sqrt{\lambda} \alpha}  \big( \beta_{T+1}(\delta) + 2\eta \alpha C \big)^2  \left( 1 + \frac{\alpha L_\mu}{\sqrt{\lambda}} \right) \log \left( 1 + \frac{\alpha L_\mu T}{2 \sqrt{d \lambda}} \right)}_{\triangleq \gamma_4}.
\end{align}

\paragraph{GLB with $R_s=0$.}
{When $R_s = 0$, the link is affine, $\lambda = \frac{d}{4S^2}$, and $\eta = 1$.}
Since Eqn.~\eqref{eqn:det-ratio} implies $\beta_T = \mathcal{O}( \sqrt{ d\log T})$, the inequality is reduced to
\begin{equation}
    \Reg(T) \leq \gamma_1 \sqrt{\sum_{t=1}^{T} {\dispersioncolor{g(\tau_t)}}{\curvaturecolor{\dmu_{t,\star}}}} + \gamma_2.
\end{equation}
Since we set $\alpha = \Theta\big( \sqrt{d} \: (C \vee 1 )^{-1} \big)$, we have 
\begin{equation}
    \gamma_1 = \gO \big( d \log T \big) \quad \text{and} \quad\gamma_2 = \gO \big( d (\log T)^{3/2} (C \vee 1)\big).
\end{equation}
Therefore, when $R_s = 0$, the regret is bounded as
\begin{equation}
    \Reg(T) \leq \mathcal{O} \left( d \log T \sqrt{\sum_{t=1}^{T} {\dispersioncolor{g(\tau_t)}}{\curvaturecolor{\dmu_{t,\star}}}} + d  (C \vee 1 ) \left( \log T \right)^{3/2} \right),
\end{equation}
which is $\widetilde{\mathcal{O}} \big( d \sqrt{\sum_{t} {\dispersioncolor{g(\tau_t)}}{\curvaturecolor{\dmu_{t,\star}}}} + d (C \vee 1) \big)$.

\paragraph{Nonlinear GLB with $R_s>0$.}

Solving the inequality with respect to $\sqrt{\sum_{t} {\dispersioncolor{g(\tau_t)}} {\curvaturecolor{\dmu_{t,\star}}} + {\dispersioncolor{g_{\max}}} R_s \cdot \Reg(T)}$ using \Cref{lemma:solve-ineq} and rearranging the terms, we get
\begin{equation}
    \Reg(T) \leq 2 \gamma_1 \sqrt{ \sum_{t=1}^{T} {{\dispersioncolor{g(\tau_t)}} \curvaturecolor{\dmu_{t,\star}}}}+ 2\gamma_1 \sqrt{{\dispersioncolor{g_{\max}}}R_s (\gamma_2 + \gamma_3 + \gamma_4) }+ {\dispersioncolor{g_{\max}}} R_s \gamma_1^2 + \gamma_2 + \gamma_3 + \gamma_4.
\end{equation}
Since we set $\alpha = \Theta \big( \sqrt{d} \: (C \vee 1)^{-1} \big)$, we have
\begin{equation}
    \begin{cases}
        \gamma_1 = \gO \big( d \log T \big)\\
        \gamma_2 = \gO \big( d (C \vee 1) (\log T)^{3/2} \big)\\
        \gamma_3 = \gO \big( d^2 {\kappacolor{\kappa}} {\dispersioncolor{g_{\max}}} ( \log T)^2 \big)\\
        \gamma_4 = \gO \big(d {\kappacolor{\kappa}} (C \vee 1) (\log T)^2 \big)
    \end{cases},
\end{equation}
which lead to the regret bound of order 
\begin{equation}\label{eq:first-regret-bound}
    \widetilde{\gO} \left( d\sqrt{ \sum_{t=1}^{T} {\dispersioncolor{g(\tau_t)}} {\curvaturecolor{\dmu_{t,\star}}}} + d^2 {\dispersioncolor{g_{\max}}} {\kappacolor{\kappa}} + d (1+ {\kappacolor{\kappa}}) ( C \vee 1 ) \right).
\end{equation}

\qed

For the remaining part, we present the proofs of the supporting lemmas used throughout the proof.

\subsection{\texorpdfstring{Proof of \Cref{lemma:decomposition}: Regret Decomposition}{Proof of Lemma~\ref{lemma:decomposition}: Regret Decomposition}}
\label{sec:regret-decompostion-proof}
Let $(\vx_t, \tilde{\vtheta}_t) = \argmax_{\vx \in \mathcal{X}_t, \vtheta \in \mathcal{C}_t(\delta)} \mu(\langle \vx, \vtheta \rangle)$, then the regret is upper bounded as:
\begin{align}
    \Reg(T) & = \sum_{t=1}^{T} \mu(\langle \vx_{t,\star}, \vtheta_\star \rangle) - \sum_{t=1}^{T} \mu(\langle \vx_t, \vtheta_\star \rangle)\\
    & \leq \sum_{t=1}^{T} \mu(\langle \vx_t, \tilde{\vtheta}_t \rangle) - \sum_{t=1}^{T} \mu( \langle \vx_t, \vtheta_\star \rangle) \tag{Definition of $(\vx_t, \tilde{\vtheta}_t) $}\\
    & = \underbrace{\sum_{t=1}^{T} \curvaturecolor{\dmu}(\langle \vx_t, \vtheta_\star \rangle) \langle \vx_t, \tilde{\vtheta}_t - \vtheta_\star \rangle}_{=A} + \underbrace{\sum_{t=1}^{T} \int_0^1 (1-v) \ddmu( \langle \vx_t, \vtheta_\star \rangle + v \langle \vx_t, \tilde{\vtheta}_t - \vtheta_\star \rangle) dv \langle \vx_t, \tilde{\vtheta}_t - \vtheta_\star \rangle^2}_{=B},
\end{align}
where the inequality comes from the definition of $(\vx_t, \tilde{\vtheta}_t)$, and the last equation is due to the integral formulation of Taylor's expansion.

\paragraph{Analysis of term $A$}
For term $A$, we have
\begin{align}
    A & = \sum_{t=1}^{T} \curvaturecolor{\dmu}( \langle \vx_t, \vtheta_\star \rangle) \langle \vx_t, \tilde{\vtheta}_t - \vtheta_\star \rangle\\
    & \leq \sum_{t=1}^{T} \curvaturecolor{\dmu}(\langle \vx_t, \vtheta_\star \rangle) \lVert \vx_t \rVert_{\mH_t^{-1}} \lVert \tilde{\vtheta}_t - \vtheta_\star \rVert_{\mH_t} \tag{Cauchy-Schwarz}\\
    & \leq 2 \sum_{t=1}^{T} \curvaturecolor{\dmu}( \langle \vx_t, \vtheta_\star \rangle) \big( \beta_t(\delta) + 2\eta \alpha C \big) \lVert \vx_t \rVert_{\mH_t^{-1}} \tag{Definition of $\tilde{\vtheta}_t$ and \Cref{theorem:confidence-sequence}}\\
    & \leq 2 \big( \beta_T(\delta) + 2\eta \alpha C \big) \sum_{t=1}^{T} \curvaturecolor{\dmu}( \langle \vx_t, \vtheta_\star \rangle)\lVert \vx_t \rVert_{\mH_t^{-1}} \tag{Monotonicity of $\beta_t$}\\
    & = \underbrace{2 \big( \beta_T(\delta) + 2\eta \alpha C \big) \sum_{t \in \mathcal{T}_1} \curvaturecolor{\dmu}(\langle \vx_t, \vtheta_\star \rangle)\lVert \vx_t \rVert_{\mH_t^{-1}}}_{=A_1}  + \underbrace{2 \big( \beta_T(\delta) + 2\eta \alpha C \big) \sum_{t \in \mathcal{T}_2} \curvaturecolor{\dmu}( \langle \vx_t, \vtheta_\star \rangle)\lVert \vx_t \rVert_{\mH_t^{-1}}}_{=A_2},
\end{align}
where the first inequality is due to the Cauchy-Schwarz inequality and the second inequality follows the fact $\lVert \tilde{\vtheta}_t - \vtheta_\star \rVert_{\mH_t} \leq \lVert \tilde{\vtheta}_t - \vtheta_t \rVert_{\mH_t} + \lVert \vtheta_t - \vtheta_\star \rVert_{\mH_t} \leq 2 (\beta_t(\delta) + 2 \eta \alpha C ) $.
In the last equation, we decompose time horizon $[T]$ into two sets $\mathcal{T}_1$ and $\mathcal{T}_2$, which are defined as $\mathcal{T}_1 = \{ t \in [T] : \curvaturecolor{\dmu}(\langle\vx_t , \vtheta_\star \rangle) \geq \curvaturecolor{\dmu}( \langle \vx_t, \vtheta_{t+1} \rangle) \}$ and $\mathcal{T}_2 = [T] \setminus \mathcal{T}_1 $.

\noindent\textit{\underline{Bounding term $A_{1}$}}:
Note that by the integral formulation of Taylor's expansion, we have
\begin{align}
    \curvaturecolor{\dmu}(\langle \vx_t, \vtheta_\star \rangle) & = \curvaturecolor{\dmu}(\langle \vx_t, \vtheta_{t+1} \rangle) + \int_0^1 \ddmu ( \langle \vx_t, \vtheta_{t+1} \rangle + v \langle \vx_t, \vtheta_\star - \vtheta_{t+1}\rangle ) dv \langle \vx_t, \vtheta_\star - \vtheta_{t+1} \rangle\\
    & \leq \curvaturecolor{\dmu}(\langle \vx_t, \vtheta_{t+1} \rangle) + R_s \int_0^1 \curvaturecolor{\dmu} ( \langle \vx_t, \vtheta_{t+1} \rangle + v \langle \vx_t, \vtheta_\star - \vtheta_{t+1} \rangle) dv | \langle \vx_t, \vtheta_\star - \vtheta_{t+1} \rangle | \tag{Self-concordance of $\mu$}\\
    & \leq \curvaturecolor{\dmu}(\langle \vx_t, \vtheta_{t+1} \rangle) + R_s L_\mu \lVert \vx_t \rVert_{\mH_t^{-1}} \lVert \vtheta_\star - \vtheta_{t+1}\rVert_{\mH_{t+1}} \tag{\Cref{assumption:link-function}}\\
    & \leq \curvaturecolor{\dmu}( \langle \vx_t, \vtheta_{t+1} \rangle) + R_s L_\mu \big( \beta_{t+1}(\delta) + 2\eta \alpha C \big)\lVert \vx_t \rVert_{\mH_t^{-1}}, \tag{\Cref{theorem:confidence-sequence}}
\end{align}
where the first inequality holds by the self-concordance of $\mu$, the second inequality follows \Cref{assumption:link-function} and Cauchy-Schwarz  inequality, and the last inequality comes from \Cref{theorem:confidence-sequence}.
By substituting this into the term $A_1$, we can upper bound it as:
\begin{align}
    A_1 & \leq 2 \big( \beta_T(\delta) + 2\eta \alpha C \big) \sum_{t \in \mathcal{T}_1} \curvaturecolor{\dmu}(\langle \vx_t, \vtheta_{t+1} \rangle)\lVert \vx_t \rVert_{\mH_t^{-1}}\\
    & \quad + 2 R_s L_\mu \big( \beta_{T+1}(\delta) + 2\eta \alpha C \big)^2 \sum_{t \in \mathcal{T}_1} \lVert \vx_t \rVert_{\mH_t^{-1}}^2\\
    & = 2 \big( \beta_T(\delta) + 2\eta \alpha C \big) \sum_{t \in \mathcal{T}_1: \weightcolor{w_t} = 1} \curvaturecolor{\dmu}(\langle \vx_t, \vtheta_{t+1} \rangle)\lVert \vx_t \rVert_{\mH_t^{-1}}\\
    & \quad + 2 \big( \beta_T(\delta) + 2\eta \alpha C \big) \sum_{t \in \mathcal{T}_1: \weightcolor{w_t} < 1} \curvaturecolor{\dmu}(\langle \vx_t, \vtheta_{t+1} \rangle)\lVert \vx_t \rVert_{\mH_t^{-1}}\\
    & \quad + 2 R_s L_\mu \big( \beta_{T+1}(\delta) + 2\eta \alpha C \big)^2 \sum_{t \in \mathcal{T}_1} \lVert \vx_t \rVert_{\mH_t^{-1}}^2\\
    & \leq 2 \big( \beta_T(\delta) + 2\eta \alpha C \big) \sum_{t \in \mathcal{T}_1: \weightcolor{w_t} =1} \sqrt{\curvaturecolor{\dmu}( \langle \vx_t, \vtheta_\star \rangle)} \left\lVert \sqrt{\curvaturecolor{\dmu}( \langle \vx_t, \vtheta_{t+1} \rangle)} \vx_t \right\rVert_{\mH_t^{-1}} \tag{Definition of $\mathcal{T}_1$}\\
    & \quad + 2 \big( \beta_T(\delta) + 2\eta \alpha C \big) \sum_{t \in \mathcal{T}_1: \weightcolor{w_t} < 1} \curvaturecolor{\dmu}(\langle \vx_t, \vtheta_{t+1} \rangle)\lVert \vx_t \rVert_{\mH_t^{-1}}\\
    & \quad + 2 R_s L_\mu \big( \beta_{T+1}(\delta) + 2\eta \alpha C \big)^2 \sum_{t \in \mathcal{T}_1} \lVert \vx_t \rVert_{\mH_t^{-1}}^2,
\end{align}
where the last inequality holds since $\curvaturecolor{\dmu}( \langle \vx_t, \vtheta_{t+1} \rangle) \leq \curvaturecolor{\dmu}( \langle \vx_t, \vtheta_\star \rangle)$ for $t \in \mathcal{T}_1$.

\noindent\textit{\underline{Bounding term $A_{2}$}}:
A similar decomposition to the above yields the following:
\begin{align}
    A_2 & = 2 \big( \beta_T(\delta) + 2\eta \alpha C \big) \sum_{t \in \mathcal{T}_2} \curvaturecolor{\dmu}( \langle \vx_t, \vtheta_\star \rangle)\lVert \vx_t \rVert_{\mH_t^{-1}}\\
    & = 2 \big( \beta_T(\delta) + 2\eta \alpha C \big) \sum_{t \in \mathcal{T}_2: \weightcolor{w_t} = 1} \curvaturecolor{\dmu}(\langle \vx_t, \vtheta_\star \rangle)\lVert \vx_t \rVert_{\mH_t^{-1}}\\
    & \quad + 2 \big( \beta_T(\delta) + 2\eta \alpha C \big) \sum_{t \in \mathcal{T}_2: \weightcolor{w_t} < 1} \curvaturecolor{\dmu}(\langle \vx_t, \vtheta_\star \rangle)\lVert \vx_t \rVert_{\mH_t^{-1}}\\
    & \leq 2 \big( \beta_T(\delta) + 2\eta \alpha C \big) \sum_{t \in \mathcal{T}_2: \weightcolor{w_t} = 1} \sqrt{\curvaturecolor{\dmu}(\langle \vx_t, \vtheta_\star \rangle)} \left\lVert \sqrt{\curvaturecolor{\dmu}( \langle \vx_t, \vtheta_{t+1} \rangle)} \vx_t \right\rVert_{\mH_t^{-1}} \tag{Definition of $\mathcal{T}_2$}\\
    & \quad + 2 \big( \beta_T(\delta) + 2\eta \alpha C \big) \sum_{t \in \mathcal{T}_2: \weightcolor{w_t} < 1} \curvaturecolor{\dmu}(\langle \vx_t, \vtheta_{t+1} \rangle)\lVert \vx_t \rVert_{\mH_t^{-1}},
\end{align}
where the inequality holds because $\curvaturecolor{\dmu}( \langle \vx_t, \vtheta_{t+1} \rangle) > \curvaturecolor{\dmu}( \langle \vx_t, \vtheta_\star \rangle)$ for $t \in \mathcal{T}_2$.

Combining the upper bounds on $A_1$ and $A_2$, we have
\begin{align}
    A & \leq 2 \big( \beta_T(\delta) + 2\eta \alpha C \big) \sum_{t \in [T]: \weightcolor{w_t} = 1} \sqrt{\curvaturecolor{\dmu}(\langle \vx_t, \vtheta_\star \rangle)} \left\lVert \sqrt{\curvaturecolor{\dmu}( \langle \vx_t, \vtheta_{t+1} \rangle)} \vx_t \right\rVert_{\mH_t^{-1}}\\
    & \quad + 2 \big( \beta_T(\delta) + 2\eta \alpha C \big) \sum_{t \in [T]: \weightcolor{w_t} < 1} \curvaturecolor{\dmu}(\langle \vx_t, \vtheta_{t+1} \rangle)\lVert \vx_t \rVert_{\mH_t^{-1}}\\
    & \quad + 2 R_s L_\mu \big( \beta_{T+1}(\delta) + 2\eta \alpha C \big)^2 \sum_{t \in \mathcal{T}_1} \lVert \vx_t \rVert_{\mH_t^{-1}}^2\\
    & = 2 \big( \beta_T(\delta) + 2\eta \alpha C \big) \sum_{t \in [T]: \weightcolor{w_t} = 1} \sqrt{\curvaturecolor{\dmu}(\langle \vx_t, \vtheta_\star \rangle)} \left\lVert \sqrt{\curvaturecolor{\dmu}( \langle \vx_t, \vtheta_{t+1} \rangle)} \vx_t \right\rVert_{\mH_t^{-1}}\\
    & \quad + 2 \big( \beta_T(\delta) + 2\eta \alpha C \big) \sum_{t \in [T]: \weightcolor{w_t} < 1} \curvaturecolor{\dmu}(\langle \vx_t, \vtheta_{t+1} \rangle) \left\lVert \vx_t \right\rVert_{\mH_t^{-1}}\\
    & \quad + 2 R_s L_\mu \big( \beta_{T+1}(\delta) + 2\eta \alpha C \big)^2 \sum_{t \in \mathcal{T}_1: \weightcolor{w_t} = 1} \lVert \vx_t \rVert_{\mH_t^{-1}}^2\\
    & \quad + 2 R_s L_\mu \big( \beta_{T+1}(\delta) + 2\eta \alpha C \big)^2 \sum_{t \in \mathcal{T}_1: \weightcolor{w_t} < 1} \lVert \vx_t \rVert_{\mH_t^{-1}}^2\\
    & \leq 2 \big( \beta_T(\delta) + 2\eta \alpha C \big) \sum_{t \in [T]: \weightcolor{w_t} = 1} \sqrt{\curvaturecolor{\dmu}(\langle \vx_t, \vtheta_\star \rangle)} \left\lVert \sqrt{\curvaturecolor{\dmu}( \langle \vx_t, \vtheta_{t+1} \rangle)} \vx_t \right\rVert_{\mH_t^{-1}}\\
    & \quad + 2 \big( \beta_T(\delta) + 2\eta \alpha C \big) \sum_{t \in [T]: \weightcolor{w_t} < 1} \curvaturecolor{\dmu}(\langle \vx_t, \vtheta_{t+1} \rangle) \left\lVert \vx_t \right\rVert_{\mH_t^{-1}}\\
    & \quad + 2 R_s L_\mu \big( \beta_{T+1}(\delta) + 2\eta \alpha C \big)^2 \sum_{t \in [T]: \weightcolor{w_t} = 1} \lVert \vx_t \rVert_{\mH_t^{-1}}^2 \tag{Monotonicity of summation}\\
    & \quad + 2 R_s L_\mu \big( \beta_{T+1}(\delta) + 2\eta \alpha C \big)^2 \sum_{t \in [T]: \weightcolor{w_t} < 1} \lVert \vx_t \rVert_{\mH_t^{-1}}^2. \tag{Monotonicity of summation}
\end{align}

\paragraph{Analysis of term $B$}
For the term $B$, using self-concordance and boundedness of $\dmu$ (\Cref{assumption:link-function}), Cauchy-Schwarz inequality and \Cref{theorem:confidence-sequence}, we upper bound and decompose the term as:
\begin{align}
    B & = \sum_{t=1}^{T} \int_0^1 (1-v) \ddmu \big( \langle \vx_t, \vtheta_\star \rangle + v \langle \vx_t, \tilde{\vtheta}_t - \vtheta_\star \rangle \big) dv  \langle \vx_t, \tilde{\vtheta}_t - \vtheta_\star \rangle^2\\
    & \leq R_s \sum_{t=1}^{T} \int_0^1 (1-v) \curvaturecolor{\dmu} \big( \langle \vx_t, \vtheta_\star \rangle + v \langle \vx_t, \tilde{\vtheta}_t - \vtheta_\star \rangle \big) dv  \langle \vx_t, \tilde{\vtheta}_t - \vtheta_\star \rangle^2 
    \tag{self-concordance of $\mu$}\\
    & \leq \frac{R_s L_\mu}{2} \sum_{t=1}^{T} \langle \vx_t, \tilde{\vtheta}_t - \vtheta_\star \rangle^2 \tag{\Cref{assumption:link-function}}\\
    & \leq \frac{R_s L_\mu}{2} \sum_{t=1}^{T} \lVert \vx_t \rVert_{\mH_t^{-1}}^2 \lVert \tilde{\vtheta}_t - \vtheta_\star \rVert_{\mH_t}^2 \tag{Cauchy-Schwarz}\\
    & \leq 2 R_s L_\mu \big( \beta_T(\delta) + 2\eta \alpha C\big)^2 \sum_{t=1}^{T} \lVert \vx_t \rVert_{\mH_t^{-1}}^2 \tag{\Cref{theorem:confidence-sequence}}\\
    & = 2 R_s L_\mu \big( \beta_T(\delta) + 2\eta \alpha C \big)^2 \sum_{t\in[T]:\weightcolor{w_t} = 1} \lVert \vx_t \rVert_{\mH_t^{-1}}^2\\
    & \quad + 2 R_s L_\mu \big( \beta_T(\delta) + 2\eta \alpha C \big)^2 \sum_{t\in[T]: \weightcolor{w_t} < 1} \lVert \vx_t \rVert_{\mH_t^{-1}}^2\\
    & \leq 2 R_s L_\mu \big( \beta_{T+1}(\delta) + 2\eta \alpha C \big)^2 \sum_{t\in[T]:\weightcolor{w_t} = 1} \lVert \vx_t \rVert_{\mH_t^{-1}}^2 \tag{Monotonicity of $\beta_t$}\\
    & \quad + 2 R_s L_\mu \big( \beta_{T+1}(\delta) + 2\eta \alpha C \big)^2 \sum_{t\in[T]: \weightcolor{w_t} < 1} \lVert \vx_t \rVert_{\mH_t^{-1}}^2.
\end{align}
Combining the analyses of term $A$ and $B$, we get the desired result:
\begin{align}
    \Reg(T) & \leq 2 \big( \beta_T(\delta) + 2\eta \alpha C \big) \sum_{t \in [T]: \weightcolor{w_t} = 1} \sqrt{\curvaturecolor{\dmu}(\langle \vx_t, \vtheta_\star \rangle)} \left\lVert \sqrt{\curvaturecolor{\dmu}( \langle \vx_t, \vtheta_{t+1} \rangle)} \vx_t \right\rVert_{\mH_t^{-1}}\\
    & \quad +2 \big( \beta_T(\delta) + 2\eta \alpha C \big) \sum_{t \in [T]: \weightcolor{w_t} < 1} \curvaturecolor{\dmu}(\langle \vx_t, \vtheta_{t+1} \rangle) \left\lVert \vx_t \right\rVert_{\mH_t^{-1}}\\
    & \quad + 4 R_s L_\mu \big( \beta_{T+1}(\delta) + 2\eta \alpha C \big)^2 \sum_{t \in [T]: \weightcolor{w_t} = 1} \lVert \vx_t \rVert_{\mH_t^{-1}}^2\\
    & \quad + 4 R_s L_\mu \big( \beta_{T+1}(\delta) + 2\eta \alpha C \big)^2 \sum_{t \in [T]: \weightcolor{w_t} < 1} \lVert \vx_t \rVert_{\mH_t^{-1}}^2.
\end{align}
\qed

\subsection{\texorpdfstring{Bounding the Elliptical Summations in \Cref{lemma:decomposition}: $S_1$, $S_2$, and $S_3$}{Bounding the Elliptical Summations in Lemma~\ref{lemma:decomposition}}}

\begin{lemma}\label{lemma:epl-results}
    Applying the Hessian potential lemma (\Cref{lemma:revised-epl}) to the summations in \Cref{lemma:decomposition}, we have\small
    \begin{align}
        S_1 & \triangleq \sum_{t \in [T]: \weightcolor{w_t} < 1} \curvaturecolor{\dmu} \big(\langle \vx_t, \vtheta_{t+1} \rangle \big) \left\lVert \vx_t \right\rVert_{\mH_t^{-1}} \leq \frac{2d}{\alpha} \left( 1 + \frac{\alpha L_\mu}{\sqrt{\lambda}} \right) \log \left( 1 + \frac{\alpha L_\mu T}{2 \sqrt{d \lambda}} \right); \\
        S_2 & \triangleq \sum_{t \in {[T]}: \weightcolor{w_t}=1} \lVert \vx_t \rVert_{\mH_t^{-1}}^2 \leq 2d {\kappacolor{\kappa}} {\dispersioncolor{g_{\max}}} \left( 1 + \frac{\alpha L_\mu}{\sqrt{\lambda}} \right) \log \left( 1 + \frac{\alpha L_\mu T}{2 \sqrt{d \lambda}} \right) ;\\
        S_3 & \triangleq \sum_{t \in [T]: \weightcolor{w_t}<1} \lVert \vx_t \rVert_{\mH_t^{-1}}^2 \leq \frac{2d \kappacolor{\kappa}}{\sqrt{\lambda} \alpha} \left( 1 + \frac{\alpha L_\mu}{\sqrt{\lambda}} \right) \log \left( 1 + \frac{\alpha L_\mu T}{2 \sqrt{d \lambda}} \right).
    \end{align}
    
\end{lemma}

\begin{proof}
We establish each part separately.

\paragraph{Bounding $S_1$.}
In the term $S_{1}$, we only consider the time steps in which $\weightcolor{w_t} < 1$, which implies $\weightcolor{w_t} = \alpha {\dispersioncolor{g(\tau_t)}} / \lVert \vx_t \rVert_{\mH_t^{-1}}$, or $\weightcolor{w_t} \lVert \vx_t \rVert_{\mH_t^{-1}} / (\alpha {\dispersioncolor{g(\tau_t)}}) = 1$.
By multiplying $\weightcolor{w_t} \lVert \vx_t \rVert_{\mH_t^{-1}} / (\alpha {\dispersioncolor{g(\tau_t)}})$, which is 1, to each summand, we can directly apply \Cref{lemma:revised-epl} as:
\begin{align}
    S_{1} & = \sum_{t \in [T]: \weightcolor{w_t} < 1} \curvaturecolor{\dmu} \big(\langle \vx_t, \vtheta_{t+1} \rangle \big) \left\lVert \vx_t \right\rVert_{\mH_t^{-1}}\\
    & = \frac{1}{\alpha} \sum_{t \in [T], \weightcolor{w_t} < 1} \frac{\weightcolor{w_t}}{{\dispersioncolor{g(\tau_t)}}} \curvaturecolor{\dmu}(\langle \vx_t, \vtheta_{t+1} \rangle)\lVert \vx_t \rVert_{\mH_t^{-1}}^2\\
    & = \frac{1}{\alpha} \sum_{t \in [T], \weightcolor{w_t} < 1} \left\lVert \sqrt{\frac{\weightcolor{w_t}}{{\dispersioncolor{g(\tau_t)}}} \curvaturecolor{\dmu}( \langle \vx_t, \vtheta_{t+1} \rangle)} \vx_t \right\rVert_{\mH_t^{-1}}^2\\
    & \leq \frac{1}{\alpha} \sum_{t =1}^{T} \left\lVert \sqrt{ \frac{\weightcolor{w_t}}{{\dispersioncolor{g(\tau_t)}}} \curvaturecolor{\dmu}\big(\langle \vx_t, \vtheta_{t+1} \rangle \big)} \vx_t \right\rVert_{\mH_t^{-1}}^2 \tag{Monotonicity of summation}\\
    & \leq  \frac{2d}{\alpha} \left( 1 + \frac{\alpha L_\mu}{\sqrt{\lambda}} \right) \log \left( 1 + \frac{\alpha L_\mu T}{2 \sqrt{d \lambda}} \right). \tag{\Cref{lemma:revised-epl}}
\end{align}

\paragraph{Bounding $S_2$.}
We insert $\weightcolor{w_t}$ (whose value is 1) into the summand of $S_2$ and modify it so that \Cref{lemma:revised-epl} can be applied as:
\begin{align}
    S_{2} & = \sum_{t \in \mathcal{T}_1: \weightcolor{w_t}=1} \lVert \vx_t \rVert_{\mH_t^{-1}}^2 = \sum_{t \in \mathcal{T}_1: \weightcolor{w_t}=1} \lVert \sqrt{\weightcolor{w_t}} \vx_t \rVert_{\mH_t^{-1}}^2 \tag{$\weightcolor{w_t} = 1$}\\
    & \leq {\kappacolor{\kappa}}{\dispersioncolor{g_{\max}}} \sum_{t \in \mathcal{T}_1: \weightcolor{w_t}=1} \left\lVert \sqrt{\frac{\weightcolor{w_t} }{\dispersioncolor{g(\tau_t)}}\curvaturecolor{\dmu}(\langle \vx_t , \vtheta_{t+1} \rangle)} \vx_t \right\rVert_{\mH_t^{-1}}^2\\
    & \leq {\kappacolor{\kappa}}{\dispersioncolor{g_{\max}}} \sum_{t =1}^{T} \left\lVert \sqrt{\frac{\weightcolor{w_t} }{\dispersioncolor{g(\tau_t)}}\curvaturecolor{\dmu}(\langle \vx_t , \vtheta_{t+1} \rangle)} \vx_t \right\rVert_{\mH_t^{-1}}^2 \tag{Monotonicity of summation}\\
    & \leq 2d {\kappacolor{\kappa}}{\dispersioncolor{g_{\max}}} \left( 1 + \frac{\alpha L_\mu}{\sqrt{\lambda}} \right) \log \left( 1 + \frac{\alpha L_\mu T}{2 \sqrt{d \lambda}} \right). \tag{\Cref{lemma:revised-epl}}
\end{align}

\paragraph{Bounding $S_3$.}
As in $S_1$, we first multiply $\weightcolor{w_t} \lVert \vx_t \rVert_{\mH_t^{-1}}/ (\alpha { \dispersioncolor{g(\tau_t)}})$ to each summand as:
\begin{equation}
    S_3 = \sum_{t\in[T]: \weightcolor{w_t} <1 } \lVert \vx_t \rVert_{\mH_t^{-1}}^2 = \frac{1}{\alpha} \sum_{t\in[T]: \weightcolor{w_t} <1 } \frac{\weightcolor{w_t}}{{ \dispersioncolor{g(\tau_t)}}}\lVert \vx_t \rVert_{\mH_t^{-1}}^3.
\end{equation}
We have $\lVert \vx_t \rVert_{\mH_t^{-1}} \leq \lVert \vx_t \rVert_{(\lambda \mI_d)^{-1}}= \frac{1}{\sqrt{\lambda}} \lVert \vx_t \rVert_2 \leq \frac{1}{\sqrt{\lambda}}$ since $\mH_t \succeq \lambda \mI_d $, then we can apply \Cref{lemma:revised-epl} to $S_3$ as:
\begin{align}
    S_3 & = \frac{1}{\alpha} \sum_{t\in[T]: \weightcolor{w_t} <1 } \frac{\weightcolor{w_t}}{{ \dispersioncolor{g(\tau_t)}}}\lVert \vx_t \rVert_{\mH_t^{-1}}^3\\
    & \leq \frac{{\kappacolor{\kappa}}}{\sqrt{\lambda} \alpha} \sum_{t\in[T]: \weightcolor{w_t} <1 } \frac{\weightcolor{w_t} \curvaturecolor{\dmu}(\langle \vx_t, \vtheta_{t+1} \rangle)}{{ \dispersioncolor{g(\tau_t)}}}\lVert \vx_t \rVert_{\mH_t^{-1}}^2\\
    & \leq \frac{2d {\kappacolor{\kappa}}}{\sqrt{\lambda}\alpha} \left( 1 + \frac{\alpha L_\mu}{\sqrt{\lambda}} \right) \log \left( 1 + \frac{\alpha L_\mu T}{2 \sqrt{d \lambda}} \right). \tag{\Cref{lemma:revised-epl}}
\end{align}

\end{proof}

\subsection{Hessian Potential Lemma}
\label{sec:hessian-potential-proof}

We provide the Hessian potential lemma and prove it.
\begin{lemma}[Hessian Potential Lemma]\label{lemma:revised-epl}
    Let $\{\vx_t\}_{t=1}^T$ be a sequence of arms selected by \texttt{HCW-GLB-OMD} (\Cref{algorithm:CW-GLB-OMD}), then the determinant ratio of the weighted Hessian matrices $\mH_t$ satisfies:
    \begin{equation}\label{eqn:det-ratio}
        \log \frac{\det(\mH_{t+1})}{\det(\mH_1)} \leq 2 d \log \left( 1 + \frac{\alpha L_\mu}{2 \sqrt{d \lambda}} t \right).
    \end{equation}
    Consequently, the weighted elliptical potential sum is bounded as:
    \begin{equation}
        \sum_{t=1}^{T} \frac{\weightcolor{w_t} \curvaturecolor{\dmu}(\langle \vx_t, \vtheta_{t+1} \rangle)}{\dispersioncolor{g(\tau_t)}} \lVert \vx_t \rVert_{\mH_{t}^{-1}}^2 \leq 2d \left( 1 + \frac{\alpha L_\mu}{\sqrt{\lambda}} \right) \log \left( 1 + \frac{\alpha L_\mu T}{2 \sqrt{d \lambda}} \right),
    \end{equation}
    which is $\gO(d \log T)$ if the parameters are set as in \Cref{theorem:confidence-sequence} and \Cref{theorem:regret}, treating $S$, $R_s$, and $L_\mu$ as constants.
\end{lemma}
\begin{proof}
When $\vx_t=0$, set $\weightcolor{w_t}=1$; the Hessian increment and Hessian potential are both zero. We may therefore restrict divisions by $\lVert\vx_t\rVert_{\mH_t^{-1}}$ to nonzero arms.
    By AM-GM inequality, determinant of PSD matrix can be bounded w.r.t. the trace as:
    \begin{equation}
        \det(\mH_{t+1}) = \prod_{i=1}^{d} \lambda_i \leq \left( \frac{\sum_{i=1}^{d} \lambda_i}{d} \right)^d = \left( \frac{\text{tr}(\mH_{t+1})}{d} \right)^d,
    \end{equation}
    where $\lambda_i$'s are eigenvalues of $\mH_{t+1}$.
    
    As in the standard elliptical potential lemma (\Cref{lemma:epl}), we upper bound the one-step increment of the trace, $\text{tr} (\mH_{t+1}) - \text{tr}(\mH_t)$.
    While the proof of standard EPL bounds this increment uniformly by a constant, we instead derive the following recursive inequality:
    \begin{align}
        \text{tr}(\mH_{t+1}) - \text{tr} (\mH_t) & = \frac{\weightcolor{w_t} \curvaturecolor{\dmu}(\langle \vx_t, \vtheta_{t+1} \rangle)}{\dispersioncolor{g(\tau_t)}} \lVert \vx_t \rVert_2^2 \tag{$\mH_{t+1} = \mH_t + \frac{\weightcolor{w_t}}{\dispersioncolor{g(\tau_t)}} \curvaturecolor{\dmu}(\langle \vx_t , \vtheta_{t+1} \rangle) \vx_t \vx_t^\top$}\\
        & \leq \frac{\alpha \curvaturecolor{\dmu}(\langle \vx_t, \vtheta_{t+1}  \rangle)}{\lVert \vx_t \rVert_{\mH_t^{-1}}} \lVert \vx_t \rVert_2^2 \\
        & \leq \alpha L_\mu \lVert \vx_t \rVert_{\mH_t} \leq \alpha L_\mu \sqrt{\text{tr} ( \mH_t)},
    \end{align}
    where the first inequality holds by the fact that $\weightcolor{w_t} \leq \alpha \dispersioncolor{g(\tau_t)}/\lVert \vx_t \rVert_{\mH_t^{-1}}$, the second inequality holds by the Cauchy-Schwarz inequality($\lVert \vx_t \rVert_2^2 \leq \lVert \vx_t \rVert_{\mH_t} \lVert \vx_t \rVert_{\mH_t^{-1}}$), and the final inequality holds since $\lVert \vx_t \rVert_{\mH_t}^2 = \vx_t^\top \mH_t \vx_t \leq \lambda_{\max}(\mH_t) \lVert \vx_t \rVert_2^2 \leq \text{tr}(\mH_t)$ when $\lVert \vx_t \rVert_2 \leq 1$.
    Dividing both sides by $\sqrt{\text{tr}(\mH_{t+1})} + \sqrt{\text{tr}(\mH_t)}$, we get  
    \begin{align}
        \sqrt{\text{tr} (\mH_{t+1})} - \sqrt{\text{tr} ( \mH_t)} \leq \frac{\alpha L_\mu \sqrt{\text{tr}{(\mH_t)}}}{\sqrt{\text{tr} (\mH_{t+1})} + \sqrt{\text{tr} ( \mH_t)}} \leq \frac{\alpha L_\mu}{2},
    \end{align}
    where the last inequality holds by the monotonicity of $\text{tr}(\mH_{t})$.
    Telescoping both sides and using the fact $\text{tr}(\mH_1)=d\lambda$, we obtain:
    \begin{equation}
        \sqrt{\text{tr}(\mH_{t+1})} \leq \sqrt{d\lambda} + \frac{\alpha L_\mu}{2} t.
    \end{equation}
    Combining everything, we finally get the first bound:
    \begin{align}
        \log \frac{\det(\mH_{t+1})}{\det(\mH_1)} & \leq 2d \log \frac{\sqrt{\text{tr}(\mH_{t+1})}}{\sqrt{d \lambda}}\\
        & \leq 2d \log \left( 1 + \frac{\alpha L_\mu}{2 \sqrt{d\lambda}} t \right).
    \end{align}

    For the second bound, let 
    \begin{equation}
        q_t \coloneq \frac{\weightcolor{w_t} \curvaturecolor{\dmu}(\langle \vx_t, \vtheta_{t+1} \rangle)}{\dispersioncolor{g(\tau_t)}} \lVert \vx_t \rVert_{\mH_t^{-1}}^2.
    \end{equation}
    Since $\weightcolor{w_t}\leq\alpha \dispersioncolor{g(\tau_t)}/\lVert\vx_t\rVert_{\mH_t^{-1}}$ and $\mH_t\succeq\lambda\mI_d$,
    \begin{equation}
        q_t\leq\alpha L_\mu\lVert\vx_t\rVert_{\mH_t^{-1}}
        \leq\frac{\alpha L_\mu}{\sqrt{\lambda}}\eqqcolon Q.
    \end{equation}
    From the elementary inequality $\frac{x}{x+1} \leq \log (1+x)$ that holds for $x>0$, it follows that
    \begin{equation}
        q_t \leq ( 1 + q_t ) \log ( 1 + q_t ) \leq ( 1 + Q ) \log ( 1 + q_t ).
    \end{equation}
    Thus our quantity of interest, $\sum_{t=1}^{T} q_t$, is bounded as:
    \begin{align}
        \sum_{t=1}^{T} \frac{\weightcolor{w_t} \curvaturecolor{\dmu}(\langle \vx_t, \vtheta_{t+1} \rangle)}{\dispersioncolor{g(\tau_t)}} \lVert\vx_t \rVert_{\mH_t^{-1}}^2 & \leq \left(1 + \frac{\alpha L_\mu}{\sqrt{\lambda}} \right) \sum_{t=1}^{T}\log \left( 1 + \frac{\weightcolor{w_t} \curvaturecolor{\dmu}(\langle \vx_t, \vtheta_{t+1} \rangle)}{\dispersioncolor{g(\tau_t)}} \lVert\vx_t \rVert_{\mH_t^{-1}}^2 \right)\\
        & =  \left( 1 + \frac{\alpha L_\mu}{\sqrt{\lambda}} \right) \log \frac{\det(\mH_{T+1})}{\det(\mH_1)}\\
        & \leq 2d \left( 1 + \frac{\alpha L_\mu}{\sqrt{\lambda}} \right) \log \left( 1 + \frac{\alpha L_\mu T}{2 \sqrt{d \lambda}} \right),
    \end{align}
    where the equation is from the matrix determinant lemma:
    \begin{align}
        \det(\mH_{t+1}) & = \det(\mH_t) \left(1 + \frac{\weightcolor{w_t} \curvaturecolor{\dmu}(\langle \vx_t, \vtheta_{t+1} \rangle)}{\dispersioncolor{g(\tau_t)}} \lVert\vx_t \rVert_{\mH_t^{-1}}^2 \right)\\
        &= \det(\mH_1) \prod_{s=1}^{t} \left(1 + \frac{\weightcolor{w_s} \curvaturecolor{\dmu}(\langle \vx_s, \vtheta_{s+1} \rangle)}{\dispersioncolor{g(\tau_s)}} \lVert\vx_s \rVert_{\mH_s^{-1}}^2 \right).
    \end{align}
\end{proof}

\newpage
\section{\texorpdfstring{Proof of \Cref{thm:lower-bound}: Corruption-Free Regret Lower Bound}{Corruption-Free Regret Lower Bound}}
\label{app:proof-lower-bound}

The proof follows the peeling technique of \citet[Theorem 4.1]{he2025lowerbound}, applied to the dispersion parameters.
Fix $T\geq2$, $d\geq2$, and $\vtheta_\star\in\Theta$ with $\lVert\vtheta_\star\rVert_2\geq1$.
Let $\{\dispersioncolor{g_t}\coloneqq\dispersioncolor{g(\tau_t)}\}_{t\in[T]}$ be a prescribed sequence of valid, positive dispersions, which the policy may know in advance.
Recall that $\kappacolor{\kappa_\star}=1/\curvaturecolor{\dmu}(\lVert\vtheta_\star\rVert_2)$ and $\dispersioncolor{g_{\max}}=\max_{t\in[T]}\dispersioncolor{g_t}$.
Throughout, $\curvaturecolor{\dmu}(0)>0$.

We first partition $[T]=\bigcup_{\ell=0}^L\gT_\ell$ with $L=\ceil{\log_2T}$, where
\begin{align}
    \gT_0&=\left\{t\in[T]:\dispersioncolor{g_t}\leq\frac{\dispersioncolor{g_{\max}}}{T}\right\},\\
    \gT_\ell&=\left\{t\in[T]:\frac{2^{\ell-1}\dispersioncolor{g_{\max}}}{T}<\dispersioncolor{g_t}\leq\frac{2^\ell\dispersioncolor{g_{\max}}}{T}\right\},\qquad\ell\in[L].
\end{align}
We refer to each $\gT_\ell$ as a \emph{group}.
For each $\ell\in[L]$, let $\dispersioncolor{g_\ell}=2^\ell\dispersioncolor{g_{\max}}/T$ and $n_\ell=|\gT_\ell|$.
We use the following lemma, whose proof is given in \Cref{app:sphere-lower-bound}:
\begin{lemma}\label{lem:sphere-local-lower-bound}
    For any $T\geq1$, suppose that $\lVert\vtheta_\star\rVert_2\geq1$, $\gX_t=\gS^{d-1}(1)$, and $\dispersioncolor{g_t}\in(\dispersioncolor{g}/2,\dispersioncolor{g}]$ for all $t\in[T]$, for some $\dispersioncolor{g}>0$.
    If
    \begin{equation}\label{eq:sphere-feasibility}
        T\geq d^2\dispersioncolor{g}\kappacolor{\kappa_\star}\max\{1,R_s^4\},
    \end{equation}
    then
    \begin{equation}\label{eq:sphere-lower-bound}
        \inf_\pi\sup_{\substack{\tilde\vtheta_\star\in\Theta\\
            \lVert\tilde\vtheta_\star-\vtheta_\star\rVert_2\leq\varepsilon_T}}
        \E_{\tilde\vtheta_\star}\Reg^\pi(T)
        \gtrsim_{S,R_s}d\sqrt{\frac{T\dispersioncolor{g}}{\kappacolor{\kappa_\star}}},
        \qquad \varepsilon_T^2=\frac d{64}\sqrt{\frac{\dispersioncolor{g}\kappacolor{\kappa_\star}}{T}}.
    \end{equation}
\end{lemma}

To apply the lemma, we consider the groups satisfying
\begin{equation}\label{eq:lower-admissibility}
    n_\ell\geq d^2\dispersioncolor{g_\ell}\kappacolor{\kappa_\star}\max\{1,R_s^4\}.
\end{equation}
These groups contain at least half the total dispersion whenever
\begin{equation}\label{eq:lower-total-dispersion}
    \sum_{t=1}^T\dispersioncolor{g_t}
    \geq2\dispersioncolor{g_{\max}}
       +\frac{32}{3}d^2\dispersioncolor{g_{\max}^2}\kappacolor{\kappa_\star}\max\{1,R_s^4\}.
\end{equation}
Indeed, $\sum_{t\in\gT_0}\dispersioncolor{g_t}\leq\dispersioncolor{g_{\max}}$, and a group failing \eqref{eq:lower-admissibility} contributes at most
$d^2\dispersioncolor{g_\ell^2}\kappacolor{\kappa_\star}\max\{1,R_s^4\}$.
Since $\sum_{\ell=1}^L\dispersioncolor{g_\ell^2}\leq16\dispersioncolor{g_{\max}^2}/3$, the dispersion outside the groups satisfying \eqref{eq:lower-admissibility} is at most
\begin{equation}
    \dispersioncolor{g_{\max}}+\frac{16}{3}d^2\dispersioncolor{g_{\max}^2}\kappacolor{\kappa_\star}\max\{1,R_s^4\}
       \leq\frac12\sum_{t=1}^T\dispersioncolor{g_t}.
\end{equation}
Choose $\ell$ satisfying \eqref{eq:lower-admissibility} for which $n_\ell\dispersioncolor{g_\ell}$ is maximal.
There are at most $L$ groups, so
\begin{equation}\label{eq:best-bin-mass}
    n_\ell\dispersioncolor{g_\ell}\geq\frac1{2L}\sum_{t=1}^T\dispersioncolor{g_t}.
\end{equation}

The arm sets are constructed as follows:
\begin{equation}
    \gX_t=\begin{cases}
        \gS^{d-1}(1),&t\in\gT_\ell,\\
        \{\ve_0\},&t\notin\gT_\ell,
    \end{cases}
    \qquad
    \ve_0=\frac{\vtheta_\star}{\lVert\vtheta_\star\rVert_2}
    .
\end{equation}
Under $\vtheta_\star$, both arm sets have optimal projection $\lVert\vtheta_\star\rVert_2$, and hence $\curvaturecolor{\dmu_{t,\star}}=\kappacolor{\kappa_\star}^{-1}$ for all $t$.
The parameters constructed in the proof of \Cref{lem:sphere-local-lower-bound} have the same projection on $\ve_0$.
Thus, the singleton rounds have the same reward law under all of them and incur no regret.
For any policy, these observations can be simulated as additional randomization independent of the signs.
The average-regret bound over the constructed parameters therefore applies to the selected rounds, and the supremum over the stated neighborhood is at least this average.
With $\varepsilon_T^2=\frac d{64}\sqrt{\dispersioncolor{g_\ell}\kappacolor{\kappa_\star}/n_\ell}$, we obtain
\begin{align}
    \inf_\pi\sup_{\substack{\tilde\vtheta_\star\in\Theta\\
        \lVert\tilde\vtheta_\star-\vtheta_\star\rVert_2\leq\varepsilon_T}}
        \E_{\tilde\vtheta_\star}\Reg^\pi(T)
    &\gtrsim_{S,R_s} d\sqrt{\frac{n_\ell\dispersioncolor{g_\ell}}{\kappacolor{\kappa_\star}}} \tag{\Cref{lem:sphere-local-lower-bound}} \\
    &\gtrsim_{S,R_s}\frac d{\sqrt L}\sqrt{\sum_{t=1}^T\dispersioncolor{g_t}\curvaturecolor{\dmu_{t,\star}}}. \tag{Eqn.~\eqref{eq:best-bin-mass}, $\curvaturecolor{\dmu_{t,\star}}=\kappacolor{\kappa_\star}^{-1}$ for all $t$}
\end{align}

For constant dispersion $\dispersioncolor{g_t}=\dispersioncolor{g}$, applying \Cref{lem:sphere-local-lower-bound} directly to all $T$ rounds gives the log-free rate $\Omega_{S,R_s}(d\sqrt{\dispersioncolor{g}T/\kappacolor{\kappa_\star}})$ with $\varepsilon_T^2=\frac d{64}\sqrt{\dispersioncolor{g}\kappacolor{\kappa_\star}/T}$.\footnote{In particular, for logistic rewards with $\dispersioncolor{g}=1$ and $S\geq1$, \eqref{eq:sphere-feasibility} becomes $T\geq d^2\kappacolor{\kappa_\star}$.
This and the resulting lower bound agree with \citet[Theorem 2]{abeille2021logistic}, up to constants depending polynomially on $S$.}
\qed

\begin{remark}[Comparison with \citet{he2025lowerbound}]
    Here, we retain all $d$ coordinates in one group, whereas \citet{he2025lowerbound} allocate about $d/L$ coordinates to each group and combine independent instances.
    The factor $L$ in \eqref{eq:best-bin-mass} appears under the square root, giving $d/\sqrt L$ instead of $d/L$.
    For Gaussian rewards, $\mu(z)=z$, $R_s=0$, and $\dispersioncolor{g_t}=\sigma_t^2$, yielding $\Omega_S(d\sqrt{\sum_t\sigma_t^2/\log T})$.
    When $S\geq1$ and $\sigma_t^2\leq1$, the condition $\sum_t\sigma_t^2\geq1+384d^2$ in \citet[Theorem 4.1]{he2025lowerbound} implies \eqref{eq:lower-total-dispersion}.
    This recovers their prefixed-variance rate with a smaller logarithmic loss.
    We also note that their result assumes bounded noise and variance upper bounds, while ours uses Gaussian rewards with exact variances.
\end{remark}

\subsection{\texorpdfstring{Proof of \Cref{lem:sphere-local-lower-bound}: Local Minimax Lower Bound for Comparable Dispersions}{Proof of Lemma~\ref{lem:sphere-local-lower-bound}: Local Minimax Lower Bound for Comparable Dispersions}}
\label{app:sphere-lower-bound}
We follow \citet[Appendix D.1]{abeille2021logistic}.
It suffices to consider $\lVert\vtheta_\star\rVert_2=S$: for a smaller norm, apply the same construction in $\gB^d(\lVert\vtheta_\star\rVert_2)\subseteq\Theta$.
By spherical symmetry, take $\vtheta_\star=S\ve_1$.

Let $P_{z,\dispersioncolor{g_t}}$ be the reward law with natural parameter $z$ and dispersion $\dispersioncolor{g_t}$.
Taylor's integral remainder and self-concordance replace the Bernoulli bound in their Lemma 6 by
\begin{align}\label{eq:lower-exponential-kl}
    \KL(P_{z,\dispersioncolor{g_t}}\Vert P_{z',\dispersioncolor{g_t}})
    &=\frac{(z'-z)^2}{\dispersioncolor{g_t}}\int_0^1(1-v)\curvaturecolor{\dmu}(z+v(z'-z))\,dv\\
    &\leq\frac{e^{R_s|z'-z|}\curvaturecolor{\dmu}(z)}{2\dispersioncolor{g_t}}(z'-z)^2.
\end{align}
The counterparts of their Proposition 6 follow by integrating self-concordance: for $z\in[-S,S]$,
\begin{equation}\label{eq:sphere-curvature-regret}
    |\curvaturecolor{\dmu}(z)-\kappacolor{\kappa_\star}^{-1}|
           \leq R_s[\mu(S)-\mu(z)], \qquad
    \mu(S)-\mu(z)
    \geq\frac{S-z}{1+2R_sS}\max\{\kappacolor{\kappa_\star}^{-1},\curvaturecolor{\dmu}(z)\}.
\end{equation}
Together with $S-\langle\vx,\vtheta\rangle=(S/2)\lVert\vx-\vtheta/S\rVert_2^2$ for unit $\vx$ and $\lVert\vtheta\rVert_2=S$, these give the required regret--distance bounds.

Set $h^2=\frac1{128}\sqrt{\dispersioncolor{g}\kappacolor{\kappa_\star}/T}$ and define
\begin{equation}
    \vtheta_v=\sqrt{S^2-(d-1)h^2}\,\ve_1
        +h\sum_{j=2}^dv_j\ve_j,
        \qquad v=(v_2,\ldots,v_d)\in\{-1,1\}^{d-1}.
\end{equation}
By \eqref{eq:sphere-feasibility}, $(d-1)h^2\leq1/128\leq S^2/4$ and $2R_sh\leq1/8$.
Thus $\lVert\vtheta_v\rVert_2=S$ and $\lVert\vtheta_v-\vtheta_\star\rVert_2^2\leq2(d-1)h^2\leq\varepsilon_T^2$.

Fix any policy $\pi$, let $\mathbb P_v$ be its history law under $\vtheta_v$, and write $v^{(j)}$ for $v$ with sign $j$ flipped.
Suppose that its average regret satisfies
\begin{equation}
    \overline R\coloneqq\frac1{2^{d-1}}\sum_v\E_{\vtheta_v}\Reg^\pi(T)
    <\frac{d\sqrt{T\dispersioncolor{g}/\kappacolor{\kappa_\star}}}
        {4096S(1+2R_sS)}.
\end{equation}
As in their Lemma 5, the chain rule for KL, \eqref{eq:lower-exponential-kl}--\eqref{eq:sphere-curvature-regret}, and
$\sum_{j=2}^d\langle\vx,\ve_j\rangle^2\leq2\lVert\vx-\vtheta_v/S\rVert_2^2+2(d-1)h^2/S^2$ give
\begin{align}
    \frac1{2^{d-1}}\sum_v\sum_{j=2}^d\KL(\mathbb P_v\Vert\mathbb P_{v^{(j)}})
    &\leq\frac{5h^2}{\dispersioncolor{g}}
       \left[\frac{2(d-1)h^2T}{S^2\kappacolor{\kappa_\star}}
             +\frac{5(1+2R_sS)}S\overline R\right]\\
    &\leq d\left[\frac{10}{128^2}+\frac{25}{128\cdot4096}\right]
       <\frac d{128},
\end{align}
where $\dispersioncolor{g_t}>\dispersioncolor{g}/2$, $4e^{1/8}<5$, and $S\geq1$ were used.
The sign-pairing argument in their Lemma 4 then gives at least $d/8$ sign errors on average, counting zero average coordinates as errors.
Applying the regret--distance bound above and Jensen's inequality, as in their Lemma 3, yields
\begin{equation}
    \overline R\geq\frac{dTh^2}{16S\kappacolor{\kappa_\star}(1+2R_sS)}
       =\frac{d\sqrt{T\dispersioncolor{g}/\kappacolor{\kappa_\star}}}
              {2048S(1+2R_sS)},
\end{equation}
contradicting the assumed bound.
Thus $\overline R$ is at least the assumed threshold.
Since $[S(1+2R_sS)]^{-1}$ decreases with $S$, the smaller-ball argument gives \eqref{eq:sphere-lower-bound} for every $\vtheta_\star$ in the lemma.
All alternatives have the same projection on $\vtheta_\star/\lVert\vtheta_\star\rVert_2$, so interleaved singleton rounds contribute no KL or regret.
\qed

\newpage
\section{\texorpdfstring{Proof of \Cref{thm:lower-bound-corruption}: Corruption-Dependent Lower Bound}{Corruption-Dependent Lower Bound}}
\label{app:lower-bound-corruption}

We follow the indistinguishability argument of \citet[Theorem 3]{bogunovic2021corrupted}.\footnote{The reward-erasure argument also appears in \citet[Section 5]{lykouris2018corruption}; related indistinguishability arguments under stochastic-regret assumptions are given by \citet[Theorem 4]{lykouris2018corruption} and \citet[Theorem 2]{auer-chiang}.}
The extension to stochastic GLM rewards uses a conditional coupling whose expected corruption equals the reward gap; stopping before exceeding $C$ then enforces the budget almost surely.

Fix any policy $\pi$ and a prescribed valid dispersion sequence with $\dispersioncolor{g(\tau_t)}>0$, which the policy may know in advance along with $C$.

\paragraph{Construction.}

Set $\Delta_\star=\mu(S)-\mu(S/2)$ and choose
$\Delta=dC/(4T)$. For $T\gtrsim dC$
Let $\ve_1,\ldots,\ve_d$ be the standard basis of $\mathbb R^d$, and define
(more precisely, $T\geq dC/(4\Delta_\star)$), we have $\Delta \leq \Delta_\star$.
\[
    \ve_0=\frac1{\sqrt d}\sum_{j=1}^d\ve_j,
    \qquad
    \vv_i=\sqrt{\frac d{d-1}}
        \left(\ve_i-\frac{\ve_0}{\sqrt d}\right).
\]
The vectors $\vv_1,\ldots,\vv_d$ have unit norm, and
$\langle\vv_i,\vv_j\rangle=-1/(d-1)$ for $i\neq j$.

For $\phi\in(0,\pi/4]$, define
\[
    \vx_i=\ve_0\cos\phi+\vv_i\sin\phi,
    \qquad
    \vtheta^{(i)}=S\vx_i.
\]
All arms have unit norm and all parameters have norm $S$.
In instance $i$, the arm $i$ satisfies
$\langle\vx_i,\vtheta^{(i)}\rangle = S \eqcolon z_+$, while every other arm satisfies
$\langle\vx_j,\vtheta^{(i)}\rangle
=S[1-d\sin^2\phi/(d-1)] \eqcolon z_-$ for $j\neq i$.
Thus, arm $i$ is uniquely optimal, with gap $\mu(z_+)-\mu(z_-)$.

Since $0<\Delta\leq\Delta_\star$, continuity guarantees an angle
$\phi\in(0,\pi/4]$ such that $\mu(z_+)-\mu(z_-)=\Delta$.
With this choice, arm $i$ is uniquely optimal in instance $i$,
and every other arm has reward gap $\Delta$.
The resulting arm set $\gX=\{\vx_1,\ldots,\vx_d\}$ is fixed across rounds and independent of the policy and dispersion sequence.

\paragraph{Coupling the Rewards.}
We construct a proposed replacement that makes an optimal-arm reward
look like a suboptimal-arm reward, with expected corruption cost $\Delta$.

Let $P_{+,t}$ and $P_{-,t}$ be the reward laws corresponding to arms with $z_+$ and $z_-$, respectively, with the same dispersion $\dispersioncolor{g(\tau_t)}$.
By Eqn.~\eqref{eqn:GLM}, their likelihood ratio is
\begin{equation}
    \frac{dP_{+,t}}{dP_{-,t}}(r)
    =\exp\left(\frac{(z_+-z_-)r-[m(z_+)-m(z_-)]}{\dispersioncolor{g(\tau_t)}}\right).
\end{equation}
This ratio is increasing in $r$, so larger rewards are relatively more likely under $P_{+,t}$.
Consequently, $P_{+,t}$ stochastically dominates $P_{-,t}$: their CDFs satisfy $F_{+,t}\leq F_{-,t}$ and their generalized inverses satisfy $F_{+,t}^{-1}\geq F_{-,t}^{-1}$, where $F^{-1}(u)=\inf\{y:F(y)\geq u\}$.

The replacement matches the same percentile in the two distributions.
After observing an optimal-arm reward $R_t\sim P_{+,t}$, the adversary draws an independent $V_t\sim\mathrm{Unif}(0,1)$ and sets
\begin{equation}
    U_t=F_{+,t}(R_t^-)
       +V_t\big[F_{+,t}(R_t)-F_{+,t}(R_t^-)\big],\qquad
    Y_t=F_{-,t}^{-1}(U_t).
\end{equation}
Here $F(R^-)$ is the left limit of the CDF.
The randomized distributional transform makes $U_t$ uniform and gives $F_{+,t}^{-1}(U_t)=R_t$ almost surely, including for discrete or mixed reward laws.
Consequently, the replacement obeys the observation order and satisfies
\begin{equation}
    Y_t\sim P_{-,t},\qquad Y_t\leq R_t\ \text{almost surely},\qquad
    \E[R_t-Y_t]=\mu(z_+)-\mu(z_-)=\Delta.
\end{equation}
\Cref{app:conditional-replacement} gives thinning rules with these same properties for Bernoulli and Poisson feedback.

\paragraph{Enforcing the Budget.}
The adversary proposes the replacement only when the optimal arm is pulled, and applies it only if the cost fits the remaining budget.
At the first proposal that would exceed $C$, it leaves the reward unchanged and ceases corruption.
This ensures $\sum_t|c_t|\leq C$ on every realization, including for unbounded rewards.

\paragraph{Indistinguishability and Regret.}
Let the optimal-arm index $I$ be uniform on $[d]$, and first consider the ideal experiment that applies every proposed replacement.
Conditional on the past observed history, $I$, and the selected arm, the reported reward has law $P_{-,t}$.
Thus the observed transcript and selected arm index $A_t$ are independent of $I$, giving $\mathbb P(A_t=I)=1/d$.
Write $D_t$ for the proposed cost in this experiment, equal to $R_t-Y_t$ when $A_t=I$ and zero otherwise.
The conditional expected cost on an optimal-arm pull is $\Delta$, so
\begin{equation}
    \E\sum_{t=1}^T D_t
      =\Delta\sum_{t=1}^T\mathbb P(A_t=I)
      =\frac{T\Delta}{d}=\frac C4.
\end{equation}
Using common randomness, couple the stopped and ideal experiments until the first rejected proposal.
Let $B=\{\sum_{t=1}^T D_t>C\}$; by Markov's inequality, $\mathbb P(B)\leq1/4$.
Since $D_t\geq0$, the two experiments agree on $B^c$.
Their pseudo-regrets lie in $[0,T\Delta]$, while the ideal expected regret is $T\Delta(1-1/d)$.
Therefore,
\begin{align}
    \E\Reg^\pi(T)
    &\geq\E\Reg^\pi_{\mathrm{ideal}}(T)-T\Delta\,\mathbb P(B)\\
    &\geq T\Delta\left(1-\frac1d-\frac14\right)
      \geq\frac{T\Delta}{4}=\frac{dC}{16}.
\end{align}
At least one parameter $\vtheta^{(i)}$, together with its budget-valid adversary, attains this average.
Since $\pi$ was arbitrary, this proves \Cref{thm:lower-bound-corruption} for every prescribed valid dispersion sequence.
\qed

\subsection{Instantiating the Conditional Replacement}\label{app:conditional-replacement}
The following thinning rules implement the replacement for the normalized feedback models in \Cref{sec:intro}.

\paragraph{Bernoulli.}
For means $p_+>p_-$, keep a Bernoulli reward of $1$ with probability $p_-/p_+$ and leave $0$ unchanged.
The proposed report is Bernoulli$(p_-)$, and the expected cost is $p_+-p_-$.
For a batch average $r_t=K_t/n_t$ with $K_t\sim\mathrm{Bin}(n_t,p_+)$, draw $K'_t\mid K_t\sim\mathrm{Bin}(K_t,p_-/p_+)$.
Then $K'_t\sim\mathrm{Bin}(n_t,p_-)$, so the proposed report $Y_t=K'_t/n_t\leq r_t$ has the suboptimal-arm law and $\E[r_t-Y_t]=p_+-p_-=\Delta$.

\paragraph{Poisson.}
We use the following standard thinning property:
\begin{lemma}[Poisson Thinning; Proposition 5.5 of \citet{ross2023introduction}]\label{lem:poi}
    Let $\lambda>0$, $p\in[0,1]$, $N\sim\Poi(\lambda)$, and $N'\mid N\sim\mathrm{Bin}(N,p)$.
    Then $N'\sim\Poi(p\lambda)$ and $N-N'\sim\Poi((1-p)\lambda)$, and these two counts are independent.
\end{lemma}
For rates $\mu_+>\mu_->0$ and exposure $e_t>0$, let $r_t=N_t/e_t$ with $N_t\sim\Poi(e_t\mu_+)$ and draw $N'_t\mid N_t\sim\mathrm{Bin}(N_t,\mu_-/\mu_+)$.
By \Cref{lem:poi}, $N'_t\sim\Poi(e_t\mu_-)$, so $Y_t=N'_t/e_t\leq r_t$ has the suboptimal-arm law and $\E[r_t-Y_t]=\mu_+-\mu_-=\Delta$.
For either rule, the adversary applies the proposed replacement only if it fits the remaining budget, using the stopping rule proved above.

\newpage
\section{\texorpdfstring{HCW-GLB-OMD-Specific $\kappa C$ Lower Bound}{HCW-GLB-OMD-Specific Lower Bound}}
\label{app:curvature-gap}

Our minimax corruption lower bound $\Omega(dC)$ leaves open whether the additional $\kappacolor{\kappa}$-dependence in \Cref{theorem:regret} can be removed.
We show that \texttt{HCW-GLB-OMD} with zero initialization and $\alpha=\sqrt d/C$ can incur $\Omega(\kappacolor{\kappa} C)$ regret even when $C$ is known.
The mechanism is simple: the adversary erases rare positive rewards from a suboptimal arm, while the slowly growing Hessian leaves that arm optimistic.

\begin{proposition}\label{prop:kappa-corruption-obstruction}
    Consider logistic rewards with $\dispersioncolor{g(\tau_t)}=1$, $d\geq1$, $S\geq\log3$, and $C\geq1+S$.
    Run \texttt{HCW-GLB-OMD} with $\vtheta_1=\vzero$, $\alpha=\sqrt d/C$, and the remaining parameters set as in \Cref{theorem:regret}, using $R_s=1$ and $L_\mu=1/4$.
    For $\gX=\{\vzero,\ve_1\}$ and $\vtheta_\star=-S\ve_1$, there exists an adversary satisfying $\sum_t|c_t|\leq C$ almost surely such that, for $T\geq C/\mu(-S)$,
    \begin{equation}
        \sP\left(\Reg(T)\geq\frac{3\kappacolor{\kappa}C}{128}\right)\geq\frac{3}{4}, \qquad \E\Reg(T)\geq\frac9{512}\kappacolor{\kappa} C,
    \end{equation}
    where $\kappacolor{\kappa}=1/ {\curvaturecolor{\dmu}(-S)}$.
\end{proposition}
To see why this obstructs removing $\kappacolor{\kappa}$ from the corruption term in our regret bound, fix $d$ and take $C\asymp\kappacolor{\kappa}$.
Then \Cref{prop:kappa-corruption-obstruction} gives $\Omega(\kappacolor{\kappa^2})$ regret at $T\asymp\kappacolor{\kappa^2}$.
Here, $S\asymp\log\kappacolor{\kappa}$ and the optimal-arm curvature is $\curvaturecolor{\dmu}(0)=1/4$.
A bound containing only the leading stochastic term, $d^2\kappacolor{\kappa}$, and $dC$ would therefore be $\widetilde{\gO}(\kappacolor{\kappa})$, allowing at most polynomial factors in $S$.

\begin{proof}
Set $p=\mu(-S)\leq1/4$ and $n=\lfloor C/(4p)\rfloor\leq T$.
Arm $\vzero$ has mean $1/2$, whereas arm $\ve_1$ has mean $p$.
Thus, $\vzero$ is optimal, and each selection of $\ve_1$ incurs
pseudo-regret $1/2-p$.

During the first $n$ rounds, the adversary replaces each reward of $1$
from $\ve_1$ by $0$ until the next replacement would exceed budget $C$.
All other rewards are unchanged, and no corruption is applied after
round $n$.

First consider an ideal interaction in which every reward of $1$ from
$\ve_1$ is replaced by $0$ during the first $n$ rounds, without imposing
the corruption budget.
All algorithm parameters remain unchanged, and the algorithm chooses
its arms according to its usual optimistic rule.
We prove inductively that it selects $\ve_1$ in each of these rounds,
then show that the total cost of these replacements is at most $C$
with probability at least $3/4$.

Fix $t\leq n$ and suppose that the algorithm has selected $\ve_1$
in every round before $t$.
Write $z_s=\langle\ve_1,\vtheta_s\rangle$ for the estimated logit and
$h_s=\ve_1^\top\mH_s\ve_1$ for the accumulated curvature along this arm.
Since the OMD estimator starts at $\vzero$, the Hessian starts at
$\lambda\mI_d$, and these updates act only along $\ve_1$, for $s\leq t$,
\[
    \vtheta_s=z_s\ve_1,
    \qquad
    \mH_s=\lambda\mI_d+(h_s-\lambda)\ve_1\ve_1^\top.
\]
For $s<t$, the observed reward is $0$, and the updates satisfy
\begin{equation}
    \begin{aligned}
    z_{s+1}
    &=\max\left\{-S,\,
        z_s-\frac{\eta\weightcolor{w_s}\mu(z_s)}
        {h_s+\eta\weightcolor{w_s}\curvaturecolor{\dmu}(z_s)}
        \right\},\\
    h_{s+1}
    &=h_s+\weightcolor{w_s}\curvaturecolor{\dmu}(z_{s+1}),
    \qquad
    \weightcolor{w_s}=\min\{1,\alpha\sqrt{h_s}\}.
    \end{aligned}
\end{equation}
Here $z_1=0$, $h_1=\lambda$, $\eta=1+S$, and $\lambda\geq14d\eta$.
The logit update subtracts a positive quantity and projects onto
$[-S,S]$, so $z_s$ decreases from $0$ toward $-S$ and remains constant
once it reaches that boundary.
For the logistic link,
$\curvaturecolor{\dmu}(z)=\mu(z)(1-\mu(z))$ is increasing on $[-S,0]$,
so this decrease in the logit also decreases the local curvature.
We show that the resulting Hessian growth is too small to exclude
parameters with a positive first coordinate from the confidence set.

Since $C\geq\eta$,
\[
    \frac{\eta\weightcolor{w_s}\curvaturecolor{\dmu}(z_s)}{h_s}
    \leq\frac{\eta\alpha}{4\sqrt\lambda}<\frac18.
\]
Before the first update reaching $-S$, the recursion and
$\curvaturecolor{\dmu}(z_{s+1})\leq\mu(z_s)$ give
\begin{equation}
    \log\frac{h_{s+1}}{h_s}
    \leq\frac9{8\eta}(z_s-z_{s+1}).
\end{equation}
Telescoping over these updates and using $z_1-z_s\leq S$ yields
\[
    h_s\leq\lambda e^{9S/(8\eta)}\leq\lambda e^{9/8},
\]
up to and including the state immediately before the first update
reaching $-S$.

If such an update occurs before round $t$, let $j<t$ denote its index.
Then $h_j\leq\lambda e^{9/8}$, and for every $j\leq s<t$,
$\curvaturecolor{\dmu}(z_{s+1})=p(1-p)$.
Consequently, we have
\begin{equation}
    h_{s+1}-h_s
    \leq\alpha p(1-p)\sqrt{h_s},
\end{equation}
and hence,
\begin{equation}
    \sqrt{h_{s+1}}-\sqrt{h_s}
    =\frac{h_{s+1}-h_s}{\sqrt{h_{s+1}}+\sqrt{h_s}}
    \leq
    \frac{\alpha p(1-p)\sqrt{h_s}}
         {\sqrt{h_{s+1}}+\sqrt{h_s}}
    \leq\frac{\alpha p(1-p)}2,
\end{equation}
where the last inequality uses $h_{s+1}\geq h_s$.
Telescoping these increments from $h_j$ gives
\begin{equation}
    \sqrt{h_t}
    \leq e^{9/16}\sqrt\lambda+\frac{n\alpha p(1-p)}2
    \leq e^{9/16}\sqrt\lambda+\frac{\sqrt d}{8}
    <2\sqrt\lambda.
\end{equation}
If no update reaches $-S$ before round $t$, the preceding exponential
bound alone gives the same conclusion.

By the confidence-set definition in
\Cref{theorem:confidence-sequence}, its radius satisfies
$\beta_t(\delta)+2\eta\alpha C\geq2S\sqrt\lambda$,
since $\beta_t(\delta)^2\geq4\lambda S^2$.
Moreover,
\[
    \|\vtheta_t-\vzero\|_{\mH_t}
    =|z_t|\sqrt{h_t}
    \leq S\sqrt{h_t}<2S\sqrt\lambda.
\]
Thus, $\vzero$ lies strictly inside the confidence ellipsoid, and
$\varepsilon\ve_1$ belongs to the confidence set for sufficiently
small $\varepsilon\in(0,S]$.
Arm $\ve_1$ therefore has an optimistic mean strictly larger than $1/2$,
whereas arm $\vzero$ has mean $1/2$ under every parameter.
The algorithm consequently selects $\ve_1$ at round $t$.
This argument also covers $t=1$, when the preceding history is empty,
and completes the induction: the algorithm selects $\ve_1$ in each
of the first $n$ rounds of the ideal interaction.

Let $B$ denote the total corruption cost in the ideal interaction.
Each replacement costs one unit, and a true reward of $1$ occurs
with probability $p$, so $\E B=np\leq C/4$.
By Markov's inequality, $\sP(B\leq C)\geq3/4$.
Couple the actual and ideal interactions using the same reward
randomness while their histories agree.
On the event $\{B\leq C\}$, the budget never prevents a replacement,
so the two interactions agree through round $n$.

Finally, $C/(4p)\geq2$ implies $n\geq C/(8p)$.
On $\{B\leq C\}$, monotonicity of pseudo-regret gives
\begin{equation}
    \Reg(T)
    \geq\Reg(n)
    =n(1/2-p)
    \geq\frac{C}{32p}
    =\frac{(1-p)\kappacolor{\kappa}C}{32}
    \geq\frac{3\kappacolor{\kappa}C}{128},
\end{equation}
where we used $p\leq1/4$ and
$\kappacolor{\kappa}=1/[p(1-p)]$.
Since $\sP(B\leq C)\geq3/4$, this proves
\[
    \sP\left(\Reg(T)\geq
        \frac{3\kappacolor{\kappa}C}{128}\right)\geq\frac34.
\]
The nonnegativity of pseudo-regret then yields
\[
    \E\Reg(T)
    \geq\frac34\cdot\frac{3\kappacolor{\kappa}C}{128}
    =\frac9{512}\kappacolor{\kappa}C.
\]
\end{proof}

\newpage
\section{Technical Lemmas}
Here, we state technical lemmas used throughout the paper.

\begin{lemma}[Self-Concordance Control, Lemma 8 of \citet{zhang2025onepass}]
\label{lemma:self-concordance}
    Let $\mu :\R \to \R$ be a strictly increasing function satisfying $|\ddmu(z)| \leq R_s \cdot \curvaturecolor{\dmu}(z)$ for all $z \in \mathcal{Z}$, where $R_s \geq 0$ and $\mathcal{Z} \subset \R$ is a bounded interval. Then, for any $z_1, z_2 \in \mathcal{Z}$ and $z \in \{ z_1, z_2\}$, we have
    \begin{equation}
        \int_0^1 \curvaturecolor{\dmu}(z_1 + v (z_2 - z_1))dv \geq \frac{\curvaturecolor{\dmu}(z)}{1 + R_s \cdot |z_1 - z_2|}
    \end{equation}
    and the weighted integral
    \begin{equation}
        \frac{{\curvaturecolor{\dmu}(z_1)}}{2 + R_s \cdot | z_1 - z_2 |} \leq \int_0^1 (1-v) \curvaturecolor{\dmu}(z_1 + v(z_2 - z_1))d{v} \leq \exp \left( R_s^2 |z_1 - z_2|^2 \right) \cdot\ \curvaturecolor{\dmu}(z).
    \end{equation}
\end{lemma}

\begin{lemma}[Elliptical Potential Lemma, Lemma 11 of \citet{zhang2025onepass}]\label{lemma:epl}
    Let $\{\vx_s\}_{s\geq1}$ be a sequence in $\R^d$ with $\lVert\vx_s\rVert_2\leq X$, and let $\lambda \geq 1$. Define $\mV_s\coloneqq\lambda\mI_d+\sum_{i=1}^{s-1}\vx_i\vx_i^\top$. Then,
    \begin{align}
        \sum_{s=1}^t\lVert\vx_s\rVert_{\mV_s^{-1}}^2 \leq 2 d\left(1+X^2\right)\log\left(1+\frac{tX^2}{d\lambda}\right).
    \end{align}
\end{lemma}

\begin{lemma}[Proposition 7 of \citet{abeille2021logistic}]
\label{lemma:solve-ineq}
    For $X, a, b \geq 0$, the inequality $X^2 \leq a X + b$ implies
    \begin{equation}
        X \leq a + \sqrt{b}.
    \end{equation}
\end{lemma}

\begin{lemma}\label{lemma:solve-log}
    Let ${a, b > 0}$ with $ab \geq 1$.
    If $n \in \mathbb{N}$ satisfies the inequality $n \leq a \log(1 + bn)$, then
    $$n \leq 2a \log(1 + ab).$$
\end{lemma}
\begin{proof}
The case when $n=0$ is immediate, so we assume $n>0$. 
For $x>0$, define $f(x) = \frac{x}{\log(1+bx)}$, then 
\begin{equation}
    f'(x) = \frac{\log(1 + bx) - \frac{bx}{1+bx}}{\log^2(1 + bx)}.
\end{equation}
Let $h(y) = \log (1 +y) - \frac{y}{1+y}$, then $h(0)=0$ and $h'(y) = \frac{1}{1+y} - \frac{1}{(1+y)^2} = \frac{y}{(1+y)^2} >0$ for $y>0$. Therefore, $f(x)$ is increasing in $x>0$.

Now, we show that $f(2a\log(1+ab))>a$.
Since $f(n)\leq a$, strict monotonicity of $f$ implies $n<2a\log(1+ab)$.
Substituting $x = 2a \log (1+ab)$ gives
\begin{align}
    f(2a \log (1+ab)) & = \frac{2a \log (1+ab)}{\log ( 1 + 2ab \log (1+ab))}\\
    & = a \frac{ 2\log (1 + ab)}{\log(1 + 2ab \log(1+ab))}.
\end{align}
If $2ab + a^2b^2 > 2ab \log(1+ab)$ holds, we complete the proof.
Since $ab \geq 1 > 0$, the inequality is equivalent to $1 + \frac{ab}{2} > \log ( 1+ ab)$.
Due to concavity of the logarithm, this holds because
\[
    \log(1+ab)
    \leq\log2+\frac{ab-1}{2}
    <1+\frac{ab}{2}.
\]
This proves $f(2a\log(1+ab))>a$ and completes the proof.

\end{proof}

\newpage
\section{Comparison with Prior Art}
\label{app:comparison}

\subsection{Detailed Comparison}
In this section, we provide a detailed discussion comparing our results against the specific prior works listed in Table~\ref{tab:results}.
{For consistency at small corruption budgets, corruption-dependent upper bounds are stated using the valid budget upper bound $C\vee1$.}
We note that our \texttt{HCW-GLB-OMD} inherits all the computational benefits of using an online estimator~\citep{zhang2025onepass}: it is one-pass, relying on recursive updates and a single projection onto the convex set $\Theta$.
Our algorithm only requires $\gO(1)$ space and time per iteration $t$, making it the most computationally efficient algorithm for \textbf{\textit{heteroskedastic GLBs with adversarial corruptions}}.

\paragraph{No Corruptions ($C = 0$).}
In the absence of corruption, we immediately recover the existing state-of-the-art regret bounds $\widetilde\gO(d\sqrt{\sum_t\sigma_t^2})$ for heteroskedastic linear bandits~\citep[Theorem 4.1]{zhou2022variance} and $\widetilde\gO(d\sqrt{\sum_t\curvaturecolor{\dmu_{t,\star}}})$ for self-concordant GLBs~\citep[Theorem 4.1]{lee2024glm}.

\paragraph{With Corruption, Linear ($\mu(z) = z$).}
In this setting where $R_s=0$, we obtain $\widetilde\gO(d\sqrt{\sum_t\sigma_t^2}+d(C\vee1))$.
\citet{zhao2021corrupted} established an early variance-aware corruption-robust guarantee, but its variance-dependent term $\widetilde\gO(C^2d\sqrt{\sum_t\sigma_t^2})$ has a suboptimal multiplicative dependence on $C$~\citep{gupta2019corrupted}.
\citet[Corollary 5]{yu2025corruption} subsequently achieved an additive dependence on $C$ for heavy-tailed rewards under a finite-variance assumption.
\citet{tani2026robust} further relaxed this assumption to bounded $(1+\epsilon)$-moments for $\epsilon\in(0,1]$, allowing for infinite variance when $\epsilon<1$ and recovering the finite-variance rate when $\epsilon=1$.
These guarantees use upper bounds on the relevant conditional moments supplied to the learner.
Thus, in the finite-variance case, our regret bound is of the same order as theirs when the supplied variance upper bounds equal the true variances.

\paragraph{With Corruption, Generalized Linear.}

Our parametric setting (Eqn.~\eqref{eqn:GLM}) generalizes the prior settings in the following sense.
Projecting the prior heteroskedastic generalized linear bandit literature (where the learner accesses the conditional variance of the played arm) onto our setting is equivalent to assuming that ${\dispersioncolor{g(\tau_t)}} \curvaturecolor{\dmu}(\langle \vx_t, \vtheta_\star \rangle)$ is known to the learner.
Because we allow the learner to access \emph{only} ${\dispersioncolor{g(\tau_t)}}$, our scenario is strictly weaker regarding the feedback model.
Furthermore, the literature on heteroskedastic linear bandits does not account for nonlinear $\mu$ and, with the exception of \citet{yu2025corruption}, cannot account for (adaptive) adversarial corruption either.
Despite these challenges, we show that our regret bound either subsumes or improves upon all known prior results, under self-concordance (\Cref{assumption:link-function}).

We compare our results against \citet{yu2025corruption} in our parametric GLBs setting, i.e., assuming that the reward follows the following distribution:
\begin{equation}
    d p (r \mid \vx, \tau; \vtheta_\star) \propto \exp \left( \frac{r \langle \vx, \vtheta_\star \rangle - m ( \langle \vx, \vtheta_\star \rangle) }{\dispersioncolor{g(\tau)}} \right) d \nu,
\end{equation}
where the link function satisfies \Cref{assumption:link-function}.
In the setting of \citet{yu2025corruption}, at each timestep $t$, the learner pulls an arm $\vx_t$ and observes a corrupted reward $\tilde{r}_t = \mu(\langle\vx_t, \vtheta_\star \rangle) + \epsilon_t + c_t$, where $\epsilon_t$ is a martingale difference noise with $\E[\epsilon_t^2 \mid \vx_t] \leq \nu_t^2$ and $c_t$ is the adversarial corruption.
In \citet{yu2025corruption}, a valid bound $\nu_t$ on the conditional standard deviation is available at the end of round $t$; it need not be the exact value.
Their regret bound~\citep[Corollary 5]{yu2025corruption} is of order $\widetilde{\gO}\left( d {\kappacolor{\kappa}} \sqrt{\sum_t \nu_t^2} + d {\kappacolor{\kappa}} (C\vee1) \right)$.
Here, the comparison is more nuanced due to the distinction between exogenous and endogenous heteroskedasticity, as elaborated in \Cref{sec:intro}.
Indeed, as $\E[\epsilon_t^2 \mid \vx_t] = {\dispersioncolor{g(\tau_t)}} { \curvaturecolor{\dmu}(\langle \vx_t, \vtheta_\star \rangle)}$, we divide their regret bound into two scenarios depending on how much ``information'' is available in $\nu_t^2$.

First, when the endogenous heteroskedasticity ${ \curvaturecolor{\dmu}(\langle \vx_t, \vtheta_\star \rangle)}$ is completely unobserved, {one available choice is $\nu_t^2=L_\mu \dispersioncolor{g(\tau_t)}$}, where we na\"{i}vely upper bound $\curvaturecolor{\dmu}$ by $L_\mu$.
With this, their regret bound reduces to $\widetilde{\gO}\left( d {\kappacolor{\kappa}} \sqrt{\sum_t {\dispersioncolor{g(\tau_t)}}} + d {\kappacolor{\kappa}} {(C\vee1)} \right)$.
For heteroskedastic linear bandits (with linear $\mu$), this bound is tight.
However, for nonlinear $\mu$ (e.g., logistic or Poisson bandits), this bound is quite loose; the leading term's dependence on the true, instance-dependent curvature ${\curvaturecolor{\dmu_{t,\star}}} = { \curvaturecolor{\dmu}(\langle {\vx_{t,\star}}, \vtheta_\star \rangle)}$ is lost, and the leading term scales directly with ${\kappacolor{\kappa}}$.

Second, suppose that $\nu_t^2 = {\dispersioncolor{g(\tau_t)}} \curvaturecolor{\dmu}(\langle \vx_t, \vtheta_\star \rangle)$ is observed by the learner; we remark that we do \emph{not} require this.
Then, their regret bound becomes $\widetilde{\gO}\left( d {\kappacolor{\kappa}} \sqrt{ \sum_t {\dispersioncolor{g(\tau_t)}} \curvaturecolor{\dmu}(\langle \vx_t, \vtheta_\star \rangle)} + d {\kappacolor{\kappa}} {(C\vee1)} \right)$.
{This is a valid bound in terms of the variance along the selected arms. To compare it with our optimal-arm quantity, we further bound its algorithm-dependent sum.}
Utilizing the self-concordance arguments as in \Cref{app:proof-regret}, we can further upper bound the algorithm-dependent term as $\sum_{t=1}^{T} {\dispersioncolor{g(\tau_t)}} \curvaturecolor{\dmu}(\langle \vx_t, \vtheta_\star \rangle) \leq \sum_{t=1}^{T} {\dispersioncolor{g(\tau_t)}} {\curvaturecolor{\dmu_{t,\star}}} + {\dispersioncolor{g_{\max}}} R_s \Reg^{\mathrm{Yu}}(T).$
This then yields the following recursive inequality:
\begin{equation}
    \Reg^{\mathrm{Yu}}(T) \lesssim_{\log} d {\kappacolor{\kappa}} \sqrt{ \sum_{t=1}^{T} {\dispersioncolor{g(\tau_t)}} {\curvaturecolor{\dmu_{t,\star}}} + {\dispersioncolor{g_{\max}}} \Reg^{\mathrm{Yu}}(T)} + d {\kappacolor{\kappa}} {(C\vee1)}.
\end{equation}
Solving the inequality for $\Reg^{\mathrm{Yu}}(T)$ via \Cref{lemma:solve-ineq} yields
\begin{equation}
    \Reg^{\mathrm{Yu}} \lesssim_{\log} d {\kappacolor{\kappa}} \sqrt{ \sum_{t=1}^T {\dispersioncolor{g(\tau_t)}} {\curvaturecolor{\dmu_{t, \star}}}} + d {\kappacolor{\kappa}} {(C\vee1)} + d^2 {\kappacolor{\kappa^2}} \dispersioncolor{g_{\max}}.
\end{equation}
Note that the leading term directly scales with the global curvature ${\kappacolor{\kappa}}$, and we again remark that this is obtainable only when the learner observes $\nu_t^2 = {\dispersioncolor{g(\tau_t)}} \curvaturecolor{\dmu}(\langle \vx_t, \vtheta_\star \rangle)$.
The dependency on ${\kappacolor{\kappa}}$ in the leading term is due to the fact that \texttt{GAdaOFUL} constructs a curvature-agnostic confidence sequence whose radius scales with ${\kappacolor{\kappa}}$.

\begin{remark}
    We acknowledge that the results of \citet{yu2025corruption} apply to generic link functions without requiring self-concordance, assuming only that the noise has finite variance (allowing for heavy tails), and also that their work only requires an upper bound on the conditional variance of the noise.
    Thus, their results are more general in scope.
    However, we demonstrate that under the mild assumption of self-concordance (\Cref{assumption:link-function}), it is possible to achieve significantly tighter, instance-wise minimax-optimal bounds.
    We also remark that a similar regret bound can be derived for \texttt{MOR-UCB}~\citep{zhao2023optimal}, although their algorithm requires that the reward (noise) is bounded almost surely.
\end{remark}

\paragraph{With Corruption, General Function Approximation.}
Finally, we compare our results with \texttt{CR-Eluder-UCB} of \citet{ye2023corruption}, which is designed for general function approximation~\citep{foster-rakhlin} with adversarial corruptions.
While this setting is more general, the regret bound scales with the eluder dimension~\citep{russo-vanroy,osband-vanroy}, which represents the worst-case complexity of the function class.
When instantiated to GLMs, the eluder dimension scales with ${\kappacolor{\kappa}}$.
Consequently, the instantiated regret bound of $\widetilde{\gO}(d {\kappacolor{\kappa}}\sqrt{T} + d {\kappacolor{\kappa^2}} {(C\vee1)})$ (see \Cref{app:ye} for the derivation) fails to exhibit curvature/instance-wise optimality, or variance-adaptivity when $\mu(z) = z$ and ${\tau_t}$ varies, in the leading term.

\subsection{\texorpdfstring{Regret Bound of \texttt{CR-Eluder-UCB} of \citet{ye2023corruption}}{Regret Bound of CR-Eluder-UCB of Ye et al. (2023)}}
\label{app:ye}

Here, we instantiate the regret bound of \texttt{CR-Eluder-UCB} \citep[Theorem 4.1]{ye2023corruption} within the context of \textbf{\textit{sub-Gaussian} GLB} with adversarial corruption.
Consider the function class $\gF = \{ \vx \mapsto \mu(\vx^\top \vtheta) : \lVert \vtheta \rVert_2 \leq S \}$, then this corresponds to our setting with a fixed dispersion parameter $\tau_t = \tau$.

While their upper bound may be optimal for some function classes, it becomes loose when applied to a parametric GLM class.
Specifically, the regret bound is given as:
\begin{equation}
    \widetilde{\mathcal{O}} \left( \sqrt{T \dim_E \big( \gF, \lambda/T \big) \log N(\gamma, \gF, \lVert \cdot \rVert_{\infty})} + {(C\vee1)} \dim_E \big( \gF, \lambda / T\big) \right),
\end{equation}
where $\lVert \cdot \rVert_\infty$ denotes the infinity-norm, $\lambda = \log N( \gamma, \mathcal{F}, \bignorm{\cdot}_\infty)$ and {$\gamma = (T(C\vee1))^{-1}$}.
For the GLM class, the eluder dimension and covering number scale as $\dim_E (\mathcal{F}, \lambda/T) = \mathcal{O}(d {\kappacolor{\kappa^2}} \log T)$ (\Cref{lemma:eluder}) and $\log N(\gamma, \mathcal{F}, \lVert \cdot \rVert_\infty) = \mathcal{O}(d \log(1 + T(1\vee C)))$ (\Cref{lemma:covering}) respectively.

By substituting these values, the resulting regret order becomes $\widetilde{\mathcal{O}}(d {\kappacolor{\kappa}} \sqrt{T} + d {\kappacolor{\kappa^2}} {(C\vee1)})$.
This bound not only fails to exhibit instance-wise optimality as it scales linearly to the global curvature ${\kappacolor{\kappa}}$, but also suffers from a quadratic dependence on ${\kappacolor{\kappa^2}}$ of the corruption-dependent term.
In contrast, our approach leverages the specific parametric structure of GLMs to achieve tighter, curvature-dependent regret that is more robust to the underlying geometry of the link function.

\begin{remark}
    Recently, \citet{bakhtiari2025eluder} proposed a localization approach to the eluder dimension, accompanied by a simple algorithm ($\ell$-UCB) that enables first-order regret bounds.
    While extending their approach to the corrupted setting of \citet{ye2023corruption} might theoretically yield desired first-order bounds, such an extension is non-trivial. Furthermore, a tractable version of $\ell$-UCB is currently only known for self-concordant $\mu$; see \citet[Appendix A]{bakhtiari2025eluder}.
\end{remark}

\begin{definition}[$\epsilon$-dependency, Definition 3 of \citet{russo-vanroy}]
    An action $a \in \mathcal{A}$ is $\epsilon$-dependent on actions $\{a_1, \cdots, a_n\} \subseteq \mathcal{A}$ with respect to $\mathcal{F}$ if any pair of functions $f, \tilde{f}\in \mathcal{F}$ satisfying $\sqrt{\sum_{i=1}^{n} (f(a_i) - \tilde{f}(a_i))^2} \leq \epsilon$ also satisfies $|f(a) - \tilde{f}(a)| \leq \epsilon$.
    Furthermore, $a \in \gA$ is $\epsilon$-independent of actions $\{a_1, \cdots, a_n\} \subseteq \mathcal{A}$ with respect to $\mathcal{F}$ if $a \in \gA$ is not $\epsilon$-dependent on $\{a_1, \cdots, a_n \}$.
\end{definition}

\begin{definition}[Eluder dimension, Definition 4 of \citet{russo-vanroy}]
    The $\epsilon$-eluder dimension $\dim_E(\mathcal{F}, \epsilon)$ is the length $d_e \in \sN$ of the longest sequence of elements in $\mathcal{A}$ such that for some $\epsilon^\prime \geq \epsilon$, every element is $\epsilon^\prime$-independent of its predecessors.
\end{definition}

\begin{lemma}[Eluder dimension of GLM class]\label{lemma:eluder}
Consider the GLM class $\mathcal{F} = \{ \vx \mapsto \mu(\vx^\top \vtheta) : \lVert \vtheta \rVert_2 \leq S \}$ for $\vx \in \gB^d(1)$, and let $0 < \epsilon \leq 2\sqrt{2} S L_\mu$.
Then the $\epsilon$-eluder dimension $\dim_E(\mathcal{F}, \epsilon)$ is upper bounded as:
\begin{equation}
    \dim_E(\mathcal{F}, \epsilon) \leq 16 d {\kappacolor{\kappa^2}} L_\mu^2 \log \left (1 + \frac{8 S^2 L_\mu^2}{\epsilon^2} \right).
\end{equation}
\end{lemma}

\begin{proof}
Fix $\epsilon'\geq\epsilon$ and an $\epsilon'$-independent sequence $\vx_1,\ldots,\vx_n$. If $\epsilon'\geq2SL_\mu$, no such nonempty sequence exists. Otherwise, for each $m\in[n]$, there are $\vtheta_{m,1},\vtheta_{m,2}\in\Theta$ such that 
\begin{align}
    \sum_{i<m}\big(\mu(\langle\vx_i,\vtheta_{m,1}\rangle)-\mu(\langle\vx_i,\vtheta_{m,2}\rangle)\big)^2\leq(\epsilon')^2,\qquad 
    \big|\mu(\langle\vx_m,\vtheta_{m,1}\rangle)-\mu(\langle\vx_m,\vtheta_{m,2}\rangle)\big|>\epsilon'.
\end{align}
Set $\lambda=\kappacolor{\kappa^2}(\epsilon')^2/(4S^2)$ and $\mV_m=\lambda\mI_d+\sum_{i<m}\vx_i\vx_i^\top$. By $1/\kappacolor{\kappa}\leq\curvaturecolor{\dmu}\leq L_\mu$ and $\lVert\ \vtheta_{m,1}-\vtheta_{m,2} \rVert_2\leq2S$,
\begin{equation}
    \lVert \vtheta_{m,1}-\vtheta_{m,2} \rVert_{\mV_m}^2\leq2\kappacolor{\kappa^2}(\epsilon')^2,
    \qquad |\langle\vx_m, \vtheta_{m,1}-\vtheta_{m,2} \rangle|>\epsilon'/L_\mu.
\end{equation}
Cauchy--Schwarz gives $\lVert\vx_m\rVert_{\mV_m^{-1}}^2>1/(2\kappacolor{\kappa^2}L_\mu^2)$. Since $\kappacolor{\kappa} L_\mu\geq1$, the same lower bound holds after clipping the squared norm at $1$. Therefore, \Cref{lemma:epl} gives
\begin{equation}
    n\leq4d\kappacolor{\kappa^2}L_\mu^2\log\left(1+\frac{4S^2n}{d\kappacolor{\kappa^2}(\epsilon')^2}\right).
\end{equation}
Apply \Cref{lemma:solve-log} with $a=4d\kappacolor{\kappa^2}L_\mu^2$ and $b=4S^2/(d\kappacolor{\kappa^2}(\epsilon')^2)$; here $ab>4$. We obtain
\begin{equation}
    n\leq8d\kappacolor{\kappa^2}L_\mu^2\log\left(1+\frac{16S^2L_\mu^2}{(\epsilon')^2}\right)
    \leq16d\kappacolor{\kappa^2}L_\mu^2\log\left(1+\frac{8S^2L_\mu^2}{\epsilon^2}\right),
\end{equation}
where the last inequality uses $1+2x\leq(1+x)^2$ and $\epsilon'\geq\epsilon$.
\end{proof}

\begin{lemma}[Covering number of GLM class]\label{lemma:covering}
    Let $N(\epsilon, \mathcal{F}, \lVert \cdot \rVert_\infty)$ be the $\epsilon$-covering number of the class $\mathcal{F}$ with respect to $\lVert \cdot \rVert_\infty$. Then the $\epsilon$-covering number of the class $\mathcal{F} = \{ \vx \mapsto \mu(\vx^\top \vtheta) : \lVert \vtheta \rVert_2 \leq S \}$ satisfies
    \begin{equation}
        N(\epsilon, \mathcal{F}, \lVert \cdot \rVert_\infty) \leq \left( 1+ \frac{2 S L_\mu}{\epsilon} \right)^d.
    \end{equation}
\end{lemma}
\begin{proof}
For any $\vtheta_1, \vtheta_2 \in \Theta$, by Lipschitz continuity of $\mu$ and boundedness of $\Theta$ and $\lVert \vx \rVert_2 \leq 1$, we have
\begin{align}
    \left| \mu( \langle \vx, \vtheta_1 \rangle) - \mu( \langle \vx, \vtheta_2 \rangle) \right| \leq L_\mu \left| \langle \vx, \vtheta_1 -  \vtheta_2 \rangle \right| \leq L_\mu \lVert \vtheta_1 - \vtheta_2 \rVert_2.
\end{align}
This implies that if $\lVert \vtheta_1 - \vtheta_2 \rVert_2 \leq \epsilon / L_\mu$, $\left| \mu(\langle \vx, \vtheta_1 \rangle) - \mu(\langle \vx, \vtheta_2 \rangle) \right| \leq \epsilon$ for any $\vx$.
Therefore, the $\epsilon$-covering number $N(\epsilon, \mathcal{F}, \lVert \cdot \rVert_\infty)$ is upper bounded by the $\epsilon/L_\mu$-covering number of Euclidean ball $\Theta$ with respect to $\lVert \cdot \rVert_2$, i.e.
\begin{align}
    N(\epsilon, \mathcal{F}, \lVert \cdot \rVert_\infty) \leq N(\epsilon/L_\mu, \Theta, \lVert \cdot \rVert_2 ) \leq \left( 1 + \frac{2S L_\mu}{\epsilon}\right)^d.
\end{align}
\end{proof}

\newpage
\section{\texorpdfstring{Dealing with Unknown Corruption Budget $C$}{Dealing with Unknown Corruption Budget C}}
\label{app:unknown-C}

\paragraph{A Known Upper Bound Suffices.}
Our main analysis uses the corruption budget in both the choice of $\alpha$ and the confidence radius; the individual corruptions remain unobserved.
A known deterministic upper bound $\overline C\geq C$ suffices: set $\alpha=\Theta(\sqrt d/(\overline C\vee1))$ and replace $C$ by $\overline C$ in the confidence set, with the corresponding choices of $\eta$ and $\lambda$ in \Cref{theorem:confidence-sequence}.
Since $\sum_{s<t}|c_s|\leq\overline C$, the same proof applies, and \Cref{theorem:regret} holds with $\overline C$ in place of $C$.
An overestimate is therefore valid, although it can increase the regret bound.

\paragraph{Unknown $C$ Under the Present Adversary.}
The known-budget rate cannot, in general, be retained uniformly over unknown $C$ under our action-observing adversary.
\citet[Theorem~4.12]{he2022corruption} establish a tradeoff between clean-instance regret and tolerated corruption in a noiseless scenario: in the sublinear-regret regime, an algorithm not given $C$, whose expected regret on every uncorrupted linear-bandit instance is at most $R_T$, can be forced to incur $\Omega(T)$ expected regret using corruption of order $R_T/d$.
As done similarly in our proof of \Cref{thm:lower-bound-corruption}, we can easily extend their argument to Gaussian linear bandits.\footnote{Take $d\geq2$, arms $\{\ve_1,\ldots,\ve_d\}$, and parameters $\vtheta^{(0)}=\Delta\ve_1$ and $\vtheta^{(j)}=\Delta\ve_1+2\Delta\ve_j$, with a fixed $0<\Delta\leq S/\sqrt5$.
On the clean instance $\vtheta^{(0)}$, some $j\geq2$ has expected pull count at most $R_T/[\Delta(d-1)]$; on $\vtheta^{(j)}$, subtract $2\Delta$ on the first $m=\lceil2R_T/[\Delta(d-1)]\rceil$ pulls of $\ve_j$, costing at most $2\Delta m=\mathcal O(R_T/d+1)$.
Coupling the learner's randomness and unit-variance Gaussian noises makes the observed histories coincide whenever the clean-instance pull count of $\ve_j$ is at most $m$, an event of probability at least $1/2$, giving $\mathbb E[\Reg(T)]\geq\Delta(T-m)/2=\Omega(T)$ when $m\leq T/2$.}
For $R_T=\widetilde{\mathcal O}(d\sqrt T)$, such an attack uses budget $\widetilde{\mathcal O}(\sqrt T)$, at which our Gaussian known-budget bound would still be sublinear for fixed $d$.
A corresponding uniform high-probability guarantee would also contradict this tradeoff: taking $\delta=T^{-2}$ yields an expected-regret bound because one-step pseudo-regret is bounded.

\paragraph{Model Selection Under Weaker Adversaries.}
Under weaker corruption models, model-selection reductions such as \texttt{COBE}~\citep[Theorem~3]{wei2022cobbe} can remove prior knowledge of the corruption budget under suitable feedback and base-learner assumptions.
The \texttt{COBE} reduction requires corruption to be fixed before the meta-algorithm samples its current base learner and, for arm-dependent corruption, uses the budget $C'=\sum_t\sup_{\vx\in\gX_t}|c_t(\vx)|$ rather than only the selected-arm corruption~\citep[Remark~4.11]{he2022corruption}.
It therefore does not directly establish our guarantee under the present adversary or preserve its local dispersion- and curvature-dependent leading term; linear dependence on $C$ alone is insufficient to invoke the reduction.
We leave further exploration to future work.

\newpage
\section{Experiments}
\label{app:experiments}

We conduct synthetic experiments on four bandit instances: Gaussian linear bandits, heteroskedastic Gaussian linear bandits, Bernoulli logistic bandits, and binomial logistic bandits,
each with and without adversarial corruption.

\paragraph{Setup:}
We use dimension $d=5$, $K=20$ fixed unit-vector arms, and horizon $T=2{,}000$.
Each configuration is evaluated over 30 independent seeds (7000--7029).
For each seed, we sample a unit vector $\vv$, set $\vtheta_\star=S\vv$, and construct a randomly ordered arm set containing $\vv$ and 19 independently sampled unit vectors.
We use $S=1$ for the linear instances and $S=4$ for the logistic instances.
Arm geometry and reward randomness are paired across algorithms and corruption budgets.
Figures report mean cumulative pseudo-regret, with shaded regions indicating one standard error.

\paragraph{Bandit instances:}
\begin{itemize}
    \item A. Gaussian linear bandits: Clean rewards are generated as $r_t\sim\gN(\langle\vx_t,\vtheta_\star\rangle,1)$, with $g(\tau_t)=1$.
    \item B. Heteroskedastic Gaussian linear bandits: Clean rewards are generated as $r_t\sim\gN(\langle\vx_t,\vtheta_\star\rangle, g(\tau_t))$, where $g(\tau_t)$ alternates between $1$ and $0.05$ every 200 rounds.
    \item C. Bernoulli logistic bandits: Clean rewards are generated as $r_t\sim\mathrm{Bernoulli} ( \mu (\langle\vx_t,\vtheta_\star\rangle))$, with $g(\tau_t)=1$.
    \item D. binomial logistic bandits: 
    Clean rewards are generated as $r_t = Y_t / n_t$, where $Y_t\sim\mathrm{Binomial} (n_t,  \mu (\langle\vx_t,\vtheta_\star\rangle))$ with the number of trials $n_t$ that alternates between 1 and 20 every 200 rounds, giving $g(\tau_t)=1/n_t$.
\end{itemize}

\paragraph{Adversarial corruption.}
We consider budgets $C\in\{0,100,300\}$.
Starting from the first round, an adversary observes the selected arm and clean reward, then replaces $r_t$ by $-r_t$ in the linear instances, or by $1-r_t$ in the logistic instances.
The final perturbation is truncated to satisfy $\sum_{t=1}^T|c_t|\leq C$, and subsequent observations are uncorrupted.
For the binomial instance, the budget is measured in normalized reward units.

\paragraph{Baselines:} We compare ours with \texttt{GLB-OMD}~\citep[non-robust]{zhang2025onepass}, \texttt{CW-OFUL}~\citep[robust, linear only]{he2022corruption}, \texttt{OFUL} (non-robust, linear-only), and \texttt{GAdaOFUL}~\citep[robust and variance-adaptive]{yu2025corruption}.
We evaluate with $\delta=0.01$, scaled confidence radii and regularization $\lambda=d$ rather than the theoretically exact hyperparameters, which may be overly conservative.
We made sure that the scaling is consistent across the considered baselines for fair comparison.
Implementation details are as follows:
\begin{itemize}
    \item \texttt{HCW-GLB-OMD} (ours): We set $\alpha=16\sqrt d/(C\vee1)$.
    \item \texttt{GLB-OMD}~\citep{zhang2025onepass}: We incorporate the observed dispersion into its likelihood updates in the heteroskedastic instances.
    \item \texttt{GAdaOFUL}~\citep{yu2025corruption}: We adopt the practical weighting parameters from the authors' code:
    $m_0=1.5$ and $\tau_0=\alpha_{\mathrm{GAdaOFUL}}=1$.
    Its variance input is the valid upper bound $L_\mu g(\tau_t)$, equal to $g(\tau_t)$ for the linear instances and $g(\tau_t)/4$ for the logistic instances.
    \item \texttt{CR-Eluder-UCB}~\citep{ye2023corruption} is excluded: it lacks a tractable implementation for nonlinear classes because its uncertainty weight is a supremum over function pairs and its radius uses the eluder dimension, which is an analysis quantity.
    \item For every algorithm, we multiply the entire confidence radius by 1 when $C=0$ and by $0.125$ when $C>0$.
\end{itemize}

\begin{figure*}[t]
    \centering
    \includegraphics[width=\textwidth]{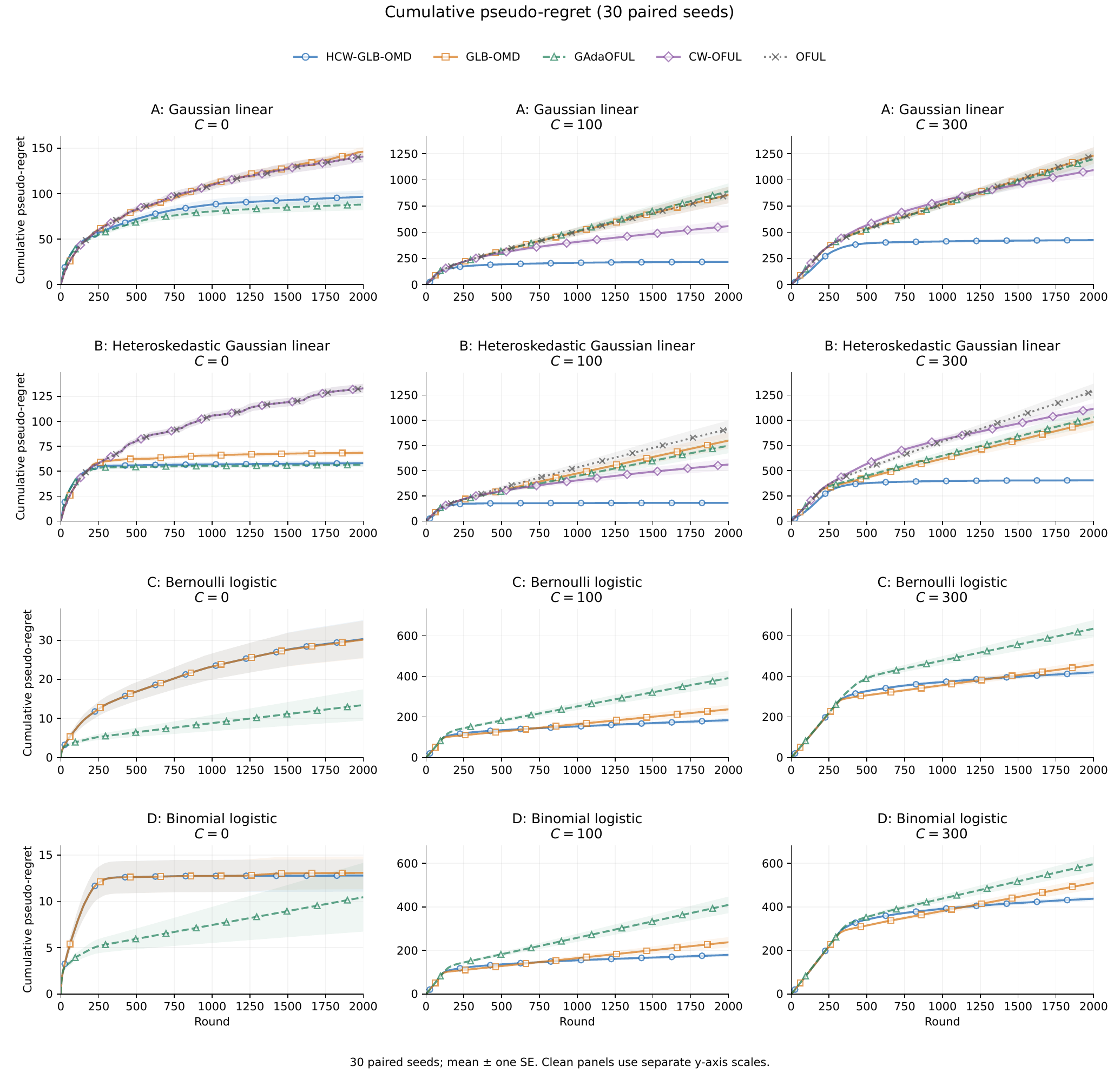}
    \caption{
        Cumulative pseudo-regret over 30 seeds.
        Rows correspond to instances A--D, and columns to corruption
        budgets $C\in\{0,100,300\}$.
        Lines show means, and shaded regions indicate one standard error.
        The corruption-free panels use separate y-axis scales.
    }
    \label{fig:synthetic-regret}
\end{figure*}

\clearpage

\newpage
\section{Potential Future Directions}
\label{app:future}

In this section, we discuss several directions for future work that extend the scope of our framework.

\paragraph{(1) Beyond Uniform Self-Concordance.}
A natural direction is to relax our self-concordance assumption (\Cref{assumption:link-function}) to the more general notion studied by \citet{liu2024free}, which allows the factor controlling $|\ddmu(z)|/\curvaturecolor{\dmu}(z)$ to depend on the natural parameter $z$.
Such an extension could accommodate a broader class of GLMs, including Gamma and inverse Gaussian models, under suitable restrictions on their natural-parameter domains.

This extension would require revisiting both our confidence-sequence construction and our regret analysis.
In particular, our current arguments use uniform self-concordance to compare local curvatures and control approximation errors.
For models with restricted natural-parameter domains, the Gaussian mixtures used in our analysis would also need to be replaced or modified to respect these domains.
Determining whether these changes can preserve computational efficiency, adaptation to local curvature, and an additive dependence on adversarial corruption remains an interesting open question.

\paragraph{(2) Adaptation to Unknown Dispersion.}
Our framework assumes that the learner observes the exact dispersion parameter ${\tau_t}$.
This assumption plays a central technical role: the confidence-sequence proof relies on a weighted likelihood-ratio process, constructed using the true conditional likelihood, being a supermartingale (\Cref{lemma:mixture}).
Simply substituting an estimate of ${\tau_t}$ need not preserve this property, even when the estimate is based only on past observations.

Developing algorithms that adapt to unknown dispersion therefore requires more than modifying the likelihood weights.
Advances in heteroskedastic linear bandits~\citep{zhang2021variance,kim2022variance,zhao2023variance} suggest that variance-adaptive confidence sets may provide a useful starting point.
However, in GLMs, the conditional variance depends jointly on dispersion and the unknown mean parameter, while adversarial corruptions can further distort attempts to estimate either quantity.
An important question is whether one can retain regret guarantees that adapt to the realized noise levels while controlling the effect of corruption, without observing ${\tau_t}$.

\paragraph{(3) Robustness to Model Misspecification.}
A related direction is to relax the assumption that the uncorrupted observations follow a correctly specified GLM.
Model misspecification has been studied in bandit problems~\citep{ghosh2017misspecification,foster2020misspecification}, but extending our framework raises additional questions about which aspects of the model are misspecified.
For example, an inaccurate conditional mean directly affects the regret benchmark, whereas an incorrect likelihood can invalidate our confidence-sequence argument even when the conditional mean is correctly specified.

The statistical literature on misspecified models and robust estimation offers possible tools for addressing these issues~\citep{white1982misspecified,robins1994robust,walker2013misspecified,fortunati2017misspecified}.
One promising objective is to develop confidence sequences and regret guarantees that degrade gracefully with a specified measure of model error.
Understanding how persistent misspecification interacts with budgeted adversarial corruption, and whether their contributions can be controlled separately, would clarify the robustness limits of likelihood-based GLB algorithms.

\paragraph{(4) Online RLHF under Adversarial Corruptions.}
We also hope that our technical insights will facilitate algorithmic and theoretical advances in \emph{online} reinforcement learning from human feedback (RLHF) under adversarial corruptions.
Recent studies of attacks on preference data and reward learning highlight the practical relevance of this problem~\citep{wang2024robust,entezami2025robust,kusaka2025robust}.
Representative work on robust preference learning considers random label noise~\citep{chowdhury2024robust,liang2025robust}, sparse perturbations to preference logits~\citep{bukharin2024robust}, or strong adversarial contamination in \emph{offline} settings~\citep{mandal2025robust,zhou2025robust}.

Extending these guarantees to online interaction with adversarially flipped preference labels presents additional challenges.
The learner's policy determines which comparisons are observed, and corrupted feedback can influence both the reward estimate and subsequent exploration.
Our corruption-robust confidence sequences may provide a starting point for controlling reward-estimation uncertainty under adaptive data collection.
Combining such guarantees with policy optimization and exploration, while accounting for the propagation of corrupted feedback through the learning process, remains an important direction for future work.

\end{document}